\documentclass[twoside,11pt]{article}

\usepackage{blindtext}

\usepackage[preprint]{jmlr2e}

\newcommand{\myCodeFolder}{art}

\usepackage{amsmath,amssymb,bm}
\usepackage{graphics,graphicx,caption,subcaption,tikz, color}
\hypersetup{
	colorlinks   = true,
	linkcolor    = blue,
	citecolor    = blue
}
\usepackage{comment}
\usepackage{stmaryrd}
\usepackage{algorithm}
\usepackage{algpseudocode}
\usepackage[capitalise]{cleveref} % load after 'hyperref'
\newcommand{\R}{\mathbb{R}}

\newcommand{\Prob}{\mathbb{P}}
\newcommand{\Ex}{\mathbb{E}}
\newcommand{\Var}{\operatorname{Var}}
\newcommand{\Cov}{\operatorname{Cov}}
\renewcommand{\Pr}{\operatorname{Pr}}

\newcommand{\tp}{{\scriptscriptstyle\operatorname{T}}} % matrix transpose

\newcommand{\diag}{\operatorname{diag}}

\newcommand{\Ohp}{\mathcal{O}_{\mathbb{P}}}
\newcommand{\Oh}{\mathcal{O}}

\newcommand{\ohp}{\mathsf{o}_{\mathbb{P}}}
\newcommand{\oh}{\mathsf{o}}

\newcommand{\op}{\operatorname{op}}
\newcommand{\frob}{\operatorname{F}}
\newcommand{\trace}{\operatorname{trace}}
\newcommand{\vech}{\operatorname{vech}}

\newcommand{\wh}[1]{\widehat{#1}}

\newcommand{\wt}[1]{\widetilde{#1}}

\newcommand{\emphh}[1]{\emph{#1}}
\newcommand{\Nset}[1]{\llbracket #1 \rrbracket}

\newcommand{\membmtx}{\mathbb{M}}
\newcommand{\meanmtx}{B}
\newcommand{\stdevmtx}{S}
\newcommand{\medianmtx}{\breve{B}}

\newcommand{\numSample}{m}
\newcommand{\intermRank}{\widetilde{R}_{Ni}^{\prime}}

\newcommand{\rank}{\operatorname{rank}}
\newcommand{\rankspec}{\widetilde{B}}
\newcommand{\rankspecBig}{\Theta \widetilde{B} \Theta^{\tp}}

\newcommand{\ranksdev}{\widetilde{S}}

\newcommand{\nrk}[1]{\widetilde{R}_{#1}}
\newcommand{\nrkmtx}{\widetilde{R}_{A}}

\newcommand{\sampMtx}{\widehat{M}}
\newcommand{\popMtx}{M}
\newcommand{\noiseMtx}{E}

\newcommand{\sampVecs}{\widehat{U}}
\newcommand{\popVecs}{U}

\newcommand{\sampVals}{\widehat{\Lambda}}
\newcommand{\popVals}{\Lambda}

\newcommand{\Id}{\operatorname{Id}}
\usepackage{mathtools} % for \coloneqq

\newcommand{\blockmodel}{(\mathcal{F}^{K}, \Theta^{n \times K})}
\newcommand{\PermSet}{\mathbb{J}}

\newcommand{\AsyNormVar}{\widetilde{Z}}

\newcommand{\epsP}{\nu} % For probability statements and Markov's inequality.
\newcommand{\epsK}{\epsilon} % For approximate k-means. See Lei & Rinaldo (2015) notation.

\usepackage{natbib}

\makeatletter
\def\@opargbegintheorem#1#2#3{\trivlist
	\item[]{\bfseries #1\ #2\ (#3)} \itshape}
\makeatother

\usepackage{lastpage}
\jmlrheading{23}{2024}{1-\pageref{LastPage}}{1/23; Revised 1/23}{1/22}{21-0000}{Joshua Cape, Xianshi Yu, and Jonquil Z. Liao}

\ShortHeadings{Robust spectral clustering with rank statistics}{Cape, Yu, and Liao}
\firstpageno{1}

\begin{document}

\title{Robust spectral clustering with rank statistics}

\author{\name Joshua Cape \email jrcape@wisc.edu \\
	\addr Department of Statistics\\
	University of Wisconsin--Madison\\
	Madison, WI 53706, USA
	\AND
	\name Xianshi Yu \email xyu384@wisc.edu \\
	\addr Department of Computer Sciences\\
	University of Wisconsin--Madison\\
	Madison, WI 53706, USA
	\AND
	\name Jonquil Z. Liao \email zliao42@wisc.edu \\
	\addr Department of Statistics\\
	University of Wisconsin--Madison\\
	Madison, WI 53706, USA \\
	%\textbf{Version: \today}
}

\editor{My editor}

\maketitle

\begin{abstract}%   <- trailing '%' for backward compatibility of .sty file
	This paper analyzes the statistical performance of a robust spectral clustering method for latent structure recovery in noisy data matrices. We consider eigenvector-based clustering applied to a matrix of nonparametric rank statistics that is derived entrywise from the raw, original data matrix. This approach is robust in the sense that, unlike traditional spectral clustering procedures, it can provably recover population-level latent block structure even when the observed data matrix includes heavy-tailed entries and has a heterogeneous variance profile. Here, the raw input data may be viewed as a weighted adjacency matrix whose entries constitute links that connect nodes in an underlying graph or network.

Our main theoretical contributions are threefold and hold under flexible data generating conditions. First, we establish that robust spectral clustering with rank statistics can consistently recover latent block structure, viewed as communities of nodes in a graph, in the sense that unobserved community memberships for all but a vanishing fraction of nodes are correctly recovered with high probability when the data matrix is large. Second, we refine the former result and further establish that, under certain conditions, the community membership of any individual, specified node of interest can be asymptotically exactly recovered with probability tending to one in the large-data limit. Third, we establish asymptotic normality results associated with the truncated eigenstructure of matrices whose entries are rank statistics, made possible by synthesizing contemporary entrywise matrix perturbation analysis with the classical nonparametric theory of so-called simple linear rank statistics. Collectively, these results demonstrate the statistical utility of rank-based data transformations when paired with spectral techniques for dimensionality reduction. Numerical examples illustrate and support our theoretical findings. Additionally, for a dataset consisting of human connectomes, our approach yields parsimonious dimensionality reduction and improved recovery of ground-truth neuroanatomical cluster structure. We conclude with a discussion of extensions, practical considerations, and future work.
\end{abstract}


\section{Introduction and background}
\label{sec:introduction}
This paper focuses on spectral methods for data analysis, namely tools and algorithms based on factorizations, truncations, and perturbations of large data matrices. We primarily consider matrices that are real-valued and symmetric which arise naturally, for example, when quantifying pairwise interactions between distinct entities, as with general (dis)similarity matrices in multivariate statistics and adjacency matrices for undirected graphs in network analysis. In such cases, the umbrella term `spectral methods' commonly refers to the behavior of matrix eigenvalues and eigenvectors, collectively comprising the spectral decomposition, or equivalently, the matrix eigendecomposition.

Within the realm of spectral methods, embedding and clustering are popular techniques that, broadly speaking, refer to producing and subsequently analyzing low-dimensional representations of large, complex data \citep{shi2000normalized,mcsherry2001spectral,ng2002spectral}. Considerable effort has been devoted to studying the statistical foundations of spectral clustering and embedding, with early works including \cite{balakrishnan2011noise,rohe2011spectral}. Multiple surveys and monographs now exist, providing an overview of spectral graph clustering \citep{von2007tutorial} and applications of spectral methods in data science \citep{chen2021spectral}. Interest in these topics extends across computer science, machine learning, engineering, and beyond, amassing methods and know-how found in recent review articles such as \cite{cui2018survey,chen2020graph}.

Spectral embedding and clustering are widely-employed techniques for exploratory data analysis that simultaneously provide a starting point for subsequent inference. Given their practical flexibility and computational feasibility, spectral embedding and clustering are often applied to matrix-valued representations of graph or network data and to large data matrices more generally. The rows and columns of such data matrices are frequently conceptualized as indexing nodes (i.e.,~vertices; entities of interest), whose connectivity patterns are reflected in the matrix entries, frequently conceptualized as links (i.e.,~edges; connections). As such, many works investigate spectral techniques in the context of random graph models and statistical network analysis \citep{le2016optimization,ali2018improved,abbe2020entrywise,paul2020spectral,zhang2020detecting,arroyo2021inference,loffler2021optimality,noroozi2021estimation,fan2022simple}, surveyed in part in \cite{abbe2018community,athreya2018statistical}. Applications of spectral embedding and clustering methodology span the physical, natural, and engineering sciences \citep{rohe2016co,liu2018global,priebe2019two}. Collectively, these works form a backdrop for the present paper.

Recent years have witnessed mounting interest in robust data analysis techniques within statistical machine learning, such as in the study of low-rank matrix estimation \citep{elsener2018robust,fan2021shrinkage}, covariance matrix estimation \citep{fan2018eigenvector,minsker2018sub,zhang2022heteroskedastic}, high-dimensional factor models \citep{fan2019farmtest,fan2021robust}, and robust clustering \citep{jana2023adversariallyARXIV,jana2024generalARXIV}. In theoretical computer science, there has been particular interest in adversarially robust algorithms \citep{chen2011robust,yi2016fast,yin2018byzantine} and learning mixture distributions whose component distributions may have heavy tails \citep{dasgupta2005learning,chaudhuri2008beyond,diakonikolas2023algorithmic}. Correspondingly, growing interest has arisen in pursuit of robust spectral methods. The need for such methods is apparent since in practice data may be structured yet contain outliers, artifacts, or significant heterogeneity more generally, qualities that can drastically degrade the performance of non-robust, baseline data analysis procedures. Still, to date, comparatively few works in the literature address robustness considerations in the context of spectral clustering and embedding. Existing works on robust spectral clustering include \cite{li2007noise,chang2008robust,zhu2014constructing,bojchevski2017robust,srivastava2023robust} and collectively contribute novel methodology and algorithms. Still, with the exception of works such as \cite{cai2015robust,srivastava2023robust}, theoretical guarantees for recently proposed methods are largely unavailable in the literature.

\subsection{Overview and contributions of this paper}
\label{sec:contribution}
The principal novelty of this paper is to establish theoretical performance guarantees for a user-friendly parameter-free robust spectral clustering paradigm based on the entrywise rank transformation $a_{ij} \mapsto \operatorname{rank}(a_{ij})/(N+1)$. We establish consistency results for recovering latent block structure from observed data, in the form of finite-sample error bounds that hold with high probability and go to zero in the large-matrix limit. Our results apply to a broad range of model settings, including \textsf{Cauchy} and \textsf{Pareto} heavy-tailed data generating distributions. We further develop distributional properties associated with the proposed methodology, via large-sample asymptotic normality results for expressions involving normalized rank statistics. Collectively, our contributions quantify large-sample patterns observed in point cloud representations of data embeddings derived from large data matrices. Code for this paper is available online on the first author's webpage.

Throughout this paper, the raw input data matrix $A = (a_{ij})$ can be interpreted as a weighted adjacency matrix. From $A$, we derive a matrix of normalized rank statistics, $\nrk{A}$, which itself can be viewed as a weighted adjacency matrix but for an implicit dependent-edge inhomogeneous undirected random graph. The vast majority of existing works on spectral clustering assume entrywise (conditional) independence of matrix elements, modulo matrix symmetry. In contrast, matrices of normalized rank statistics exhibit inherent entrywise dependencies, with heterogeneous pairwise covariances, three-wise dependence, four-wise dependence, etc., leading to additional technical challenges. Our proof analysis tracks entrywise dependencies among normalized rank statistics and establishes that dependence does not preclude probabilistic concentration or statistical consistency. Furthermore, this paper rigorously studies the truncated eigenstructure of matrices comprised of rank statistics, a non-traditional and seemingly under-explored topic.

A key ingredient that underpins our main results is a combinatorial bound for the higher-order trace of a matrix of centered rank statistics, presented in \cref{lem:ms_trace} in \cref{sec:pre_consistency}. To establish weak consistency in \cref{sec:weak_consistency}, we apply this ingredient together with a modified version of the clustering analysis paradigm in \cite{lei2015consistency} which employs the Davis--Kahan subspace perturbation theorem together with the analysis of solutions to the $(1 + \epsK)$-approximate $k$-means clustering problem. In \cref{sec:strong_consistency}, we then consider a more refined level of granularity when analyzing the row-wise perturbations of vectors comprising data embeddings, in order to quantify the large-sample behavior of individual node representations, leading to node-specific asymptotic perfect clustering guarantees. Proving node-wise strong consistency requires a well-chosen step-by-step matrix perturbation analysis to which we apply a suite of concentration inequalities and bounds, notably the expectation bound in \cref{lem:ho_cctrtn}. Thereafter, in \cref{sec:asymptotic_normality}, we establish component-wise and row-wise asymptotic normality for the proposed robust rank-based embedding. Our strong consistency and asymptotic normality results are enabled by adaptations and modifications of the recent entrywise matrix perturbation analysis programme in \cite{cape2019signal,cape2019two}, together with the classical theory of so-called simple linear rank statistics under multi-sample alternatives developed extensively in \cite{hajek1968asymptotic,dupac1969asymptotic}. All proof materials and technical tools are collected in \cref{appendix}.

\subsection{Notation}
\label{sec:notation}
The notation used in this paper is mostly standard and detailed below for completeness.

We write the product of $a$ and $b$ as $a b$, $a \cdot b$, or $a \times b$, depending on context, in an effort to improve readability. We use the equivalence symbol $\equiv$ to emphasize dependencies or relationships that are sometimes suppressed for ease of shorthand, for example $\nrk{} \equiv \nrk{A}$ denotes the reliance of the term $\nrk{}$ on the term $A$.

Given sets $\mathcal{I}$ and $\mathcal{J}$, let $\mathcal{I} \backslash \mathcal{J}$ and $\mathcal{I} - \mathcal{J}$ both denote the set difference operation. Let $\overline{\mathcal{I}}$ denote the complement of the set $\mathcal{I}$. For a general real-valued $n \times n$ matrix $M$ and sets of indices $\mathcal{I}, \mathcal{J} \subseteq \Nset{n} = \{1,2,\dots,n\}$, let $M_{\mathcal{I}\star}$ and $M_{\star\mathcal{J}}$ denote the submatrices of $M$ consisting of the corresponding rows and columns. Let $\membmtx^{n \times K}$ denote the set of all $n \times K$ matrices where each row has a single entry equal to $1$ and all remaining $K-1$ row entries are zero. In other words, for each $\Theta \in \membmtx^{n \times K}$, the $n$ rows of $\Theta$ are standard basis vectors in $\R^{K}$. We call such matrices $\Theta$ membership matrices, and the block membership of row $i$ is denoted by $g_{i} \in \Nset{K}$; in particular, $\Theta_{i g_{i}} = 1$. For each $1 \le k \le K$, denote the set of node indices with block membership $k$ by $\mathcal{G}_{k} \equiv \mathcal{G}_{k}(\Theta) = \{1 \le i \le n: g_{i} = k\}$ having cardinality $n_{k} = |\mathcal{G}_{k}|$. Let $n_{\min} = \min_{1 \le k \le K} n_{k}$, let $n_{\max} = \max_{1 \le k \le K} n_{k}$, and let $n_{\max}^{\prime}$ denote the second-largest block size, namely the second-largest community size. Let $\PermSet^{K \times K}$ denote the set of $K \times K$ permutation matrices.

We write $\doteq$ to denote approximate (truncated) numerical equality, for example, $\pi \doteq 3.14$ and $1/3 \doteq 0.333$. The vector of all ones in $\R^{n}$ is denoted by $1_{n}$, whereas the binary zero-one indicator function evaluated at a condition $\{\cdot\}$ is written as $1\{\cdot\}$. The squared Euclidean norm of $x = (x_{1},\dots,x_{n})^{\tp} \in \R^{n}$ is given by $\|x\|_{\ell_{2}}^{2} = \sum_{i=1}^{n} x_{i}^{2}$, while the Euclidean norm of the $i$-th row of a matrix $M$ is denoted by $\|M\|_{i,\ell_{2}}$. The eigenvalues of a square matrix $M \in \R^{n \times n}$ are denoted by $\lambda_{i} \equiv \lambda_{i}(M)$, for $i=1,\dots,n$, and are ordered in the manner $|\lambda_{1}| \ge |\lambda_{2}| \ge \dots \ge |\lambda_{n}|$. We call $\lambda_{1}, \dots, \lambda_{K}$ the top-$K$ eigenvalues of $M$ with associated top-$K$ orthonormal eigenvectors $u_{1}, \dots, u_{K} \in \R^{n}$ provided $|\lambda_{K}| > |\lambda_{K+1}|$. The trace of a square $n \times n$ matrix $M$ is given by $\trace(M) = \sum_{i=1}^{n} M_{ii} = \sum_{i=1}^{n} \lambda_{i}$. The Frobenius norm of a generic matrix $M$ is given by $\|M\|_{\frob} = \sqrt{\operatorname{trace}(M^{\tp}M)}$. The operator (spectral) norm of a matrix is given by $\|M\|_{\op} = \sup_{\|x\|_{\ell_{2}}=1} \|M x\|_{\ell_{2}} = \max_{1 \le i \le n} \sqrt{\lambda_{i}(M^{\tp}M)} = \sqrt{\lambda_{1}(M^{\tp}M)}$. The norm expression $\|M\|_{0}$ counts the number of nonzero entries in $M$.

Given a vector $x \in \R^{n}$, write $\diag(x) \in \R^{n \times n}$ to indicate the diagonal matrix whose main diagonal entries are given by the entries of $x$. Given a square matrix $M \in \R^{n \times n}$, write $\diag(M) \in \R^{n \times n}$ to denote the diagonal matrix consisting of the main diagonal elements of $M$. The Hadamard (entrywise) product between matrices $A, B \in \R^{n \times n}$ is denoted by $A \circ B = (A_{ij} \times B_{ij})_{1 \le i,j \le n} \in \R^{n \times n}$. Similarly, $A^{\circ 2} \equiv A \circ A$. Entry $(i, j)$ of the matrix $M$ is at times denoted by $m_{i j}$, $m_{i, j}$, $M_{i j}$, or $M_{i, j}$, depending on context.

Given a matrix $M$, its $i$-th row when viewed as a column vector, is denoted by $M_{\textnormal{row }i} \equiv M_{i\star}$ where formally $M_{\textnormal{row }i} = (e_{i}^{\tp}M)^{\tp}$ when $e_{i}$ denotes the $i$-th standard basis vector. The identity matrix in dimension $n$ is denoted by $\Id_{n}$ or $I_{n}$.

Given a random variable $\xi \equiv \xi_{n} \in \R$ indexed by $n$, write $\xi = \ohp(1)$ to denote that for any choice of $\delta > 0$, $|\xi| \le \delta$ holds with probability tending to one as $n \rightarrow \infty$. We write $\xi = \Ohp(f(n,\epsP))$ when there exists a constant $\epsP > 0$ such that for some positive constants $C^{\prime}, C, n_{0}$, it holds that $|\xi| \le C f(n,\epsP)$ with probability at least $1 - C^{\prime} n^{-\epsP}$ for all $n \ge n_{0}$. In this paper, $0 < \epsP \ll 1$. The deterministic analogues are written as $\oh(\cdot)$ and $\Oh(\cdot)$, respectively.

\section{Robust spectral embedding and clustering}
\label{sec:ptr_overview}
\cref{alg:ptr_baseline} specifies the process of converting a symmetric input matrix to a matrix of normalized rank statistics, colloquially called `passing to ranks'. This procedure amounts to a nonparametric entrywise matrix transformation $A \mapsto \nrk{A} \in [0,1]^{n \times n}$ that pre-processes the original input data matrix. Subsequently, it will be shown that this normalization enables the study of $\nrk{A} \approx \Ex[\nrk{A}]$ for establishing spectral clustering guarantees, even when $\Ex[A]$ does not exist.
\begin{algorithm}[H]
	\caption{Pass-to-ranks (PTR) rank-transform procedure}
	\label{alg:ptr_baseline}
	\begin{algorithmic}
		\Require A symmetric matrix $A = (a_{ij}) \in \R^{n \times n}$ with distinct elements $\{a_{ij}\}_{i < j}$.
		\begin{enumerate}
			\item Define $N = n(n-1)/2$.
			\item For $i<j$, define the rank of $a_{ij}$ as $r_{ij} = \sum_{1 \le i^{\prime} < j^{\prime} \le n}1\{a_{i^{\prime} j^{\prime}} \le a_{ij}\}$. Set $\widetilde{r}_{ij} = \frac{r_{ij}}{N+1}$.
			\item For $i > j$, set $\widetilde{r}_{ij} = \widetilde{r}_{ji}$. For each $1 \le i \le n$, set $\widetilde{r}_{ii} = 0$.
		\end{enumerate}
		\Return The symmetric matrix of normalized rank statistics $\widetilde{R}_{A} = (\widetilde{r}_{ij}) \in [0,1]^{n \times n}$.
	\end{algorithmic}
\end{algorithm}

The theory subsequently developed in this paper holds for stochastic input matrices $A$ that have distinct upper-triangular entries with probability one, thereby yielding $\nrk{A}$ per \cref{alg:ptr_baseline}. Elsewhere in practice, a tie-breaking procedure can be specified if the matrix $A$ is allowed to have repeated values among its entries. Setting the diagonal elements of $\nrk{A}$ to zero can be interpreted as removing possible self-loops in a graph.

Given the matrix of normalized rank statistics $\nrk{A}$, we next consider top-$d$ eigenvector-based spectral embedding and clustering as outlined in \cref{alg:ptr_clustering}. We shall focus our study on the performance of $(1+\epsK)$-approximate $k$-means clustering, though in practice applying other variants of $k$-means clustering can also be appropriate. Further, our subsequent asymptotic normality results suggest that clustering spectral embeddings via Gaussian mixture modeling can also be appropriate.

\begin{algorithm}[ht]
	\caption{Spectral embedding and clustering with matrices of rank statistics}
	\label{alg:ptr_clustering}
	\begin{algorithmic}
		\Require A symmetric matrix $A = (a_{ij}) \in \R^{n \times n}$ with distinct elements $\{a_{ij}\}_{i < j}$.
		\begin{enumerate}
			\item Obtain the matrix of normalized rank statistics $\nrk{A}$ via \cref{alg:ptr_baseline}.
			\item Compute a truncated eigendecomposition of $\nrk{A}$, keeping the $1 \le d \le n$ largest-in-magnitude eigenvalues and associated orthonormal eigenvectors, denoted $\sampVecs_{\nrk{}} \in \R^{n \times d}$. Here, $d$ could be chosen \emph{a priori} or selected, for example, via an eigenvalue ratio test, parallel analysis, or inspection of the scree plot.
			\item Cluster the rows of (the embedding) $\sampVecs_{\nrk{}}$ into $K$ groups. Here, $K$ could be chosen \emph{a priori} or selected, for example, via an information criterion.
		\end{enumerate}
		\Return A clustering $\mathcal{C} = \left\{\mathcal{C}_{1},\dots,\mathcal{C}_{K}\right\}$ that partitions the set $\{1,\dots,n\}$.
	\end{algorithmic}
\end{algorithm}

\section{Matrix-valued data with latent block structure}
\label{sec:blockmodel_data}
This section specifies a flexible data generating model for matrices with population-level block structure. Block structure can be interpreted as unobserved community structure that drives connectivity patterns between entities, such as in friendship networks derived from social media platforms or in brain graphs derived from neuroimaging scans. Thereafter, in \cref{sec:main_results} we show that the leading eigenvectors $\sampVecs \equiv \sampVecs_{\nrk{}}$ obtained in \cref{alg:ptr_clustering} accurately estimate population-level ground-truth eigenstructure.

Let $\mathcal{F} \equiv \mathcal{F}^{K} = \{F_{(k,k^{\prime})} : 1 \le k \le k^{\prime} \le K\}$ denote a collection of $K(K+1)/2$ absolutely continuous cumulative distribution functions $F_{(k, k^{\prime})}$. For notational convenience, define $F_{(k,k^{\prime})} = F_{(k^{\prime},k)}$ whenever $k > k^{\prime}$. We do not restrict the support sets of these distributions nor do we impose the existence of certain moments. For example, for $K=2$, we may associate $F_{(1,1)}$, $F_{(1,2)}$, and $F_{(2,2)}$ with, say, the distributions $\textsf{Normal}(\mu,\sigma^{2})$, $\textsf{Gamma}(\alpha,\beta)$, and $\textsf{Pareto}(m,\gamma)$, respectively, or even contaminated versions thereof. Let $\Theta \equiv \Theta^{n \times K} \in \membmtx^{n \times K}$ be a membership matrix as in \cref{sec:notation}. We consider symmetric data matrices $A$ with population-level block structure arising via \cref{def:data_matrices}. This definition is reminiscent of weighted stochastic blockmodel formulations put forth in the literature \citep{aicher2015learning, xu2020optimal, gallagher2023spectral}.

\begin{definition}[Data matrices with blockmodel structure]
	\label[definition]{def:data_matrices}
	We say that $A \in \R^{n \times n}$ is a symmetric random data matrix from the generative model $\blockmodel$ when its entries are random variables given by
	\begin{equation}
		A_{ij} =
		\begin{cases}
			\textnormal{independent } F_{(g_{i},g_{j})}
			& \textnormal{ if } i \le j,\\
			A_{ji}
			& \textnormal{ if } i > j.
		\end{cases}
	\end{equation}
	In particular, for each row (node) index $i \in \Nset{n}$, it holds that $\Theta_{i g_{i}} = 1$ and $\Theta_{i k} = 0$ for all $k \neq g_{i}$. Additionally, if $A_{ii}=0$ for all $1 \le i \le n$ is enforced, then $A$ is said to be hollow. We write $A \sim \operatorname{Blockmodel}\blockmodel$.
\end{definition}

Let $\medianmtx$ denote a $K \times K$ symmetric matrix of coefficients whose upper-triangular entries are given by $\medianmtx_{k k^{\prime}} = \operatorname{Median}_{X \sim F_{(k, k^{\prime})}}[X]$ for all $k \le k^{\prime}$. Here, $\medianmtx$ always exists since each distribution function is absolutely continuous, while the possibility of non-unique medians does not pose any problems for this paper. Analogously, let $\meanmtx$ denote the $K \times K$ symmetric matrix of coefficients whose upper-triangular entries are given by $\meanmtx_{k k^{\prime}} = \Ex_{X \sim F_{(k, k^{\prime})}}[X]$ for all $k \le k^{\prime}$, provided the expectations all exist. In particular, $\meanmtx$ is the matrix of block-wise expectations for the entries of $A$. Analogously, define the matrix of variances, $\stdevmtx^{\circ 2} \in \R^{K \times K}_{\ge 0}$, in the entrywise manner $\stdevmtx_{k k^{\prime}}^{\circ 2} = \Var_{X \sim F_{(k, k^{\prime})}}[X]$ for all $k \le k^{\prime}$, provided all the second moments exist.

\begin{remark}[Population-level block structure]
	\label[remark]{remark:pop_struct_blockmodel_data}
	We emphasize that the parameter matrices $\medianmtx$, $\meanmtx$, and $\stdevmtx$ are determined indirectly, as a consequence of specifying the collection of distributions $\mathcal{F}^{K}$. Further, the entrywise median, expectation, and variance matrices associated with $A$, provided the latter exist, are block-structured with
	\begin{equation}
		\label{eq:adj_population_block_structure}
		\operatorname{Median}[A]
		=
		\Theta \medianmtx \Theta^{\tp},
		\quad
		\Ex[A]
		=
		\Theta \meanmtx \Theta^{\tp},
		\quad
		\Var[A]
		=
		\Theta \stdevmtx^{\circ 2} \Theta^{\tp}.
	\end{equation}
	If $A$ is taken to be hollow, then $\operatorname{Median}[A] = \Theta \medianmtx \Theta^{\tp} - \diag(\Theta \medianmtx \Theta^{\tp})$ and similarly for $\Ex[A]$ and $\Var[A]$, thereby amounting to approximate population-level block structure.
\end{remark}

\subsection{Matrices of normalized rank statistics}
\label{sec:matrices_of_normalized_rank_statistics}
In \cref{def:data_matrices}, the upper-triangular entries of $A$ are independent random variables with absolutely continuous cumulative distribution functions. Hence, they are distinct with probability one, and \cref{alg:ptr_baseline} can be applied to $A$ to yield a symmetric random matrix of normalized rank statistics, $\nrk{A}$, whose entries are necessarily dependent.

At the population level, write $\rankspec \in [0,1]^{K \times K}$ to denote the symmetric matrix where
\begin{equation}
	\label{eq:nranks_entry_expectation}
	\Ex\left[(\wt{R}_{A})_{ij}\right]
	= \begin{cases}
		\rankspec_{g_{i},g_{j}}
		\quad
		& \text{if~}i \neq j,\\
		0
		& \text{if~}i = j.
	\end{cases}
\end{equation}
In words, $\rankspec$ is the matrix of block-wise expected normalized ranks derived from the model $\operatorname{Blockmodel}\blockmodel$. This matrix of expected values is always well-defined since the normalized ranks are discrete bounded random variables hence all their moments exist.

The pass-to-ranks procedure zeroes out the main diagonal elements of $\nrk{A}$, hence its expectation is a perturbation of $\rankspecBig$, namely
\begin{equation}
	\label{eq:nranks_matrix_expectation}
	\Ex[\nrk{A}]
	=
	\rankspecBig - \diag(\rankspecBig).
\end{equation}
In what follows, the additive perturbation $\diag(\rankspecBig)$ will be shown to have a negligible effect, letting us write that $\Ex[\nrk{A}]$ is itself approximately block-structured vis-\`{a}-vis $\rankspecBig$, whose entries are necessarily bounded between zero and one.

Looking ahead, \cref{sec:ranks_one_sample,sec:ranks_multi_sample} provide exact, finite-sample formulas characterizing $\Ex[(\nrk{A})_{ij}]$, $\Var[(\nrk{A})_{ij}]$, and $\Cov[(\nrk{A})_{ij},(\nrk{A})_{i^{\prime}j^{\prime}}]$ for all entries of the normalized ranks matrix $\nrk{A}$. These expressions translate directly to expressions for $\rankspec$ and $\ranksdev$; further, they are crucial for the numerical examples provided throughout this paper. Notably, the entries of $\rankspec$ and $\ranksdev$ depend on the block sizes $n_{k} = |\mathcal{G}_{k}|$, $1 \le k \le K$, unlike for $\meanmtx$ and $\stdevmtx$.

\subsection{Eigenvector preliminaries for block-structured matrices}
\label{sec:eigenvector_preliminaries}
This section highlights key aspects of the population-level eigenstructure in blockmodel data matrices that are important throughout this paper. Firstly, $\rankspec$ may have rank equal to $K$ even if $\medianmtx$ does not have rank equal to $K$. Next, suppose that the matrices $\Theta$, $\medianmtx$, and $\rankspec$ each have rank equal to $K$. Importantly, the assumption that $\rank(\medianmtx) = K$ places an implicit requirement on the collection of distributions $\mathcal{F}$ (e.g.,~$F \in \mathcal{F}$ cannot all be identical when $K > 1$), while assuming $\rank(\rankspec) = K$ places an implicit, joint requirement on $\mathcal{F}$ and $\Theta$. Next, define $\Delta = (\Theta^{\tp}\Theta)^{1/2} = \diag(\sqrt{n_{1}},\dots,\sqrt{n_{K}}) \in \R^{K \times K}$ which is itself rank $K$ since $\rank(\Theta) = K$ if and only if $n_{k} > 0$ for each $k \in \Nset{K}$, i.e.,~$\mathcal{G}_{k} \neq \emptyset$ for each $k$. The matrix $\Theta \Delta^{-1} \in \R^{n \times K}$ has orthonormal column vectors, hence its columns span the top-$K$ dimensional eigenspace of both $\Theta \medianmtx \Theta^{\tp} \equiv (\Theta \Delta^{-1}) \Delta \medianmtx \Delta (\Theta \Delta^{-1})^{\tp}$ and $\rankspecBig$. Namely, the block structure underlying both $\operatorname{Median}[A]$ and $\Ex[\nrk{A}]$ approximately share the same $K$-dimensional leading eigenspace, noting that $\Delta \medianmtx \Delta$ and $\Delta \rankspec \Delta$ are themselves full-rank $K \times K$ matrices. The same conclusion holds for $\meanmtx$ if it exists and satisfies $\rank(\meanmtx) = K$.

In fact, the leading population-level eigenvectors encode the memberships for all nodes (rows) in block-structured data matrices per \cref{result:blockmodel_eigenvectors}. This observation and the discussion above are well known in the literature on stochastic blockmodel random graphs where spectral clustering is applied to symmetric random matrices $A$ whose upper-triangular entries are independent \textsf{Bernoulli} random variables; see \cite{lei2015consistency}. Consequently, a main challenge in this paper is to show that the top-$K$ truncated eigenspace of $\nrk{A}$ accurately estimates the top-$K$ eigenspace of $\rankspecBig$ by virtue of its close proximity to $\Ex[\nrk{A}]$.

\subsection{An example with contaminated normal distributions}
\label{sec:example_illustration}

\begin{figure}[tp]
	\centering
	\begin{subfigure}{0.99\linewidth}
		\centering
		\includegraphics[width=0.7\linewidth]{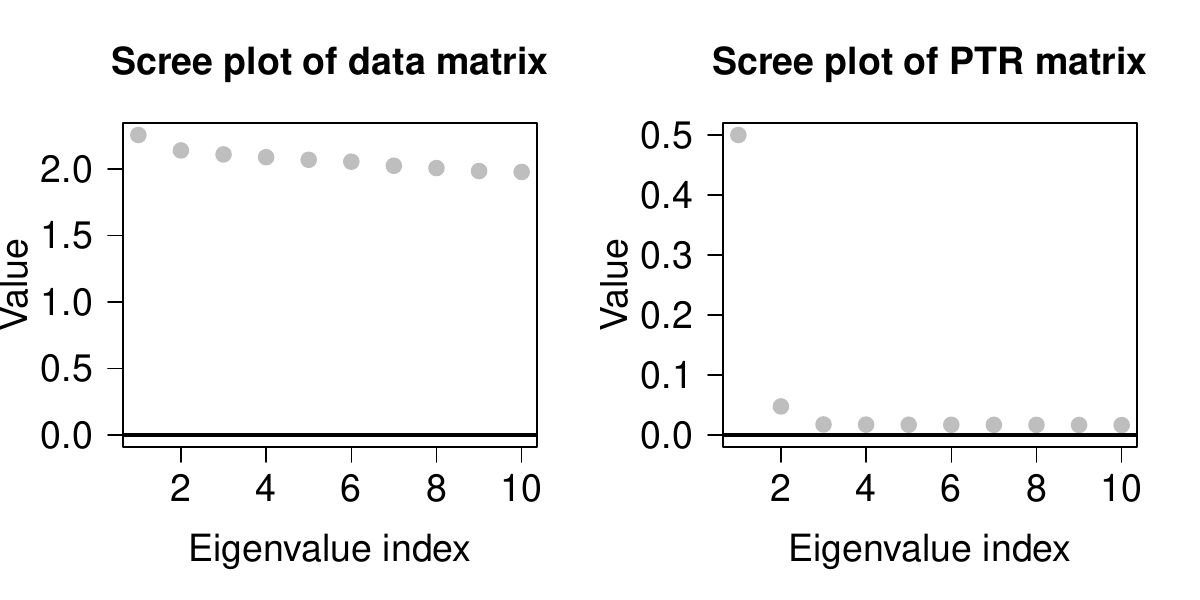}
	\end{subfigure}
	\begin{subfigure}{0.99\linewidth}
		\centering
		\includegraphics[width=0.7\linewidth]{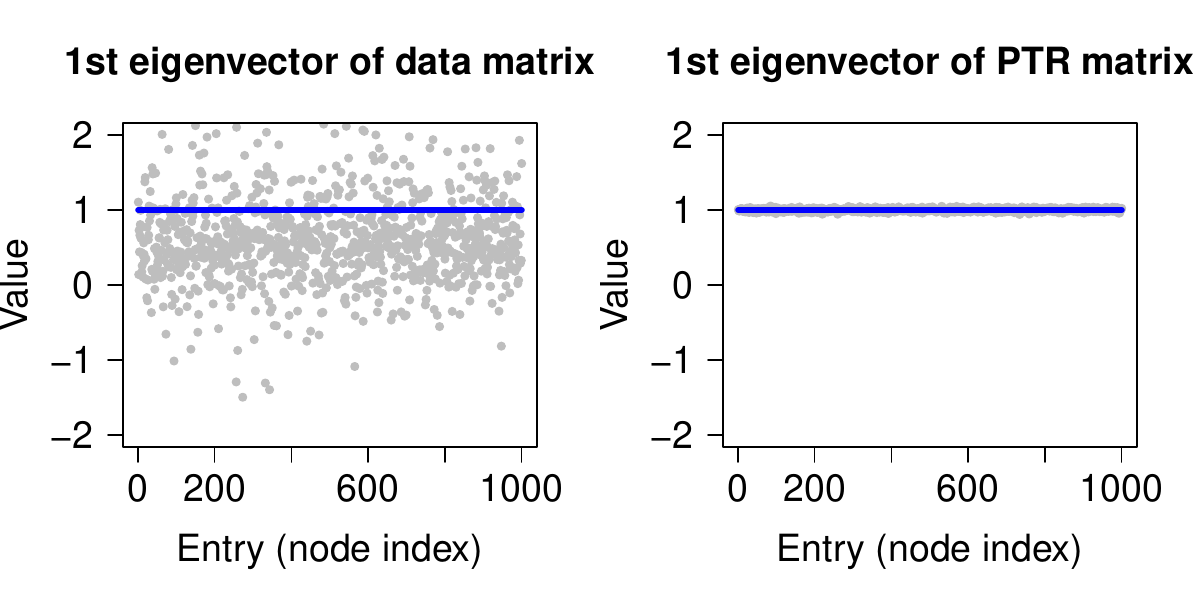}
	\end{subfigure}
	\begin{subfigure}{0.99\linewidth}
		\centering
		\includegraphics[width=0.7\linewidth]{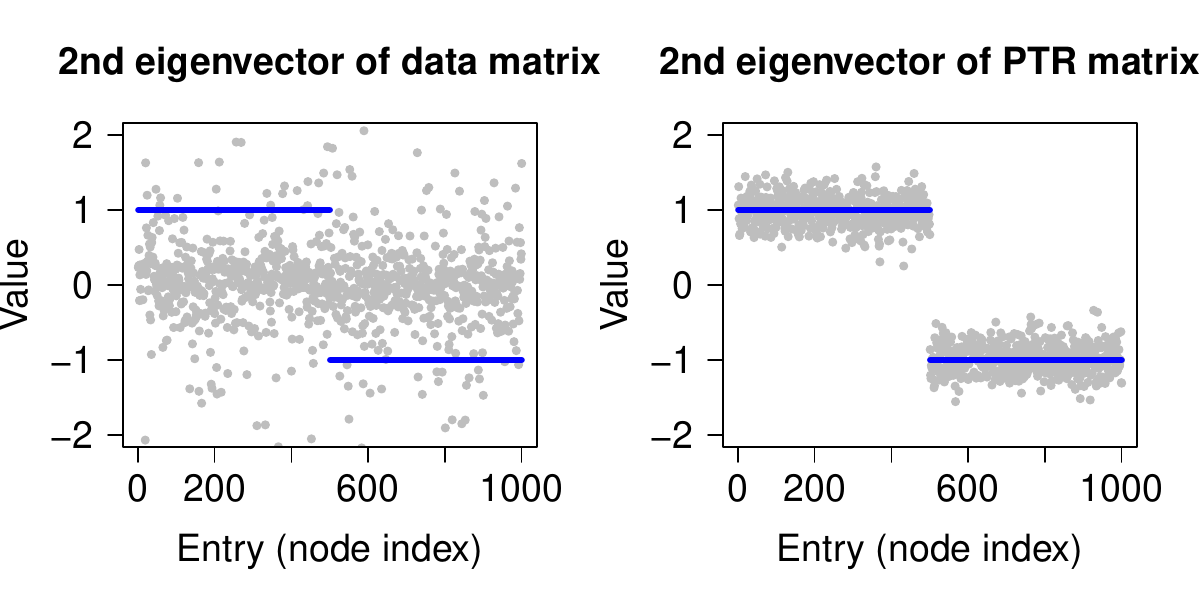}
	\end{subfigure}
	\caption{Simulation example in \cref{sec:example_illustration} showing the truncated eigenstructure of $A$ (left column) and $\nrk{A}$ (right column). Blue lines denote the population-level ground truth eigenvector behavior. Gray points are derived from the simulated data. Several outlier values outside the interval $[-2,2]$ are not shown. The two-dimensional eigentruncation of the normalized ranks matrix recovers the unobserved two-block ground truth structure.}
	\label{fig:intro_plots}
\end{figure}

We briefly pause to provide a numerical illustration. Consider a $K=2$ blockmodel per \cref{def:data_matrices} with data drawn from contaminated normal distributions of the form
\begin{equation*}
	F_{(k,k^{\prime})}
	=
	(1-\epsilon) \cdot \textsf{Normal}(\mu_{kk^{\prime}},\sigma_{kk^{\prime}}^{2})
	+
	\epsilon \cdot \textsf{Normal}(\mu_{kk^{\prime}}, 100^{2} \cdot \sigma_{kk^{\prime}}^{2}),
	\quad\quad
	1 \le k \le k^{\prime} \le 2.
\end{equation*}
Concretely, set $\epsilon = 0.01$, $\mu_{11} = \mu_{22} = 2$, $\mu_{12} = 1$, and $\sigma_{11} = \sigma_{12} = \sigma_{22} = 3$. Notice that the standard deviation of the contaminating distribution is one hundred times larger than that of the baseline normal distribution for each block.

We generate $A \in \R^{n \times n}$ with $n=1000$ according to the above model with equal block sizes $n_{1} = n_{2} = n/K = 500$. We order the block memberships for ease of presentation but emphasize that the memberships are themselves unobserved in practice. Our goal is to recover the row (equivalently, column) memberships via spectral embedding and clustering. Here, $\rank(\medianmtx) = \rank(\meanmtx) = \rank(\rankspec) = 2$. Furthermore, here the matrices $\meanmtx$ and $\rankspec$ have the same eigenvectors, enabling easier visual comparison.

For both $A$ and $\nrk{A}$, we compute their ten largest eigenvalues. While this specification technically does not guarantee that we retain the top ten largest-in-magnitude eigenvalues, we make it here for the sake of visualization and because historically some practitioners chose to ignore negative eigenvalues. \cref{fig:intro_plots} shows these eigenvalues in partial scree plots, scaled by $n^{-1}$, alongside the sign-adjusted leading two eigenvectors, scaled by $n^{1/2}$. Even for the specified small level of contamination, the spectrum of $A$ is detrimentally influenced by large-in-magnitude entries arising from the huge variance of the contaminating distributions. The eigenvalues of $A$ do not reflect the two-block ground truth structure, and neither do its eigenvectors, even upon inspecting the third eigenvector, fourth eigenvector, etc.

The spectrum of $\nrk{A}$, on the other hand, exhibits a noticeable `elbow' at the second eigenvalue, correctly indicating approximate two-block structure. The corresponding $n^{1/2}$-norm eigenvectors are plotted against the ground-truth $n^{1/2}$-norm eigenvectors, displayed in blue. Here, the ground-truth second eigenvector reveals the two-block structure which is recovered by the corresponding eigenvector of the normalized ranks matrix. Any sensible clustering algorithm, when applied to this second eigenvector, will perfectly distinguish the two latent blocks based on the entrywise signs. Our node-specific strong consistency results in \cref{sec:strong_consistency} apply to this model setting, thereby theoretically justifying spectral clustering applied to the pass-to-ranks matrix. In this simulation example, the first two eigenvalues of $\nrk{A}$ are in fact its largest-in-magnitude eigenvalues.

The truncated eigenstructure of $\nrk{A}$ continues to successfully reflect the ground-truth block structure of the model even when the contaminating distribution is more extreme, such as $\textsf{Cauchy}$, both in theory and in simulations. In such situations $\meanmtx$ may not exist, nevertheless $\rank(\rankspec) = 2$ may still hold. \cref{sec:numerical} provides numerical illustrations for data matrices whose entries have heavy-tailed distributions.

\section{Statistical considerations for clustering}
\label{sec:cluster_approx_kmeans_errors}
\cref{sec:cluster_approx_kmeans_errors} provides background and context needed for the weak consistency results in \cref{sec:weak_consistency}. The presentation here follows the paradigm for adjacency-based spectral clustering of unweighted stochastic blockmodel random graphs in \cite{lei2015consistency}.

In pursuit of weak consistency, we seek to recover the membership matrix $\Theta$ via an estimate $\wh{\Theta} \in \membmtx^{n \times K}$, modulo column-wise permutation $\Pi \in \PermSet^{K \times K}$. We consider two popular criteria for evaluating estimation error.
\begin{align*}
	\textnormal{Overall relative error:}
	\quad
	L(\wh{\Theta}, \Theta)
	&=
	\underset{\Pi \in \PermSet^{K \times K}}{ \min } \,
	\frac{1}{n} \|\wh{\Theta}\Pi - \Theta\|_{0}.\\	
	\textnormal{Worst case relative error:}
	\quad
	\wt{L}(\wh{\Theta}, \Theta)
	&=
	\min_{\Pi \in \PermSet^{K \times K}} \max_{1 \le k \le K} \,
	\frac{1}{n_{k}} \, \|(\wh{\Theta} \Pi)_{\mathcal{G}_{k}\star} - \Theta_{\mathcal{G}_{k}\star}\|_{0}.
\end{align*}
In particular, we seek conditions under which $L(\wh{\Theta},\Theta), \wt{L}(\wh{\Theta},\Theta) \rightarrow 0$ as $n \rightarrow \infty$, noting that $0 \le L(\wh{\Theta},\Theta) \le \wt{L}(\wh{\Theta},\Theta) \le 2$. Here, $\wh{\Theta}$ yielding a small value $L(\wh{\Theta},\Theta)$ may have large relative error for some blocks of comparatively smaller size, whereas $\wt{L}(\wh{\Theta},\Theta)$ is a more stringent condition that quantifies how well $\wh{\Theta}$ uniformly estimates each block relative to its size. We emphasize that the wide tilde notation in $\wt{L}(\wh{\Theta},\Theta)$ is not associated with its usage for normalized rank statistics in $\nrk{A}$.

For the purposes of this paper, the $k$-means clustering problem is expressed formally as
\begin{align}
	\label{eq:kmeans_problem}
	\begin{split}
		(\widehat{\Theta}, \widehat{T})
		=
		\underset{\Theta \in \membmtx^{n \times K}, T \in \R^{K \times K}}{\arg \min} \|\wh{U} - \Theta T\|_{\frob}^{2},
	\end{split}
\end{align}
where $\wh{U}$ denotes a generic orthonormal input matrix of leading eigenvectors, and where $T$ scales and reorients $\Theta$ to enable commensurability with $\wh{U}$. A more computationally tractable and efficient alternative is to instead relax the above problem to the $(1+\epsK)$-approximate $k$-means problem, of the form
\begin{align}
	\label{eq:approx_kmeans_problem}
	\begin{split}
		(\widehat{\Theta}, \widehat{T}) \in \membmtx^{n \times K} \times \R^{K \times K} 
		\textnormal{ s.t. }
		\|\wh{U} - \wh{\Theta}\wh{T}\|_{\frob}^{2}
		\le
		(1+\epsK) \cdot \left\{ \underset{\Theta \in \membmtx^{n \times K}, T \in \R^{K \times K}}{\min}\|\wh{U} - \Theta T\|_{\frob}^{2} \right\}.
	\end{split}
\end{align}
\cref{alg:approx_kmeans} combines these ideas in the language of traditional adjacency-based spectral clustering. In this paper, our robust procedure applies \cref{alg:approx_kmeans} to $\nrk{A}$ to obtain the analogous quantities $\wh{U}_{\nrk{}}$ and $\wh{T}_{\nrk{}}$ and $\wh{\Theta}_{\nrk{}}$.

\begin{algorithm}[H]
	\caption{Spectral clustering with approximate $k$-means; see \cite{lei2015consistency}}
	\label{alg:approx_kmeans}
	\begin{algorithmic}
		\Require A symmetric input matrix $M \in \R^{n \times n}$; number of communities (blocks) $K \ge 1$; approximation parameter $\epsK > 0$.
		\begin{enumerate}
			\item Compute $\wh{U}_{M} \in \R^{n \times K}$, whose columns consist of $K$ orthonormal leading eigenvectors of $M$, ordered in decreasing order of eigenvalue modulus.
			\item Let $(\wh{\Theta}_{M}, \wh{T}_{M})$ be a solution to the $(1+\epsK)$-approximate $k$-means problem in \cref{eq:approx_kmeans_problem} with $K$ clusters and input matrix $\wh{U}_{M}$.
		\end{enumerate}
		\Return Estimated membership matrix $\wh{\Theta}_{M} \in \membmtx^{n \times K}$.
	\end{algorithmic}
\end{algorithm}

\section{Main results}
\label{sec:main_results}
\cref{sec:pre_consistency} is a precursor to the ensuing consistency results and may be of independent interest. We shall distinguish between `weak consistency' and node-specific `strong consistency', motivated by existing terminology in the literature \citep{abbe2018community}. In \cref{sec:weak_consistency}, we establish consistency of our matrix-valued estimator for the population-level target (i.e.,~the membership matrix; its top low-rank eigenspace). In \cref{sec:strong_consistency}, we consider a more refined level of granularity by quantifying the behavior of individual nodes' vector embedding representations. Of note, our strong consistency results are node-specific, in contrast to uniform strong consistency results that simultaneously hold for all nodes. In \cref{sec:asymptotic_normality}, we establish asymptotic normality results associated to the embedding vectors in $\R^{K}$.

\subsection{Bounding matrices of normalized rank statistics}
\label{sec:pre_consistency}
This section presents two technical lemmas which are proved in \cref{app:trace_bound_proof}. The first, \cref{lem:os_trace}, bounds the higher-order trace of the matrix difference $\nrk{A} - \Ex[\nrk{A}]$ when $K=1$.

\begin{lemma}[Single block expected trace bound]\label[lemma]{lem:os_trace}
	Let $A$ be a symmetric $n \times n$ random matrix with i.i.d.~entries $A_{ij} \sim F$ for $i \leq j$, where $F$ is an absolutely continuous cumulative distribution function. Let $\nrk{A}$ denote the matrix of normalized rank statistics, per \cref{alg:ptr_baseline}. If $n \geq 4$, then
	\begin{equation}
		\label{equ::one_sample_trace_bound}
		0
		\leq
		\Ex\left[\trace\left((\nrk{A}-\Ex[\nrk{A}])^4\right)\right]
		\leq
		\frac{1}{50} n^3.
	\end{equation} 
\end{lemma}

\cref{lem:ms_trace} generalizes \cref{lem:os_trace} from $K=1$ to the general case $K \ge 1$.

\begin{lemma}[Multiple block expected trace bound]\label[lemma]{lem:ms_trace}
	Let $A \sim \operatorname{Blockmodel}\blockmodel$ per \cref{def:data_matrices}. Let $\nrk{A}$ denote the matrix of normalized rank statistics, per \cref{alg:ptr_baseline}. If $n_{k} \geq 5$ for all $k \in \Nset{K}$, then
	\begin{equation}
		\label{equ::multiple_sample_trace_bound}
		0
		\leq
		\Ex\left[\trace\left((\nrk{A}-\Ex[\nrk{A}])^4\right)\right]
		\leq
		66K^{2}n^{2}+2n^{3}.
	\end{equation}
\end{lemma}

\subsection{Weak consistency}
\label{sec:weak_consistency}
This section presents the weak consistency results which are proved in \cref{app:weak_consistency}. \cref{result:weak_consistency_kmeans} states our most general weak consistency result and relies on \cref{lem:ms_trace}.

\begin{theorem}[Weak consistency of robust spectral clustering with approximate $k$-means clustering]
	\label[theorem]{result:weak_consistency_kmeans}
	Suppose $A \sim \operatorname{Blockmodel}\blockmodel$ where $\rank(\Theta) = \rank(\rankspec) = K$ and $K \le n^{1/2}$ and $n_{k} \ge 5$ for all $k \in \Nset{K}$. Let $\gamma_{n}$ denote the smallest nonzero singular value of $\rankspecBig$. Let $\widehat{\Theta}$ be the output of spectral clustering using $(1+\epsK)$-approximate $k$-means clustering per \cref{alg:approx_kmeans} with input $\nrk{A}$ per \cref{alg:ptr_baseline}. There exists an absolute constant $c > 0$ such that if
	\begin{equation*}
		(2 + \epsK) \left(\frac{K n^{1 + \epsP}}{\gamma_{n}^{2}}\right)
		<
		c
	\end{equation*}
	for some $\epsP > 0$ with $\gamma_{n} > C_{2}^{\prime}n^{3/4+\epsP}$ for some sufficiently large absolute constant $C_{2}^{\prime} > 0$, then with probability at least $1 - n^{-\epsP}$ there exist subsets $\mathcal{S}_{k} \subset \mathcal{G}_{k}$ for $1 \le k \le K$ and a $K \times K$ permutation matrix $J$ such that $\widehat{\Theta}_{\mathcal{G}^{\star}} J = \Theta_{\mathcal{G}^{\star}}$, where $\mathcal{G} = \bigcup_{k=1}^{K} (\mathcal{G}_{k} \, \backslash \, \mathcal{S}_{k})$ and
	\begin{equation*}
		\sum_{k=1}^{K} \frac{|\mathcal{S}_{k}|}{n_{k}}
		\le
		c^{-1} (2 + \epsK) \left(\frac{K n^{1 + \epsP}}{\gamma_{n}^{2}}\right).
	\end{equation*}
\end{theorem}

\cref{result:weak_consistency_kmeans} does not make any parametric or moment assumptions on the distributions comprising $\mathcal{F}^{K}$. Further, the block-specific distributions may differ from one another, for example, they may consist of \textsf{Normal}, \textsf{Beta}, \textsf{Cauchy}, \textsf{Pareto}, etc. As such, this consistency result for robust spectral clustering with rank statistics holds under very flexible data generating conditions. Below, \cref{result:corollary_weak_consistency} translates \cref{result:weak_consistency_kmeans} into an easier-to-interpret statement involving the loss functions $L$ and $\wt{L}$.

\begin{corollary}[Weak consistency of robust spectral clustering via relative error bounds]
	\label[corollary]{result:corollary_weak_consistency}
	Assume the hypotheses in \cref{result:weak_consistency_kmeans}. There exists a constant $c > 0$ such that if
	\begin{equation*}
		(2+\epsK) \left(\frac{K n^{1 + \epsP}}{n_{\min}^{2} \lambda_{K}^{2}(\rankspec)}\right)
		<
		c
	\end{equation*}
	for some $\epsP > 0$ with $n_{\min} \cdot |\lambda_{K}(\rankspec)| > C_{2}^{\prime}n^{3/4+\epsP}$ for some sufficiently large absolute constant $C_{2}^{\prime} >0$, then with probability at least $1 - n^{-\epsP}$ it holds that
	\begin{align*}
		L(\wh{\Theta}, \Theta)
		&\le
		2c^{-1} (2+\epsK) \left(\frac{K n^{\epsP} n_{\max}^{\prime}}{n_{\min}^{2} \lambda_{K}^{2}(\rankspec)}\right),
	\end{align*}
	and simultaneously
	\begin{align*}
		\centering
		\wt{L}(\wh{\Theta},\Theta)
		&\le
		2c^{-1} (2+\epsK) \left(\frac{K n^{1 + \epsP}}{n_{\min}^{2} \lambda_{K}^{2}(\rankspec)}\right).
	\end{align*}
\end{corollary}

Our interest is in regimes for which the upper bounds decay to zero as the matrix size $n$ increases. For example, when each $n_{k}$ is order $n/K$ and $|\lambda_{K}(\rankspec)|$ is an absolute constant, then by \cref{result:corollary_weak_consistency}, for $0 < \epsP < 1/4$,
\begin{equation*}
	L(\wh{\Theta},\Theta)
	=
	\Ohp\left(K^{2} \cdot n^{-1+\epsP}\right),
	\quad \quad
	\wt{L}(\wh{\Theta},\Theta)
	=
	\Ohp\left(K^{3} \cdot n^{-1+\epsP}\right).
\end{equation*}
Our bounds show that $L(\wh{\Theta},\Theta), \wt{L}(\wh{\Theta},\Theta) \rightarrow 0$ in probability as $n \rightarrow \infty$ when $K = \Oh(1)$ and even in regimes where $K \equiv K_{n} \rightarrow \infty$ as $n \rightarrow \infty$. It is also possible for the upper bounds to behave more differently. For example, suppose there is one dominant block, hence $n_{\max}^{\prime} \ll n$ and so the upper bound on the overall relative error $L$ is much smaller than the upper bound on the worst case relative error $\wt{L}$.

\cref{result:weak_consistency_kmeans} is a precise clustering statement for approximate $k$-means clustering. It is derived by employing subspace perturbation bounds to show that the estimated top eigenspace is very close to the population-level top eigenspace and hence that the estimated membership matrix recovers the true membership matrix. Subspace perturbation bounds themselves maintain the interpretation of weak consistency but without referencing a specific algorithm, namely they do not formally treat the approximation parameter $\epsK$. As such, \cref{eq:cor_top_eigen_relative_error} in \cref{result:cor_top_eigen_relative_error} provides yet another lens through which to view and understand the weak consistency result.

\begin{theorem}[Top eigenspace relative error of robust embedding]
	\label[theorem]{result:cor_top_eigen_relative_error}
	Suppose $A \sim \operatorname{Blockmodel}\blockmodel$ where $\rank(\Theta) = \rank(\rankspec) = K$ and $K \le n^{1/2}$ and $n_{k} \ge 5$ for all $k \in \Nset{K}$. Let $\popVecs\popVecs^{\tp}$ denote the orthogonal projection onto the column space of $\rankspecBig$, and let $\sampVecs\sampVecs^{\tp}$ denote the orthogonal projection onto the leading $K$-dimensional eigenspace of $\nrk{A}$. Let $\gamma_{n}$ denote the smallest nonzero singular value of $\rankspecBig$. There exist absolute constants $C, C_{2}^{\prime} > 0$ such that if $\gamma_{n} > C_{2}^{\prime} n^{3/4+\epsP}$ for some $\epsP > 0$, then
	\begin{equation}
		\label{eq:cor_top_eigen_relative_error}
		\Pr\left[ \frac{\| \sampVecs\sampVecs^{\tp} - \popVecs\popVecs^{\tp} \|_{\frob}}{\| \popVecs\popVecs^{\tp} \|_{\frob}} > \frac{C}{|\lambda_{K}(\rankspec)|} \left(\frac{n^{1/2+\epsP/2}}{n_{\min}} \right) \right]
		\le
		n^{-\epsP}.
	\end{equation}
\end{theorem}

The result in \cref{eq:cor_top_eigen_relative_error} can alternatively be understood in the language of `matrix trace correlation', which is defined for any full column-rank matrices $Y, Z \in \R^{n \times K}$ with corresponding projections $P_{Y} = Y(Y^{\tp}Y)^{-1}Y^{\tp}$ and $P_{Z} = Z(Z^{\tp}Z)^{-1}Z^{\tp}$ as
\begin{equation*}
	r_{K}(Y,Z)
	=
	\sqrt{\frac{\trace(P_{Y} P_{Z})}{K}},
	\quad \quad
	0 \le r_{K} \le 1.
\end{equation*}
In particular, weak consistency of $\widehat{\Theta}$ for $\Theta$, via $\widehat{U}$ for $U$, admits the equivalence
\begin{equation}
	\label{eq:cor_trace_correlation}
	\frac{\| \sampVecs\sampVecs^{\tp} - \popVecs\popVecs^{\tp} \|_{\frob}}{\| \popVecs\popVecs^{\tp} \|_{\frob}}
	=
	\ohp(1)
	\quad
	\underset{n \rightarrow \infty}{\iff}
	\quad
	r_{K}(\sampVecs,\popVecs)
	\overset{\Prob}{\rightarrow}
	1.
\end{equation}

In words, weak consistency establishes that on average, across the underlying graph, the node memberships are recovered with high probability for large $n$. At most a vanishing fraction of nodes are misclustered asymptotically with high probability, but this does not preclude that, say, $\Oh(\log n)$ nodes get misclustered as $n \rightarrow \infty$. Indeed, the results in \cref{sec:weak_consistency} do not guarantee that the membership of a specific, individual node of interest is necessarily perfectly recovered asymptotically with high probability.

\subsection{Node-specific strong consistency}
\label{sec:strong_consistency}
This section establishes that individual node memberships can be perfectly recovered asymptotically with high probability in blockmodels under certain conditions. The results established in \cref{sec:strong_consistency} require different tools and a more refined analysis. Proofs are provided in \cref{app:strong_consistency}.

Below, we obtain bounds for the individual row vectors of $\sampVecs - \popVecs W_{\star}$ where $W_{\star}$ denotes an appropriate $K \times K$ orthogonal matrix, and we use the subscript notation $\star$ to emphasize its special role. The rows of $\sampVecs$ correspond to the vector-valued embedding representations of the rows (columns) of the non-negative matrix $\nrk{A}$. Similarly, the rows of $\popVecs$ correspond to the population-level embedding representations of nodes in low-dimensional Euclidean space. The matrix $W_{\star}$ reflects the general non-uniqueness of $\sampVecs$ and $\popVecs$ as orthonormal bases for $K$-dimensional subspaces in $\R^{n}$. Accounting for the orthogonal transformation $W_{\star}$ is required and serves to properly align the embeddings.

For the purposes of this paper, `node-specific strong consistency' is said to hold for node $i$ if there exists a sequence of orthogonal transformations $W_{\star} \equiv W_{\star}^{(n)}$ such that, as $n \rightarrow \infty$,
\begin{equation}
	\label{eq:def_strong_consistency_bound}
	n_{g_{i}}^{1/2}\|\sampVecs - \popVecs W_{\star}\|_{i,\ell_{2}}
	=
	\ohp(1).
\end{equation}
Per \cref{sec:notation}, here $n_{g_{i}}$ denotes the size of block $g_{i} \in \Nset{K}$. Scaling by $n_{g_{i}}^{1/2}$ is crucial, since entrywise $\sampVecs_{ij} = \ohp(1)$ and $\|\popVecs\|_{i,\ell_{2}} = n_{g_{i}}^{-1/2}$. Omitting this term could otherwise trivially yield a bound satisfying $\|\sampVecs - \popVecs W_{\star}\|_{i,\ell_{2}} = \ohp(1)$ even if $\sampVecs$ were inconsistent for $\popVecs$.

\begin{theorem}[Node-specific strong consistency of robust spectral clustering, general case]
	\label[theorem]{result:strong_consistency_general}
	Suppose $A \sim \operatorname{Blockmodel}\blockmodel$ where $\rank(\Theta) = \rank(\rankspec) = K$ with $K \le n^{1/2}$ and $n_{k} \ge 5$ for all $k \in \Nset{K}$. Assume that $|\lambda_{K}| \equiv |\lambda_{K}(\rankspecBig)| \ge C \cdot n^{3/4+\epsP}$ for some sufficiently large constant $C > 0$ and some constant $0 < \epsP < 1/4$. Obtain $\nrk{A}$ using \cref{alg:ptr_baseline} and compute $\sampVecs$, a matrix whose orthonormal columns are eigenvectors for the $K$ largest-in-magnitude eigenvalues of $\nrk{A}$. Let $\popVecs$ denote the corresponding matrix of eigenvectors for $\rankspecBig$. Then, there exists $W_{\star}$ such that for any fixed choice of row index, $i$, it holds that
	\begin{align*}
		n_{g_{i}}^{1/2}\|\sampVecs-\popVecs W_{\star}\|_{i,\ell_{2}}
		&=
		\Ohp\Bigg(\Bigg\{ K^{1/2} \cdot n^{\epsP/2} \cdot |\lambda_{K}|^{-1}\\
		&\hspace{4em}+
		K^{3/2} \cdot n^{1+3\epsP/2} \cdot |\lambda_{K}|^{-3}\\
		&\hspace{4em}+
		\max\left\{n_{g_{i}}^{-1/4} \cdot n^{1/4}, K^{3/4} \right\} \cdot n_{g_{i}}^{-1/4} \cdot n^{3/4 + \epsP/2} \cdot |\lambda_{K}|^{-2}\\
		&\hspace{4em}+
		K^{1/2} \cdot n^{1+3/4+3\epsP/2} \cdot |\lambda_{K}|^{-3}\\
		&\hspace{4em}+
		K \cdot n_{g_{i}}^{-1/2} \cdot n^{1 +\epsP} \cdot |\lambda_{K}|^{-2}\Bigg\} \cdot n_{g_{i}}^{1/2}\Bigg).
	\end{align*}
\end{theorem}

\cref{result:strong_consistency_general} has the advantage of generality, nevertheless it appears somewhat unwieldy at first glance. Under certain additional assumptions commonly encountered in the literature, \cref{result:strong_consistency_general} simplifies to the following easier-to-interpret corollary.

\begin{corollary}[Node-specific strong consistency of robust spectral clustering, special case]
	\label[corollary]{result:strong_consistency_corollary}
	Assume the setting in \cref{result:strong_consistency_general}. Further assume that all block sizes $n_{k}$ are order $n/K$ and that $|\lambda_{K}(\rankspec)|$ is bounded above and below by positive absolute constants. Then, for each fixed choice of row index, $i$, it holds that
	\begin{equation*}
		n_{g_{i}}^{1/2}\|\sampVecs-\popVecs W_{\star}\|_{i,\ell_{2}}
		=
		\Ohp\left( K \cdot n^{-1/2+\epsP/2} \cdot \left\{1 + K^{2} \cdot n^{-1/4+\epsP}\right\}\right).
	\end{equation*}
	In particular, the hypotheses imply $K = \Oh(n^{1/4-\epsP})$, hence
	\begin{equation*}
		n_{g_{i}}^{1/2}\|\sampVecs-\popVecs W_{\star}\|_{i,\ell_{2}}
		=
		\ohp(1).
	\end{equation*}
\end{corollary}

\begin{remark}[Population spectral gap and blockmodel distributions]\label[remark]{rem:pop_gap_separation}
	The consistency results depend on the population eigenvalue $\lambda_{K}(\rankspec)$ which reflects properties of the block-specific distributions comprising $\mathcal{F}^{K} = \{F_{(k, k^{\prime})} : 1 \le k \le k^{\prime} \le K\}$. Unlike existing results in the literature, however, this population eigenvalue is a function of the block sizes and does not admit a simple closed-form expression in general. Consequently, for many choices of $\mathcal{F}^{K}$, it can only be evaluated numerically. Nevertheless, the relationship between $\lambda_{K}(\rankspec)$ and the blockmodel distributions can be succinctly exhibited in the following example. Fix $K \geq 2$, and suppose that $F_{(k,k)} = F_{(1,1)}$ for all $1 \le k \le K$ and $F_{(k, k^{\prime})} = F_{(1,2)}$ for all $k \neq k^{\prime}$. Let $X_{11}$ and $X_{12}$ denote independent random variables satisfying $X_{11} \sim F_{(1,1)}$ and $X_{12} \sim F_{(1,2)}$. If $n_{k} = n/K$ holds for all $1 \le k \le K$, then applying \cref{prop:ranks_multi_expectation} and the fact that $\Ex[F_{(1,1)}(X_{12})] = 1 - \Ex[F_{(1,2)}(X_{11})]$ yields
	\begin{align*}
		\lim_{n \rightarrow \infty} \lambda_{K}(\rankspec)
		&=
		\left(1 - \tfrac{1}{K}\right)\left(\Ex[F_{(1,2)}(X_{11})] - \tfrac{1}{2} \right)
		-
		\tfrac{1}{K}\left(\Ex[F_{(1,1)}(X_{12})] - \tfrac{1}{2} \right)\\
		&=
		\Ex[F_{(1,2)}(X_{11})] - \tfrac{1}{2}\\
		&=
		\Pr\left[X_{12} \le X_{11}\right] - \tfrac{1}{2}.
	\end{align*}
\end{remark}

\begin{remark}[Consistency under growing rank $K$ in \cref{result:strong_consistency_general,result:strong_consistency_corollary}]\label[remark]{remark:strong_consistency_comment}
	Although $\rankspec$ is a population-level quantity, the entries of $\rankspec$ change with $n$ (see \cref{prop:ranks_multi_expectation}), hence so does $|\lambda_{K}(\rankspec)|$ and $\|\rankspec^{-1}\|_{\op} = |\lambda_{K}(\rankspec)|^{-1}$. Further, $\rankspec \in [0,1]^{K \times K}$, so permitting $K$ to grow with $n$ in general implies that its spectrum depends on $K$. When seeking to permit $K_{n} \rightarrow \infty$ as $n\rightarrow\infty$, one could further require $\liminf_{n \rightarrow \infty} \{ \kappa_{K_{n}} \cdot |\lambda_{K_{n}}(\rankspec)|\} > 0$ for some sequence of positive scalars $\{\kappa_{K_{n}}\}_{n \ge 1}$, in which case the $\Ohp(\cdot)$ upper bound would inherit dependence on the behavior of $\kappa_{K_{n}}$. In short, fully general statements pertaining to $K_{n} \rightarrow \infty$ require extra care and are not pursued here.
\end{remark}

\subsection{Asymptotic normality}
\label{sec:asymptotic_normality}
This section begins with two supplementary lemmas. \cref{lem:asymp_part1} establishes univariate asymptotic normality for simple linear rank statistics appearing in block-structured data matrix settings. \cref{lem:asymp_part2} subsequently establishes multivariate asymptotic normality involving rank statistics when the population-level block structure is sufficiently balanced and non-degenerate. These lemmas enable \cref{thrm:asymp_norm}, the main result in this section, which establishes asymptotic normality of the leading eigenvector entries of $\nrk{A}$. Proofs are provided in \cref{app:asymptotic_normality}. Below, let $(\ranksdev^{\circ2})_{k k^{\prime}} = \{\Var[(\nrk{A})_{ij}] : g_{i}=k, g_{j}=k^{\prime}\}$.

\begin{lemma}[Asymptotic normality, part I]
	\label[lemma]{lem:asymp_part1}
	Fix $K$, $\mathcal{F}^{K}$, and fix a tuple $(\alpha,\beta)$. Suppose $A \sim \operatorname{Blockmodel}\blockmodel$ where $\rank(\Theta) = \rank(\rankspec) = K$. Write
	\begin{equation*}
		\AsyNormVar_{\alpha,\beta}^{(n)}
		=
		\Var\left[\left((\nrk{A} - \Ex[\nrk{A}]) \Theta\right)_{\alpha,\beta}\right]^{-1/2}
		\left((\nrk{A} - \Ex[\nrk{A}]) \Theta\right)_{\alpha,\beta}.
	\end{equation*}
	If both $\liminf_{n \rightarrow \infty}\min_{k,k^{\prime} \in \Nset{K}} (\ranksdev^{\circ2})_{k k^{\prime}} \ge \delta_{\operatorname{Var}}$ for some constant $\delta_{\operatorname{Var}} > 0$ and $n_{\beta} \rightarrow \infty$, then for any fixed choice of $\epsilon > 0$, it eventually holds for large $n$ that
	\begin{equation}
		\max_{x \in \R} \left| \Pr\left[\AsyNormVar_{\alpha,\beta}^{(n)} \le x \right] - \Phi_{1}(x)\right| < \epsilon,
	\end{equation}
	where $\Phi_{1}$ denotes the cumulative distribution function of the standard normal distribution. Namely, $\AsyNormVar_{\alpha,\beta}^{(n)}$ converges in distribution to a standard normal random variable.
\end{lemma}

\begin{lemma}[Asymptotic normality, part II]
	\label[lemma]{lem:asymp_part2}
	Fix $K$, $\mathcal{F}^{K}$, and fix a row index $\alpha$. Suppose $A \sim \operatorname{Blockmodel}\blockmodel$ where $\rank(\Theta) = \rank(\rankspec) = K$. Let $\pi = (\pi_{1},\dots,\pi_{K})^{\tp}$ be a probability vector of constants with $0 < \pi_{1}, \dots, \pi_{K} < 1$ and $\sum_{k=1}^{K}\pi_{k} = 1$ such that $n_{k} = n \times \pi_{k}$ for each $k \in \Nset{K}$. Assume that $\liminf_{n \rightarrow \infty}\min_{k,k^{\prime} \in \Nset{K}} (\ranksdev^{\circ2})_{k k^{\prime}} \ge \delta_{\operatorname{Var}}$ for some constant $\delta_{\operatorname{Var}} > 0$.
	Write
	\begin{equation*}
		\AsyNormVar_{\textnormal{row } \alpha}^{(n)}
		=
		\Cov\left[\left((\nrk{A} - \Ex[\nrk{A}]) \Theta\right)_{\textnormal{row } \alpha}\right]^{-1/2}\left((\nrk{A} - \Ex[\nrk{A}]) \Theta\right)_{\textnormal{row } \alpha},
	\end{equation*}
	while treating each occurrence of row $\alpha$ as a column vector. For any fixed choice $\epsilon > 0$, it eventually holds for large $n$ that
	\begin{equation}
		\max_{\vec{x} \in \R^{K}} \left| \Pr\left[ \AsyNormVar_{\textnormal{row } \alpha}^{(n)} \le \vec{x} \right] - \Phi_{K}(\vec{x})\right|
		<
		\epsilon,
	\end{equation}
	where $\Phi_{K}$ denotes the cumulative distribution function of the $K$-dimensional standard normal distribution. Namely, $\AsyNormVar_{\textnormal{row } \alpha}^{(n)}$ converges in distribution to a multivariate standard normal random vector.
\end{lemma}

Applying the aforementioned lemmas with a careful spectral perturbation analysis yields the following theorem.

\begin{theorem}[Asymptotic multivariate normality of robust spectral embedding]
	\label[theorem]{thrm:asymp_norm}
	Assume the setting and hypotheses in \cref{lem:asymp_part2}. Further suppose that
	\begin{align*}
		\Gamma^{(k)}
		&=
		\lim_{n \rightarrow\infty} \Cov\left[n^{-1/2}((\nrk{A} - \Ex[\nrk{A}])\Theta)_{\textnormal{row } i} : g_{i}=k\right],\\
		\Xi
		&=
		\lim_{n \rightarrow\infty} \left\{n^{1/2}(\Theta^{\tp}\Theta)^{-1/2} \right\} \cdot \rankspec^{-1} \cdot \left\{n(\Theta^{\tp} \Theta)^{-1}\right\},
	\end{align*}
	where each limit matrix is rank $K$. There exist sequences of $K \times K$ orthogonal matrices $W_{\star} \equiv W_{\star}^{(n)}, W \equiv W^{(n)}$, indexed by $n$, such that as $n\rightarrow\infty$, the sequence of random vectors of the form
	\begin{equation*}
		\left(n\left(\sampVecs - \popVecs W_{\star}\right) W\right)_{\textnormal{row } i}
		:
		g_{i}=k
	\end{equation*}
	converges in distribution to a $K$-dimensional multivariate normal random vector with mean zero and covariance matrix $\Xi \cdot \Gamma^{(k)} \cdot \Xi^{\tp}$.
\end{theorem}

\cref{thrm:asymp_norm} establishes asymptotic normality of the leading eigenvector components of $\nrk{A}$ by leveraging the perturbation analysis used to establish \cref{result:strong_consistency_general}. At a high level,
\begin{align*}
	n \cdot (\sampVecs - \popVecs W_{\star})_{\operatorname{row} i}
	&\approx
	n \cdot \left\{\left(\nrk{A}-\Ex[\nrk{A}]\right) \Theta (\Theta^{\tp} \Theta)^{-1} \rankspec^{-1} (\Theta^{\tp}\Theta)^{-1/2} W^{\tp}\right\}_{\operatorname{row} i}.
\end{align*}
In particular, for large $n$, the individual rows of the embedding exhibit the block-conditional Gaussian approximation
\begin{equation}
	\label{eq:informal_normality_result}
	n^{1/2} \cdot \sampVecs_{\operatorname{row} i}
	:
	g_{i} = k
	\quad
	\overset{!}{\approx}
	\quad
	\vec{Z}_{k}
	\sim
	\mathcal{N}_{K}(\operatorname{mean}(k),\operatorname{cov}(k)/n)
\end{equation}
involving block-specific mean vectors and covariance matrices. \cref{eq:informal_normality_result} is useful because in practice, data embeddings are commonly plotted all at once, e.g.,~the rows of $n^{1/2} \cdot \sampVecs$ are plotted as a point cloud in $\R^{K}$. Nevertheless, \cref{eq:informal_normality_result} should not be regarded as a substitute for the stated result, hence we emphasis caution through the use of an exclamation mark.

In \cref{thrm:asymp_norm}, the scaling is $n$, rather than $n^{1/2}$. This emerges due to the delocalization behavior of $\sampVecs$ and $\popVecs$ in our blockmodel setting. Loosely speaking, a factor of $n^{1/2}$ is needed to stabilize the rows of $\sampVecs$ and then an additional factor of $n^{1/2}$ is needed to establish stability of the fluctuations giving rise to convergence in distribution.

In the numerical examples that follow, in accordance with common practice in the literature, we plot the rows of $n^{1/2} \cdot \wh{U}$ all at once given $A$. Collectively, these embedding vectors plausibly approximately resemble a mixture distribution of Gaussian components, though we emphasize that formally \cref{thrm:asymp_norm} holds only for individual vectors, not for the entire point cloud all at once.

We emphasize that \cref{lem:asymp_part1,lem:asymp_part2,thrm:asymp_norm} do not assume that any of the data are normally distributed. Rather, asymptotic normality emerges due to the interplay of having commensurate block sizes, being able to apply classical asymptotic normality results for simple linear rank statistics, and careful entry-wise and row-wise eigenvector perturbation analysis.

\subsection{Selecting the embedding dimension}
\label{subsec:selecting_number_of_communities}
Throughout the main results, $\rank(\rankspec)$ is assumed known and equals $K$, the true number of blocks or communities in the data. To facilitate discussion, define $d = \rank(\rankspec)$. If $d$ is unknown, then it is possible to devise a selection procedure based on an eigenvalue gap in the scree plot of $\nrk{A}$. Specifically, let $\wh{d} \equiv \wh{d}_{n}$ denote the number of eigenvalues of $\nrk{A}$ whose magnitude exceeds $4n^{3/4 + \epsP}$ for any user-specified constant $0 < \epsP < 1/4$. \cref{lem:selectK} provides support for this choice of selection procedure and is proved in \cref{app:additional_proofs}.
\begin{lemma}[Selecting the embedding dimension via the scree plot of $\nrk{A}$]
	\label[lemma]{lem:selectK}
	Suppose $A \sim \operatorname{Blockmodel}\blockmodel$ where $\rank(\Theta) = \rank(\rankspec) = K$ and $K = \Oh(1)$. Further suppose that $n_{\min}$ is order $n/K$, and assume that $\liminf_{n \rightarrow\infty}|\lambda_{K}(\rankspec)| \ge \delta$ for some $\delta > 0$. Fix $0<\epsP<1/4$. Then,
	\begin{align*}
		\lim_{n\rightarrow\infty}\Pr\left[|\lambda_{k}(\nrk{A})| > 4n^{3/4 + \epsP}\right]
		= 1
		&\quad \textnormal{ for each } 1 \le k \le K;\\
		\lim_{n\rightarrow\infty}\Pr\left[|\lambda_{K+i}(\nrk{A})| > 4n^{3/4 + \epsP}\right]
		 = 0
		 &\quad \textnormal{ for each } i \ge 1.
	\end{align*}
\end{lemma}

In simulations, the above rule tends to under-select for small-to-moderate values of $n$. This happens in part because the absolute constant $4$ has been chosen for convenience and is not sharp. Moreover, the optimal rate is anticipated to be closer to $n^{1/2}$, not $n^{3/4}$. Additional empirical explorations suggest that $1.001 \cdot n^{1/2}$ is often adequate. For example, this threshold choice correctly selects $\widehat{d} = 2$ in \cref{sec:example_illustration}.

\section{Numerical examples}
\label{sec:numerical}
This section provides numerical examples that illustrate the theoretical results in \cref{sec:main_results} and suggest opportunities for future work.

\subsection{Heterogeneous heavy-tailed data}
\label{sec:numerical_pareto}
Consider a $K=2$ blockmodel per \cref{def:data_matrices} with block-specific distributions given by
\begin{equation*}
	F_{(1,1)} = \textsf{Pareto}(m_{1},1),~~
	F_{(1,2)} = \textsf{Pareto}(m_{2},2),~~
	F_{(2,2)} = \textsf{Pareto}(m_{3},3).
\end{equation*}
Of note, $\xi \sim \textsf{Pareto}(m,\gamma)$ with $m, \gamma > 0$ exhibits the tail behavior
\begin{equation*}
	\Pr[\xi > x] =
	\begin{cases}
		\left(\frac{m}{x}\right)^{\gamma}
		& \textnormal{if }x \ge m,\\
		1
		& \textnormal{if }x < m.
	\end{cases}
\end{equation*}
Furthermore,
\begin{equation*}
	\Ex[\xi] =
	\begin{cases}
		\frac{\gamma \cdot m}{\gamma - 1} & \textnormal{if } \gamma > 1, \\
		\infty & \textnormal{if } \gamma \le 1; \\	
	\end{cases}
	\quad
	\text{~and~}
	\quad
	\Var[\xi] =
	\begin{cases}
		\frac{\gamma \cdot m^{2}}{(\gamma-1)^{2}(\gamma-2)} & \textnormal{if } \gamma > 2,\\
		\infty & \textnormal{if } \gamma \le 2.
	\end{cases}
\end{equation*}
Moreover, the first $T$ moments of $\textsf{Pareto}(m,\gamma)$ exist (i.e.,~are finite) only when $\gamma > T$. Consequently, heavy-tailedness can be present and inherent to the blockmodel here. There is no concept of contamination or outliers per se.

In this model, the entry $\meanmtx_{11}$ of $\meanmtx$ does not necessarily exist. In contrast, denoting block sizes by $n_{1} = \pi_{1} \times n$, $n_{2} = (1-\pi_{1}) \times n$, and considering the large-$n$ limit yields an explicit expression for $\rankspec_{11}$ as follows.
{\small
	\begin{equation*}
		\lim_{n \rightarrow \infty} \rankspec_{11}
		=
		\frac{1}{2} \pi_{1}^{2} + 2(1-\pi_{1})\pi_{1} \times
		\begin{cases}
			\frac{2m_{1}}{3m_{2}}
			& \textnormal{if } m_{1} < m_{2}\\[10pt]
			\frac{3m_{1}^{2}-m_{2}^{2}}{3m_{1}^{2}}
			& \textnormal{else}
		\end{cases}
		+ (1-\pi_{1})^{2} \times
		\begin{cases}
			\frac{3m_{1}}{4m_{3}}
			& \textnormal{if } m_{1} < m_{3}\\[10pt]
			\frac{4m_{1}^{3}-m_{3}^{3}}{4m_{1}^{3}}
			& \textnormal{else}
		\end{cases}
	\end{equation*}
}
In fact, all limiting entries of both $\rankspec$ and $\ranksdev$ admit closed-form expressions but are omitted for brevity.

We simulate a single hollow data matrix $A \in \R^{n \times n}$ with $n=400$ and $\pi_{1}=0.25$ and $m_{i}=i$ for $1 \le i \le 3$. Consequently, for observables in block $(1,1)$, neither the mean nor the variance exists; for observables in block $(1,2)$, the mean exists but not the variance; for observables in block $(2,2)$, both the mean and the variance exist but without higher moments existing. The raw data matrix and matrix of normalized ranks are shown in \cref{fig:pareto_raster}. Here, the membership matrix $\Theta \in \{0,1\}^{400 \times 2}$ is deliberately ordered for visual convenience, hence the latent block structure is apparent, yet we emphasize that $\Theta$ is neither observed nor ordered in practice. In \cref{fig:pareto_embed} we deliberately over-embed into four-dimensional space, plotting the rows of $n^{1/2}\cdot\wh{U}_{A}$ and $n^{1/2}\cdot\wh{U}_{\nrk{}}$. In Panel~A, the embedding of the raw data is manifestly degenerate and of no use for clustering. Notice the visible `localization' of the eigenvector components; this suggests that regularization or normalization may be needed for spectral clustering. In Panel~B, perfect clustering is manifest via the first and second coordinate dimensions, even for this small value of $n$. Further, approximate multivariate normality of each block-specific point cloud is plausible in Panel~B of \cref{fig:pareto_embed}.

\begin{figure}[tp]
	\centering
	\begin{subfigure}{0.95\linewidth}
		\centering
		\includegraphics[width=0.9\linewidth]{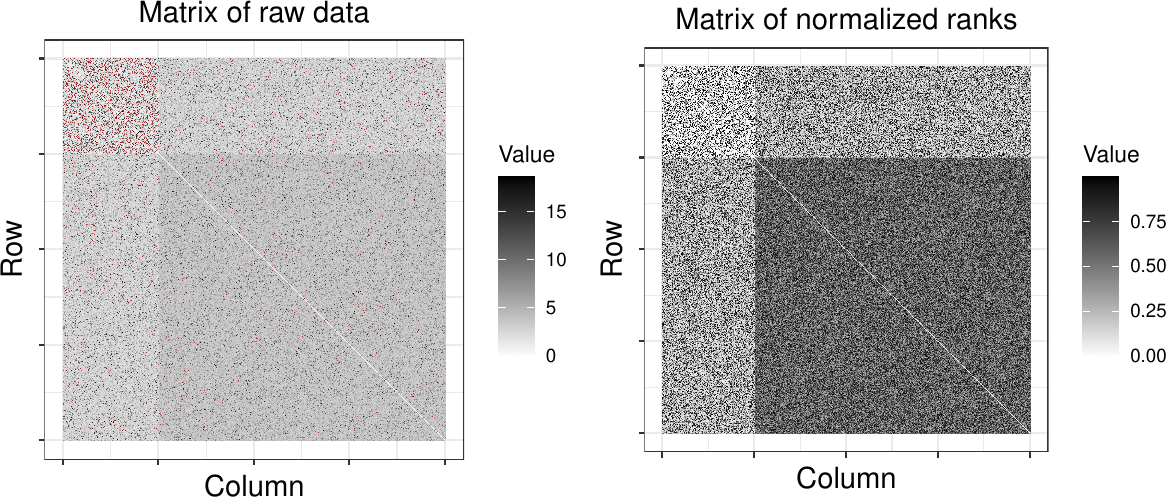}
	\end{subfigure}
	\caption{Simulation example in \cref{sec:numerical_pareto}. Row and column indices indicate corresponding node labels. Nodes are ordered to illustrate approximate block structure. Brown dots on the left represent values of the original data below the $1\%$ percentile or above the $99\%$ percentile; these values corrupt the raw data embedding in \cref{fig:pareto_embed}. Here, the matrix of normalized ranks reflects $K=2$ and satisfies the conditions for consistency in \cref{sec:main_results}.}
	\label{fig:pareto_raster}
\end{figure}

\begin{figure}[tp]
	\centering
	\begin{subfigure}{0.95\linewidth}
		\centering
		\includegraphics[width=0.95\linewidth]{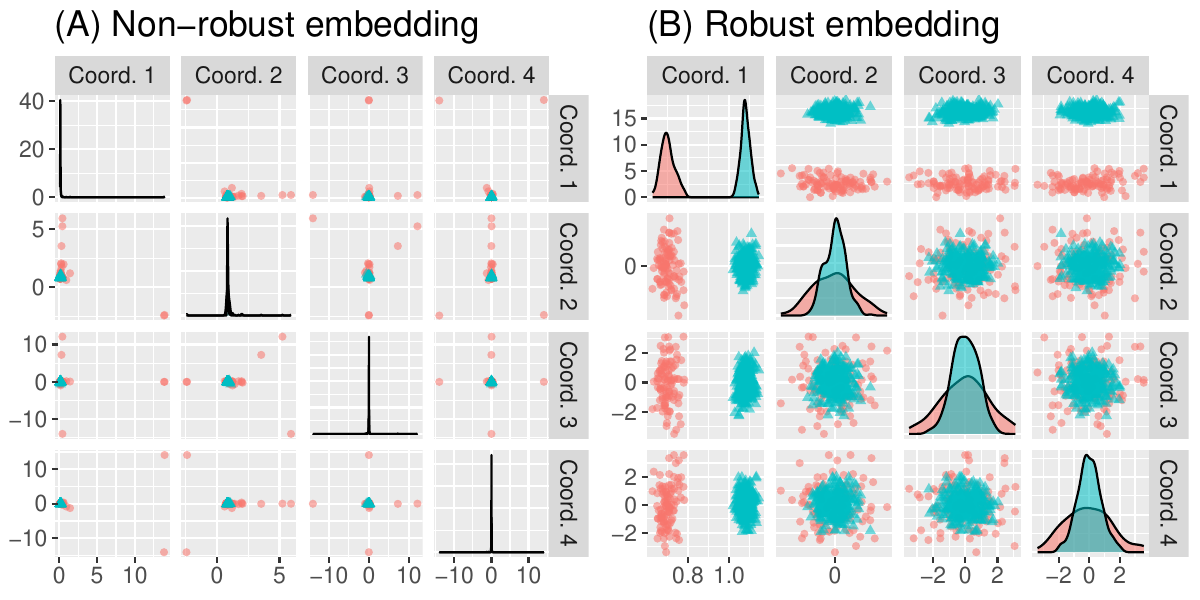}
	\end{subfigure}
	\caption{Four-dimensional embedding of the raw data (Panel~A) and after passing to ranks (Panel~B) in \cref{sec:numerical_pareto}. Coordinates correspond to the four leading eigenvectors after scaling. The underlying population-level true dimension equals two. Colors and point shapes both denote unobserved block memberships. Perfect clustering is apparent in the first column of Panel~B displaying the robust rank-based embedding.}
	\label{fig:pareto_embed}
\end{figure}

\subsection{Robustness and stability in low-dimensional embeddings}
\label{sec:numerical_overlaid_pointclouds}
This subsection illustrates that pass-to-ranks spectral embedding exhibits robust dimensionality reduction properties when heavy-tailed contamination is introduced to original data whose embedding is itself non-degenerate.

We sample a data matrix $A \in \R^{n \times n}$ from a $K=3$ blockmodel per \cref{def:data_matrices} comprised of block-specific Gaussian distributions, $F_{(k,k^{\prime})} = \textsf{Normal}(\mu_{kk^{\prime}},\sigma_{kk^{\prime}}^{2})$ for $1 \le k \le k^{\prime} \le 3$. Specifically, we consider the balanced setting, $n_{1} = n_{2} = n_{3} = n/K$, with $n=1800$ and parameter values $\mu_{11}=2$, $\mu_{12}=0.5$, $\mu_{13}=1.5$, $\mu_{22}=1$, $\mu_{23}=2.5$, $\mu_{33}=0.5$ and $\sigma_{11}=8$, $\sigma_{12}=2$, $\sigma_{13}=5$, $\sigma_{22}=2$, $\sigma_{23}=4$, $\sigma_{33}=3$. Note that the variance profile is block-wise heteroskedastic. Further, one can directly proceed with eigenvector-based spectral clustering applied to $A$, denoted as `Original' in \cref{fig:overlaid_pointclouds}.

Next, we perturb $A$ by replacing each upper-triangular entry $A_{ij}$ with probability $0.01$ by an independent draw from $\textsf{Cauchy}(\mu_{g_{i}g_{j}}, 1)$, yielding the stochastic matrix $A^{\prime}$ (preserving symmetry). We then apply \cref{alg:ptr_baseline,alg:ptr_clustering} to $A^{\prime}$, yielding the after-perturbation robust `PTR' point clouds in \cref{fig:overlaid_pointclouds}. Notice the close agreement between them. On the other hand, an embedding of $A^{\prime}$ (not shown) would exhibit the degenerate, uninformative behavior similar to Panel~A in \cref{fig:pareto_embed}.

\begin{figure}[tp]
	\centering
	\includegraphics[width=0.55\linewidth]{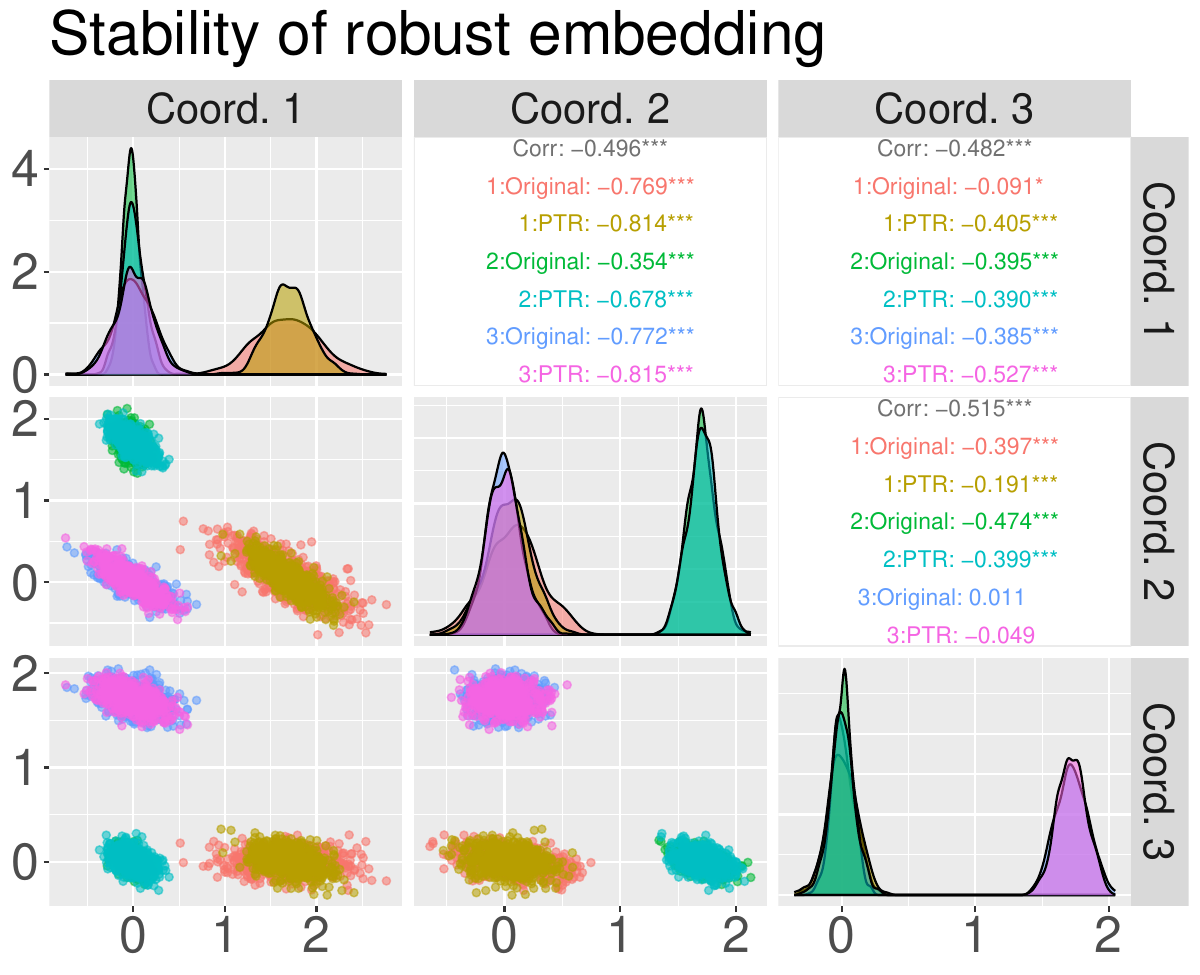}
	\caption{Shown are overlaid embedding point clouds derived from $A$ and $\nrk{A^{\prime}}$ in \cref{sec:numerical_overlaid_pointclouds}, where $A^{\prime}$ is corrupted by Cauchy entries. Coordinates correspond to the three leading eigenvectors after rotation and scaling. The above-diagonal panels show correlations for cluster-specific embeddings, asterisks indicate their level of statistical significance under the null hypothesis of zero correlation, and cluster labels are denoted by `1', `2', `3'. Using normalized rank statistics preserves properties of the original embedding.}
	\label{fig:overlaid_pointclouds}
\end{figure}

\subsection{Comparing asymptotic variances and covariances}
\label{sec:numerical_asymp_covariance}
Consider a $K=2$ blockmodel per \cref{def:data_matrices} comprised of the distributions
\begin{equation*}
	F_{(1,1)} = \textsf{Normal}(\mu_{1},\sigma_{1}^{2}), \quad
	F_{(1,2)} = \textsf{Normal}(\mu_{2},\sigma_{2}^{2}), \quad
	F_{(2,2)} = \textsf{Normal}(\mu_{3},\sigma_{3}^{2}),
\end{equation*}
specifically with $\mu_{1} = \mu_{3}$, $\mu_{1} \neq \mu_{2}$, and in the balanced regime with $n_{1} = n_{2} = n/K$. For this setting, the proof of \cref{thrm:asymp_norm} can be modified to show that for $\wh{U}_{A}$, there exists a sequence of $2 \times 2$ matrices $T^{(n)}$ such that as $n \rightarrow \infty$,
\begin{equation*}
	\sqrt{n}\left(T^{(n)} \cdot \wh{U}_{A, \operatorname{row }i} - \left[\begin{matrix} \phantom{-}1 \\ -1 \end{matrix}\right] \right)
	:
	g_{i} = 2
	\quad
	\overset{\operatorname{d}}{\longrightarrow}
	\quad
	\mathcal{N}_{2}(0,\Sigma^{(2)}),
\end{equation*}
where the asymptotic covariance matrix $\Sigma^{(2)}$ has upper-triangular entries $\Sigma_{11}^{(2)} = \frac{2(\sigma_{2}^{2}+\sigma_{3}^{2})}{(\mu_{2}+\mu_{3})^{2}}$ and $\Sigma_{12}^{(2)} = \frac{2(\sigma_{2}^{2}-\sigma_{3}^{2})}{\mu_{2}^{2}-\mu_{3}^{2}}$ and $\Sigma_{22}^{(2)} = \frac{2(\sigma_{2}^{2}+\sigma_{3}^{2})}{(\mu_{2}-\mu_{3})^{2}}$. Here, the sign difference in the second coordinate is responsible for distinguishing the first and second blocks, just as in \cref{fig:intro_plots}.

\begin{figure}[tp]
	\centering
	\begin{subfigure}{0.48\linewidth}
		\centering
		\includegraphics[width=0.95\linewidth]{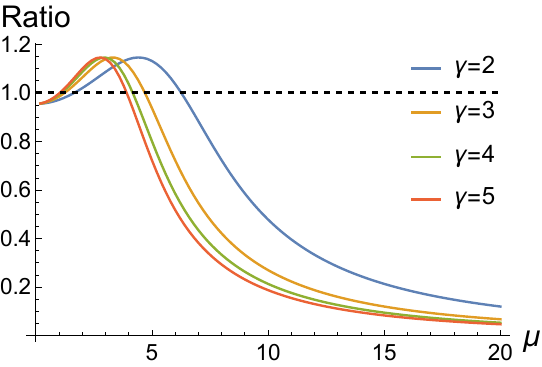}
	\end{subfigure}
	\nolinebreak
	\begin{subfigure}{0.48\linewidth}
		\centering
		\includegraphics[width=0.95\linewidth]{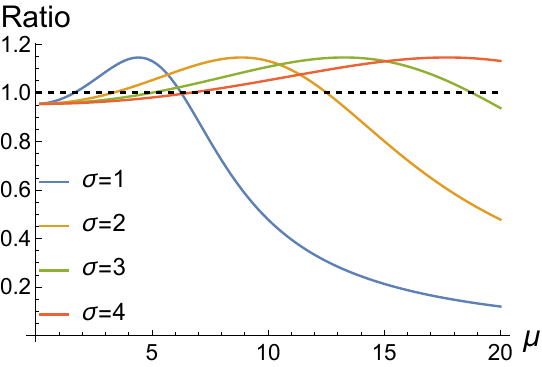}
	\end{subfigure}
	\caption{The ratio of asymptotic variances per \cref{eq:cluster_ARE_ratio} is plotted as a function of underlying model parameters $(\mu, \gamma)$ and $(\mu, \sigma)$, respectively. Ratio values larger than one (i.e.,~above the dashed black line) indicate asymptotic, theoretical grounds on which to favor robust spectral clustering, and vice versa. Plotted curves are obtained by the numerical evaluation of analytically derived functions, not by simulation. See \cref{sec:numerical_asymp_covariance} for details.}
	\label{fig:ARE_curves}
\end{figure}

Using \cref{thrm:asymp_norm}, we obtain a similar asymptotic normality result but for $\wh{U}_{\nrk{}}$, again centered around $(1, -1)^{\tp}$ but whose asymptotic covariance $\widetilde{\Sigma}^{(2)} \equiv \widetilde{\Sigma}^{(2)}(\mathcal{F},\Theta)$ is a function of $\rankspec$ and $\ranksdev$ rather than $\meanmtx$ and $\stdevmtx$. Consequently, using the concept of asymptotic relative efficiency (ARE), it becomes possible to compare pass-to-ranks spectral clustering with adjacency-based spectral clustering for large $n$ via the ratio of asymptotic variances
\begin{equation}
	\label{eq:cluster_ARE_ratio}
	\Sigma^{(2)}_{22}(\mathcal{F},\Theta)
	\Big/
	\widetilde{\Sigma}^{(2)}_{22}(\mathcal{F},\Theta).
\end{equation}
Small values of \cref{eq:cluster_ARE_ratio} indicate that non-robust spectral clustering with $A$ is relatively more efficient in the present model, whereas larger values of \cref{eq:cluster_ARE_ratio} indicate that pass-to-ranks robust spectral clustering with $\nrk{A}$ is comparatively more efficient.

\cref{fig:ARE_curves} plots the ratio in \cref{eq:cluster_ARE_ratio} for two different settings. In the left panel, the relationship between the block means is varied. In the right panel, the shared variance is varied. Specifically, $\mu = \mu_{1} = \mu_{3}$ and $\gamma = \mu/\mu_{2}$ and $\sigma = \sigma_{1} = \sigma_{2} = \sigma_{3}$. In the left panel, $\mu$ and $\gamma$ vary with $\sigma=1$, while in the right panel, $\mu$ and $\sigma$ vary with $\gamma=2$. The blue curve is identical in both plots.

The behavior in \cref{fig:ARE_curves} is instructive. All else equal, larger values of $\mu$ correspond to stronger signal strength vis-\`{a}-vis separation of cluster centroids, behavior that is directly reflected in $A$ but diminished in $\nrk{A}$. Transforming to normalized ranks typically reduces both block-wise variances and expectations, hence its overall effect is not always immediately clear, but \cref{fig:ARE_curves} shows robust spectral clustering with rank statistics can boost signal particularly when the original block-wise variances are large.

\subsection{Matrix rank versus number of communities}
\label{sec:rank_versus_blocks}
In the present paper, our theoretical results impose no explicit assumption on $\rank(\meanmtx)$ but rather require $\rank(\rankspec)=K$. This requirement is generally milder since it is often satisfied for entrywise rank transformations provided the block-specific distributions are not identical or symmetric about the same location parameter. See the explicit, finite-$n$ expression for the entries of $\rankspec$ in \cref{app:rank_statistic_properties}. Below, we provide details for two specific example settings. See also \cref{subsec:dimension_versus_number_of_communities} for more discussion.

\subsubsection{Example with block-specific Gamma distributions}
Consider the $\textsf{Gamma}(\alpha,\theta)$ distribution having expectation $\alpha \times \theta$ with shape parameter $\alpha > 0$ and scale parameter $\theta > 0$. Consider a $K=2$ blockmodel per \cref{def:data_matrices} where $F_{(1,1)} = \textsf{Gamma}(3,1/3)$ and $F_{(1,2)} = \textsf{Gamma}(2,1/2)$ and $F_{(2,2)} = \textsf{Gamma}(1,1)$ and $n_{1} = n_{2} = n/K$. Consequently, the blockwise mean and variance matrices are given by
\begin{equation*}
	\meanmtx =
	\begin{bmatrix}
		1 & 1 \\
		1 & 1
	\end{bmatrix},
	\quad
	\stdevmtx^{\circ 2} \doteq
	\begin{bmatrix}
		.333 & .5\\
		.5 & 1
	\end{bmatrix},
\end{equation*}
with $\rank(\meanmtx)=1$. On the other hand, for large $n$, the blockwise means and variances after converting to ranks are given by
\begin{equation*}
	\rankspec \doteq
	\begin{bmatrix}
		.531 & .507 \\
		.507 & .452
	\end{bmatrix},
	\quad
	\ranksdev^{\circ 2} \doteq
	\begin{bmatrix}
		.0615 & .0790\\
		.0790 & .110\phantom{0}
	\end{bmatrix},
\end{equation*}
with $\rank(\rankspec) = 2$ and $|\lambda_{K}(\rankspec)| \doteq .0169$.

\subsubsection{Example with block-specific Exponential distributions}
Consider the $\textsf{Exponential}(\theta)$ distribution with expectation $\theta > 0$. Consider a $K=2$ blockmodel per \cref{def:data_matrices} where $F_{(1,1)}, F_{(1,2)}, F_{(2,2)}$ are each $\textsf{Exponential}(\cdot)$ such that
\begin{equation*}
	\meanmtx =
	\begin{bmatrix}
		\mu_{1}^{2} & \mu_{1}\mu_{2}\\
		\mu_{1}\mu_{2} & \mu_{2}^{2}
	\end{bmatrix}
	= \begin{bmatrix}
		\mu_{1} \\ \mu_{2}
	\end{bmatrix}
	\begin{bmatrix}
		\mu_{1} & \mu_{2}
	\end{bmatrix}, \quad
	\stdevmtx^{\circ 2} = \meanmtx \circ \meanmtx,
\end{equation*}
where without loss of generality $\mu_{1} \ge \mu_{2} > 0$. Further suppose $n_{1} = \pi_{1} \times n$ and $n_{2} = (1-\pi_{1}) \times n$ for some $0 < \pi_{1} < 1$, hence $\rank(\Theta)=2$. In the large-$n$ limit, the entries of $\rankspec$ are given by the closed-form expressions
\begin{align*}
	\lim_{n\rightarrow\infty}\rankspec_{11}
	&=
	\frac{1}{2}\pi_{1}^{2}
	+ \frac{2\mu_{1}}{\mu_{1}+\mu_{2}}\pi_{1}(1-\pi_{1})
	+ \frac{\mu_{1}^{2}}{\mu_{1}^{2}+\mu_{2}^{2}}(1-\pi_{1})^{2},\\
	\lim_{n\rightarrow\infty}\rankspec_{12}
	&=
	\frac{\mu_{2}}{\mu_{1}+\mu_{2}}\pi_{1}
	+ \frac{\mu_{1}}{\mu_{1}+\mu_{2}}(1-\pi_{1}),\\
	\lim_{n\rightarrow\infty}\rankspec_{22}
	&=
	\frac{\mu_{2}^{2}}{\mu_{1}^{2}+\mu_{2}^{2}}\pi_{1}^{2}
	+ \frac{2\mu_{2}}{\mu_{1}+\mu_{2}}\pi_{1}(1-\pi_{1})
	+ \frac{1}{2}(1-\pi_{1})^{2}.
\end{align*}
Closed-form expressions are also available for the entries of $\ranksdev^{\circ 2}$ in the large-$n$ limit but are substantially more complicated and therefore omitted for brevity.

Observe that the community membership matrix has size $n \times 2$ and rank $2$, while the $2 \times 2$ matrix $\meanmtx$ of connectivity probabilities has rank one. In this case, the leading eigenvector of the raw data separates blocks, whereas the leading two eigenvectors of the pass-to-ranks matrix discriminate between the blocks.

\subsection{Extensions and complements}
\label{sec:numerical_extension}
This section concludes with empirical findings that serve to inspire future work.

\subsubsection{Towards optimal rates and constants}
\label{sec:numerical_extension_rates_constants}
The results in \cref{sec:main_results} are seemingly the first of their kind to establish consistency results for spectral clustering applied to matrices of normalized rank statistics. Still, there is room for improvement. We conjecture that for some model settings with $K$ fixed, $\|(\sampVecs\sampVecs^{\tp})_{\nrk{}} - \popVecs\popVecs^{\tp}\|_{\frob}^{2}$ converges in probability to a non-zero constant as $n \rightarrow \infty$ and that this occurs at the same scaling as the convergence of $\|(\sampVecs\sampVecs^{\tp})_{A} - \popVecs\popVecs^{\tp}\|_{\frob}^{2}$. This conjecture posits an improvement in the rate obtained in \cref{result:cor_top_eigen_relative_error} and is motivated by existing results which establish the concentration of $\|(\sampVecs\sampVecs^{\tp})_{A} - \popVecs\popVecs^{\tp}\|_{\frob}^{2}$ and closely related expressions for certain symmetric data matrices $A$ having independent entries with sufficiently light tails (e.g.,~\cite{tang2017semiparametric,tang2018limit,xu2018rates}). It remains to more deeply understand the asymptotic, theoretical distinction between using $A$ versus $\nrk{A}$.

To briefly investigate this conjecture, consider a $K=2$ blockmodel per \cref{def:data_matrices} whose data matrix entries are drawn from contaminated normal distributions of the form
\begin{equation*}
	F_{(k,k^{\prime})}
	=
	(1-\epsilon) \cdot \textsf{Normal}(\mu_{kk^{\prime}}, \sigma_{kk^{\prime}}^{2})
	+
	\epsilon \cdot \textsf{Normal}(\mu_{kk^{\prime}}, \tau \cdot \sigma_{kk^{\prime}}^{2}),
	\quad\quad
	1 \le k \le k^{\prime} \le 2.
\end{equation*}
Set $\mu_{11} = \mu_{22} = 6$, $\mu_{12} = 4$, and $\sigma_{11} = \sigma_{12} = \sigma_{22} = 0.5$. We vary $\tau > 0$, the variance multiplier in the contaminating distribution, and $0 \le \epsilon \le 1$, the contamination level.

\begin{figure}[tp]
	\centering
	\includegraphics[width=0.5\linewidth]{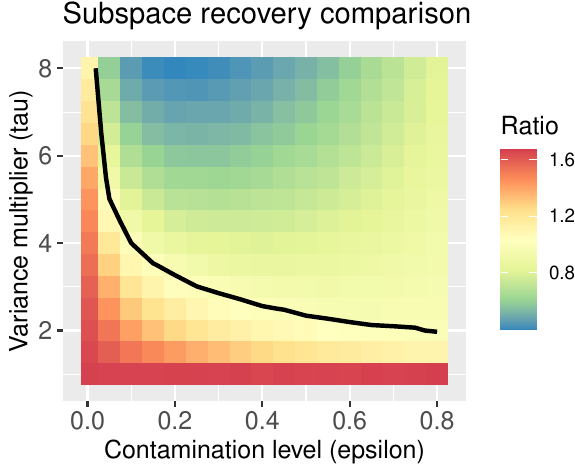}
	\caption{Monte Carlo approximation of the theoretical ratio given by \cref{eq:contour_ratio} in the contaminated normal distribution example. The solid black curve illustrates parameter values $(\epsilon,\tau)$ for which the approximated ratio equals one, namely for which there is identical asymptotic performance as quantified by limiting leading constants. Robust spectral clustering with $\nrk{A}$ outperforms traditional spectral clustering with $A$ in general for large values of contamination and (or) large values of scale in the contaminating distribution.}
	\label{fig:contour_ratio}
\end{figure}

We generate $n_{b}=50$ i.i.d.~replicates $A^{(1)},\dots,A^{(n_{b})} \in \R^{n \times n}$, $n=500$, each having equal block sizes $n_{k} = n/K$ for $1 \le k \le K$. The contour plot in \cref{fig:contour_ratio} shows 
\begin{equation*}
	\frac{1}{n_{b}} \sum_{b=1}^{n_{b}} \left\|\left(\sampVecs\sampVecs^{\tp}\right)_{\nrk{}}^{(b)} - \popVecs\popVecs^{\tp}\right\|_{\frob}^{2}
	\, \Big/ \,
	\frac{1}{n_{b}} \sum_{b=1}^{n_{b}} \left\|\left(\sampVecs\sampVecs^{\tp}\right)_{A}^{(b)} - \popVecs\popVecs^{\tp}\right\|_{\frob}^{2}
\end{equation*}
which is an empirical approximation to the theoretical ratio
\begin{equation}
	\label{eq:contour_ratio}
	\frac{\Ex\left[\|(\wh{U}\wh{U}^{\tp})_{\nrk{}} - UU^{\tp}\|_{\frob}^{2}\right]}{\Ex\left[\|(\wh{U}\wh{U}^{\tp})_{A} - UU^{\tp}\|_{\frob}^{2}\right]}
	= \frac{1-\Ex\left[r_{K}^{2}(\wh{U}_{\nrk{}},U)\right]}{1-\Ex\left[r_{K}^{2}(\wh{U}_{A},U)\right]}
\end{equation}
as a function of the parameter tuple $(\epsilon,\tau)$. Values of \cref{eq:contour_ratio} smaller than one indicate the pass-to-ranks leading eigenvectors estimate the population eigenvectors better than the leading eigenvectors of the raw data matrix, whereas the opposite is true for values of \cref{eq:contour_ratio} larger than one. In this example, neither approach dominates the other.

The empirical ratio values illustrated in \cref{fig:contour_ratio} remain stable for larger values of $n$ (not shown), namely they do not appreciably change, supporting the above conjecture. In future work, it would be interesting to derive an explicit closed-form expression for the limiting ratio of leading constants suggested by \cref{fig:contour_ratio}.

\subsubsection{Beyond block structure}
\label{sec:numerical_extension_beyond_blocks}
Suppose instead that the rows of $\Theta$ lie within the unit simplex, for example, $\Theta_{\operatorname{row }i} \sim \textsf{Dirichlet}(\vec{\alpha})$ i.i.d.~for $1 \le i \le n$. The population-level matrix $\Theta \meanmtx \Theta^{\tp}$ is then no longer of pure block structure per se but rather exhibits soft or mixed block structure, akin to mixed-membership stochastic blockmodel random graphs \citep{airoldi2008mixed}. Here then, the population embedding becomes a simplicial region, rather than a collection of discrete cluster centroids \citep{mao2021estimating,rubin2022statistical,jin2024mixed}. For the purpose of illustration, \cref{fig:mmbm_embed} visually demonstrates that eigenvector embeddings derived from normalized rank statistics are capable of reflecting simplicial latent geometry. In this example, $n=1800$, $\vec{\alpha} = (1/2, 1/3, 1/6)$, $\meanmtx_{11}=3$, $\meanmtx_{12}=\meanmtx_{22}=2$, $\meanmtx_{13}=\meanmtx_{23}=\meanmtx_{33}=1$, and a symmetric random matrix of Gaussian noise having mean zero and standard deviation one-half perturbs the expectation matrix. Developing theoretical performance guarantees for such settings will presumably require new technical tools.

\begin{figure}[tp]
	\centering
	\includegraphics[width=0.7\linewidth]{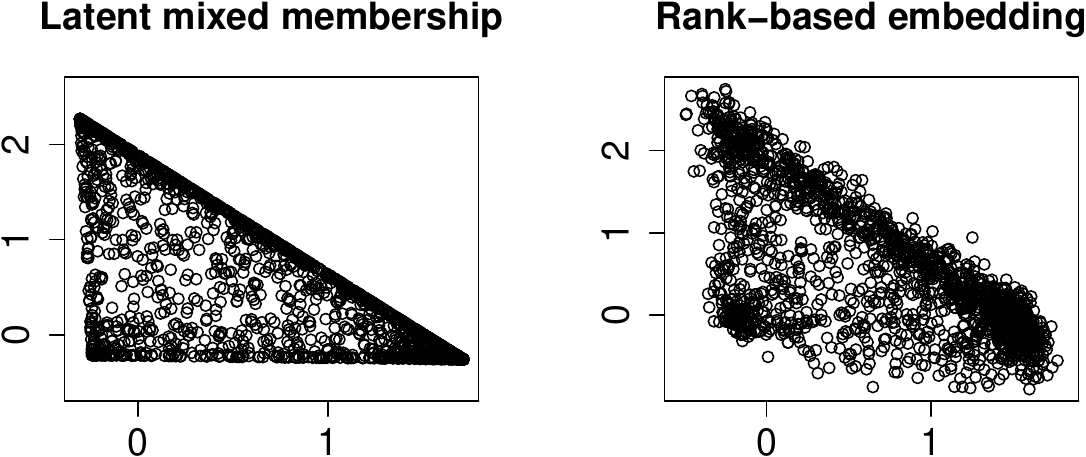}
	\caption{Pass-to-ranks robust spectral embedding (right panel) can reflect latent geometry (left panel) beyond mere block structure. Coordinates correspond to rotated and scaled leading eigenvectors. See \cref{sec:numerical_extension_beyond_blocks}.}
	\label{fig:mmbm_embed}
\end{figure}

\section{Connectome data analysis}
This section illustrates robust spectral clustering with rank statistics applied to a diffusion magnetic resonance imaging (dMRI) dataset. The dataset comprises $114$ connectomes (loosely speaking, ``brain graphs'' or ``brain networks'') derived from $57$ human subjects with two scans each. Each connectome is represented as a weighted network using the NeuroData MR Graphs data processing pipeline \citep{kiar2018high:arxiv}. Vertices or nodes represent spatially proximal brain sub-regions of interest, called ROIs, while edges or links represent tensor-based fiber streamline counts connecting different sub-regions. Importantly, each vertex in each network has an associated Left/Right brain hemisphere label and an associated Gray/White matter label. These labels serve as a benchmark for evaluating the performance of different clustering approaches and yield the ground-truth labels $\{\operatorname{LG}, \operatorname{LW}, \operatorname{RG}, \operatorname{RW}\}$.

We consider the dataset \textrm{DS01876}, publicly available from Johns Hopkins University and documented in more detail in \cite{lawrence2021standardizing}. Each network is obtained by taking the largest connected component of the induced subgraph on the vertices labeled as both Left or Right and Gray or White. For one of the connectomes, a scanning error occurred during acquisition, and consequently we omit it from our analysis and discussion. Our analysis complements the recent work on spectral clustering in \cite{priebe2019two}, though therein the authors consider different methodology and a different data resolution.

The top two panels in \cref{fig:dataConnectomes} provide a partial visual summary of the data. The top-left panel shows that all graphs have circa $1200$ vertices, while the top-right panel shows the vertex strength (i.e.,~weighted degree) distributions for twenty graphs. In fact, all graphs exhibit vertex strength heterogeneity with highly right-skewed distributions. \cref{tab:brainsSummary} provides summary statistics highlighting the severe disparity of vertex strengths as well as graph edge weights, suggesting the possible need for robust methods.

\begin{figure}[t]
	\centering
	\includegraphics[width=0.3\linewidth]{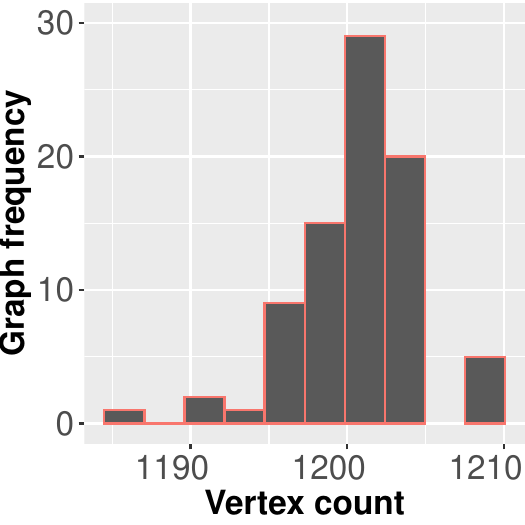}\nolinebreak\hspace{1em}
	\includegraphics[width=0.6\linewidth]{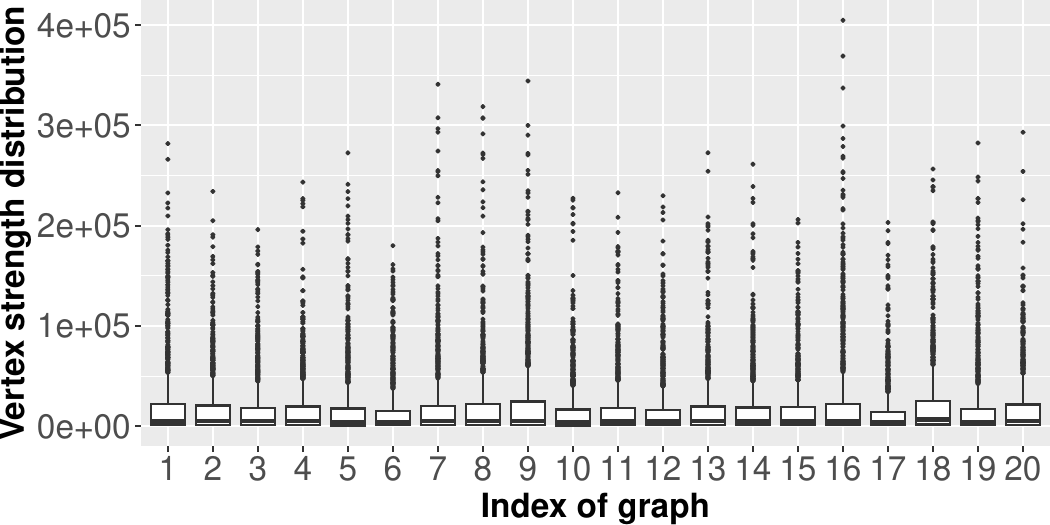}
	\includegraphics[width=0.3\linewidth]{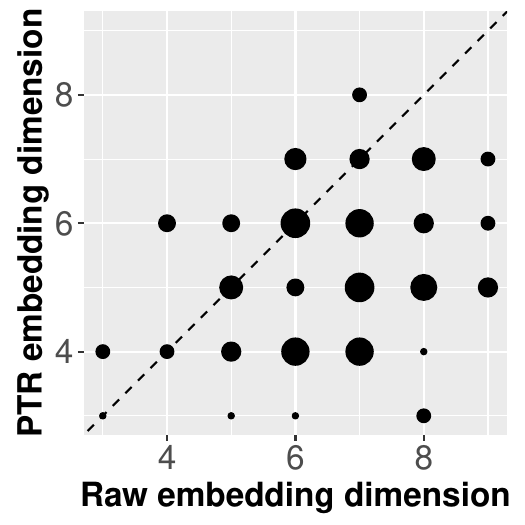}\nolinebreak\hspace{1em}
	\includegraphics[width=0.6\linewidth]{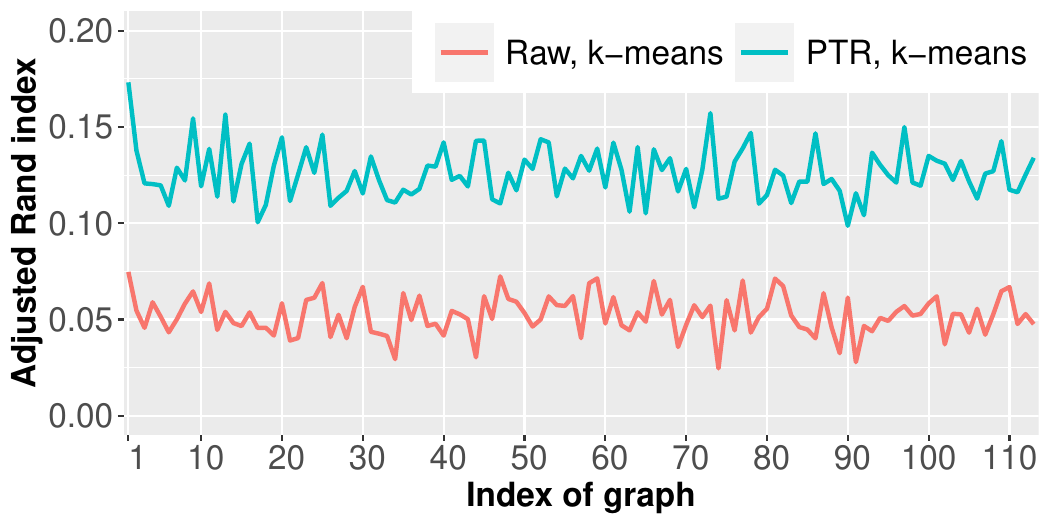}
	\caption{
		Connectome data summary and analysis. Top left: vertex counts for all graphs. Top right: vertex strength (weighted degree) distributions for twenty graphs. Bottom left: frequency of embedding dimensionality selection estimates. Bottom right: performance of $k$-means clustering for recovering the Left/Right Gray/White partition.
	}\label{fig:dataConnectomes}
\end{figure}

\begin{table}[h]
	\begin{center}
		\begin{tabular}{ c | r r r } 
			& Median & Mean & Maximum \\
			\hline
			Edge weight & 102 & 453 & 18267 \\ 
			Vertex strength & 4841 & 18365 & 253215 \\ 
		\end{tabular}
		\caption{Connectome data summary statistics. Each cell value denotes the median of the corresponding summary statistic computed across all graphs. Cell values are rounded to the nearest integer.}
		\label{tab:brainsSummary}
	\end{center}
\end{table}

For each graph index, $i$, we separately analyze the raw weighted adjacency matrix $A^{(i)}$ and its rank transformation $\widetilde{R}^{(i)}$. Edge weights are not necessarily unique for graphs in this dataset, so we use the corresponding average of normalized rank values when rank-transforming the data. We apply the profile likelihood based approach of \cite{zhu2006automatic} to the $50$ largest singular values of each matrix in order to select the number of leading eigenvectors to retain (i.e.,~embedding dimensionality), denoted by $\widehat{d}_{\operatorname{raw}}^{(i)}$ and $\widehat{d}_{\operatorname{ptr}}^{(i)}$. More precisely, to mitigate the effect of the largest singular value, we exclude it when applying this method. The bottom-left panel of \cref{fig:dataConnectomes} plots these tuples as $(\widehat{d}_{\operatorname{raw}}^{(i)}, \widehat{d}_{\operatorname{ptr}}^{(i)})$, with larger black circles denoting more frequent occurrences. For 23 of the 113 graphs, the values in these tuples are equal to each other. More notably, for nearly two-thirds of the graphs (75 of 113), fewer dimensions (i.e.,~fewer eigenvectors) are selected for embedding and clustering when using the matrix of normalized rank statistics than when using the raw adjacency matrix.

Using $k$-means clustering with four clusters, the bottom-right panel of \cref{fig:dataConnectomes} shows that robust spectral embedding and clustering achieves higher adjusted Rand index (\textsc{ari}) values \citep{hubert1985comparing} than traditional, non-robust adjacency-based spectral embedding and clustering. Here, \textsc{ari} is computed relative to the ground-truth clustering given by $\{\operatorname{LG}, \operatorname{LW}, \operatorname{RG}, \operatorname{RW}\}$, and we report the average value over one hundred runs of $k$-means clustering applied to each graph embedding. For reference, the corresponding label proportions across all subjects are approximately given by $\{32.5\%, 17\%, 32.5\%, 18\%\}$.

In general, \textsc{ari} values typically lie between zero and one, with large values indicating substantial partition agreement and small values indicating substantial partition disagreement. Although the (average) \textsc{ari} values displayed in \cref{fig:dataConnectomes} are rather small, this is because the underlying clustering task is difficult due to non-stylized, weak signal in the data. Furthermore, $\{\operatorname{LG}, \operatorname{LW}, \operatorname{RG}, \operatorname{RW}\}$ yields a natural choice of ground truth for comparing performance but is not necessarily the only neuroscientifically meaningful node partition for comparison. See \cite{priebe2019two} for related findings and discussion.

In summary, this section briefly illustrates the usefulness of robust spectral embedding and clustering in a neuroscientific data setting via improved clustering performance (i.e.,~partition recovery), together with the fact that $\widehat{d}_{\operatorname{ptr}}^{(i)} \le \widehat{d}_{\operatorname{raw}}^{(i)}$ for the majority of graph indices $i$. This analysis shows that using rank statistics can de-noise the original data in a parsimonious manner that better reflects underlying neuroanatomical cluster structure.

\section{Discussion}
\label{sec:discussion}
This section discusses extensions, practical considerations, and related open problems pertaining to robust spectral clustering with rank statistics.

\subsection{Embedding dimension versus the number of blocks}
\label{subsec:dimension_versus_number_of_communities}
Within contemporary multivariate analysis, there is growing interest in understanding the properties of data embeddings when the dimension $d$ differs from the number of ground-truth or estimated blocks (communities or clusters) $K$. In the language of blockmodels, this pertains to $d = \rank(B)$ for a $K \times K$ connectivity matrix $B$. In practice, outside of blockmodels, algorithms often yield estimates satisfying $\wh{d} \ll \wh{K}$. Even so, the majority of existing works on spectral methods assume $d=K$ when establishing theoretical guarantees, with two recent exceptions including \cite{tang2022asymptotically,zhang2022randomized}. A more refined understanding of consistency and asymptotic normality when $d < K$ is currently underway.

\subsection{Adjacency versus Laplacian representations of networks}
\label{subsec:adjacency_versus_laplacian}
This paper focuses on developing theory for the matrix of normalized rank statistics, $\nrk{A}$. It may be possible to derive similar results for graph Laplacians, $\mathcal{L}(\nrk{A})$, in part by adapting the techniques and analysis found in \cite{tang2018limit}. That being said, by its very nature, the matrix of normalized rank statistics cannot exhibit significant row-sum heterogeneity nor is it (even approximately) sparse. Consequently, the benefit of and motivation for Laplacian-based clustering are less clear.

\subsection{Communities in the presence of similarities and dissimilarities}
\label{subsec:similarity_dissimilarity}
\cref{def:data_matrices} is a flexible model for generating weighted graphs (networks), permitting both positive edge weights (indicating similarity) and negative edge weights (indicating dissimilarity). Different block-specific distributions may have different support sets, for example, $\operatorname{supp}(F_{(1,1)}) \subseteq [0, \infty)$ while $\operatorname{supp}(F_{(1,2)}) \subseteq (-\infty, 0]$ while $\operatorname{supp}(F_{(2,2)}) = \R$. Regardless, the interpretation of inferred community memberships is unchanged, namely connectivity is driven by population-level low-rank structure that separates the block-specific distributions. While the matrix of normalized rank statistics and its expectation are entrywise non-negative, the entries themselves encode the relationship between positive and negative edge weights via their functional dependence on the block-specific distributions.

\subsection{Rectangular data matrices}
\label{subsec:rectangular_data}
The rank-transform method in \cref{sec:ptr_overview} can be naturally adapted to handle asymmetric or more general rectangular data matrices $A \in \R^{n \times m}$ by ranking among all $n \times m$ entries when distinct. In preliminary investigations, the left and right singular vectors of $\nrk{A}$ can accurately recover left-specific and right-specific latent block structure when present, respectively. Theoretical guarantees for such settings will be established in future work.

\subsection{Literature on entrywise eigenvector analysis}
This paper and its contributions are related to the literature studying entrywise eigenvector estimation and high-dimensional PCA, building on classical tools in linear algebra such as Weyl's inequality, the Davis--Kahan theorem, and the Wedin sine theta subspace theorem (e.g.,~see the original papers or the textbooks \cite{stewart1990matrix,bhatia2013matrix} for more). Beyond the references cited in \cref{sec:introduction}, recent works in this vein include \cite{perry2018optimality,zhong2018near,chen2019spectral,lei2019unified,cai2021subspace,li2021minimax,abbe2022ellp,agterberg2022entrywise,jin2022improvements}. Many of these works contribute results in so-called `signal-plus-noise' model settings for adjacency matrices, covariance matrices, Gram-type matrices, and factor model settings, albeit not involving rank statistics as considered in this paper.

\subsection{Precluding certain stochastic blockmodel graphs and variants}
The contributions of this paper are distinct from but complement existing works studying regularization and normalization in stochastic blockmodels \citep{amini2013pseudo,sarkar2015role,joseph2016impact}. In the social sciences, statistics, and machine learning applications involving random graphs and network data, significant attention has been paid to (unweighted) stochastic blockmodels \citep{holland1983stochastic} including mixed-membership \citep{airoldi2008mixed}, degree-corrected \citep{karrer2011stochastic}, and weighted variants \citep{xu2020optimal}. \cref{def:data_matrices} excludes the traditional formulations of these models, since $a_{ij} \in \{0,1\}$ would lead to ties among the normalized ranks, which conflicts with the setup in \cref{alg:ptr_baseline}. Rather, \cref{def:data_matrices} can be viewed as producing weighted adjacency matrices $A = (a_{ij})$, associated with weighted undirected complete graphs, whose entries represent nodes' pairwise edge weights, specified via the block-specific distributions $F \in \mathcal{F}^{K}$. \cref{def:data_matrices} permits but does not require both positive and negative edge weights, $a_{ij} \in \R$, in contrast to the classical paradigm for random graphs which typically requires $a_{ij} \ge 0$. Finally, $\nrk{A}$ can itself be viewed as a weighted non-negative adjacency matrix of dependent entries.

\subsection{Concentration inequalities and existing benchmarks}
In order to obtain consistency guarantees, we first establish \cref{result:operator_norm_concentration} which yields $\|\nrk{A} - \Ex[\nrk{A}]\|_{\op} = \Ohp(n^{3/4+\epsP})$. In contrast, the traditional bound $\Ohp(n^{1/2})$ holds for symmetric random matrix models of the form $M - \Ex[M]$ when the entries $m_{ij}$ are independent, not-necessarily identically distributed, and sufficiently light-tailed, e.g.,~matrices with heterogeneous yet uniformly bounded variance profiles whose entries are \textsf{Bernoulli}-distributed \citep{lei2015consistency}, \textsf{Gaussian}-distributed \citep{bandeira2016sharp}, or similar \citep{tao2012topics}. After this paper was accepted for publication, we discovered that the bound $\Ohp(n^{1/2})$ also holds in the current setting but via a different proof strategy. The full details and subsequent implications of this improvement will appear in future work.

This paper establishes that the pass-to-ranks spectral embedding vectors, $\sampVecs_{\operatorname{row} i}$, properly transformed, are asymptotically multivariate Gaussian in a manner governed by their block memberships. Such results have previously been established for eigenvectors of large \textsf{Bernoulli} blockmodel matrices, e.g.,~see \cite{cape2019signal,rubin2022statistical} and elsewhere. In contrast, the results in this paper are seemingly among the first to address robust rank statistic eigenvector-based embeddings and their asymptotics.

\section{Conclusion}
\label{sec:conclusion}
This paper rigorously studies a statistically principled approach for robust clustering and dimension reduction. The idea of clustering rank-transformed data has appeared previously, as an applied tool and practical work-around, but received only brief attention, e.g.,~a passing comment in \cite{athreya2018statistical}, and a remark in the appendix of \cite{levin2020role:arxiv}. Our in-depth investigation yields results that collectively address weak consistency, node-specific strong consistency, and asymptotic normality, all at once.

Open questions remain. For example, it will be interesting to determine optimal misclustering rates and related properties of robust spectral clustering with rank statistics. Since computing exact rank values for all elements of a massive data matrix can be computationally expensive, it will be interesting to explore the potential benefit of approximate ranking. As pointed out by a referee, beyond blockmodels, it will be interesting to study the extent to which ordinal transformations are useful when clustering Euclidean data based on constructed neighborhood graphs. In addition, there may yet be further, unexplored connections between robust spectral clustering and the classical literature on nonparametric rank tests.

\section{Acknowledgments and Disclosure of Funding}
\label{sec:acknowledgments_funding}
The authors thank the Associate Editor and the reviewers for valuable feedback and thoughtful suggestions. The authors thank Carey E. Priebe, Youngser Park, and Daniel L. Sussman for inspiring conversations and for coining the phrase `passing to ranks'. The authors thank Liza Levina, Marianna Pensky, Karl Rohe, Michael Trosset, and Bodhisattva Sen for constructive comments at various stages of this work. JC gratefully acknowledges support from the National Science Foundation under grant DMS 2413552 and from the University of Wisconsin--Madison, Office of the Vice Chancellor for Research and Graduate Education, with funding from the Wisconsin Alumni Research Foundation.
% Acknowledgements and Disclosure of Funding should go at the end, before appendices and references
%\acks{All acknowledgements go at the end of the paper before appendices and references. Moreover, you are required to declare funding (financial activities supporting the submitted work) and competing interests (related financial activities outside the submitted work). More information about this disclosure can be found on the JMLR website.}

\appendix
\section{}
\label{appendix}
The appendix contains proof materials and additional details pertaining to the main text.

\subsection{Rank statistics and their properties}
\label{app:rank_statistic_properties}
Let $X_{1}, \dots, X_{N}$ denote a collection of independent absolutely continuous random variables with cumulative distribution functions $F_{1}, \dots, F_{N}$ and corresponding densities $f_{1}, \dots, f_{N}$. Consequently, the set $\{X_{i}\}_{i=1}^{N}$ takes on distinct values with probability one. Denote the \emphh{order statistics} of $\{X_{i}\}_{i=1}^{N}$ by $X_{N(1)} \le X_{N(2)} \le \dots \le X_{N(N)}$. Let $R_{Ni}$ denote the \emphh{rank} of $X_{i}$ among $\{X_{i}\}_{i=1}^{N}$, i.e.,~define $R_{Ni}$ according to the equation $X_{i} = X_{N(R_{Ni})}$. Correspondingly, define the \emphh{normalized rank} of $X_{i}$ among $\{X_{i}\}_{i=1}^{N}$ in the manner $\nrk{Ni} = \tfrac{1}{N+1}R_{Ni}$. We are principally interested in the normalized ranks, though for convenience in our derivations and presentation we also define the \emphh{intermediary normalized rank} of $X_{i}$ as $\nrk{Ni}^{\prime} = \tfrac{1}{N}R_{Ni}$, observing that $\nrk{Ni}^{\prime} = \tfrac{N+1}{N}\nrk{Ni}$. For ease of notation, we often omit dependence on $N$, writing simply $R_{i}$ and $\nrk{i}$ and $\nrk{i}^{\prime}$.

\subsubsection{One-block properties of ranks}
\label{sec:ranks_one_sample}
Suppose the absolutely continuous random variables $X_{1}, \dots, X_{N}$ are independent and identically distributed with common cumulative distribution function $F$ and density $f$. Then, the vector $R_{N} = (R_{N1}, \dots, R_{NN})$ is uniformly distributed on the set of all $N!$ permutations of the set $\Nset{N} = \{1, 2, \dots, N\}$, whence for each $i$ the distribution of $R_{Ni}$ is given by the discrete uniform distribution with support $\{1, 2, \dots, N\}$. It is straightforward to compute basic properties of the intermediary normalized ranks presented in the following proposition.
\begin{proposition}[One-block rank properties]
	\label[proposition]{prop:ranks_one_sample}
	Assume the setup in \cref{sec:ranks_one_sample}. For each $i, i^{\prime} \in \Nset{N}$ where $i^{\prime} \neq i$, the intermediary normalized ranks satisfy
	\begin{equation}
		\label{eq:null_ex_var_cov_intermediary_ranks}
		\Ex\left[\nrk{Ni}^{\prime}\right]
		=
		\frac{1}{2} + \frac{1}{2N},
		\quad
		\Var\left[\nrk{Ni}^{\prime}\right]
		=
		\frac{1}{12} - \frac{1}{12N^{2}},
		\quad
		\Cov\left[\nrk{Ni}^{\prime}, \widetilde{R}_{Ni^{\prime}}^{\prime}\right]
		=
		\frac{-1}{12N} - \frac{1}{12N^{2}}.
	\end{equation}
\end{proposition}
By basic algebra, it follows from \cref{prop:ranks_one_sample} that the normalized ranks satisfy
\begin{equation}
	\label{eq:null_ex_var_cov_normalized_ranks}
	\Ex\left[\nrk{Ni}\right]
	=
	\frac{1}{2},
	\quad
	\Var\left[\nrk{Ni}\right]
	=
	\frac{1}{12} - \frac{1}{6(N+1)},
	\quad
	\Cov\left[\nrk{Ni}, \nrk{Ni^{\prime}}\right]
	=
	\frac{-1}{12(N+1)}.
\end{equation}

\subsubsection{Multiple-block properties of ranks}
\label{sec:ranks_multi_sample}
Suppose instead that the collection of independent absolutely continuous random variables $X_{1}, \dots, X_{N}$ is partitioned into $\numSample$ blocks with deterministic sizes $\{N_{\ell}\}_{\ell \in \Nset{\numSample}}$ satisfying $\sum_{\ell \in \Nset{\numSample}} N_{\ell} = N$, where $N_{\ell} = |\{i : X_{i} \sim F_{\ell}\}|$, for a collection of cumulative distribution functions $\{F_{\ell}\}_{\ell \in \Nset{\numSample}}$ with corresponding densities $\{f_{\ell}\}_{\ell \in \Nset{\numSample}}$. Here, $R_{Ni}$ denotes the rank of $X_{i}$ among $\{X_{i}\}_{i \in \Nset{N}}$.

Below, we compute finite-$N$ expectations, variances, and covariances involving the intermediary normalized rank statistics $\intermRank$. In an effort to improve clarity, we write expectation expressions as $\Ex_{\ell}[\cdot]$ to emphasize the underlying block membership $X \sim F_{\ell}$. Similar notational shorthand applies to variances and covariances.
\begin{proposition}[Multi-block rank expectations]
	\label[proposition]{prop:ranks_multi_expectation}
	Assume the setup in \cref{sec:ranks_multi_sample}. For each $\ell \in \Nset{\numSample}$ and $i \in \Nset{n}$, the expectation of $\intermRank$ when $X_{i} \sim F_{\ell}$ is given by
	\begin{align*}
		\Ex_{\ell}\left[ \intermRank \right]
		&=
		\sum_{\ell^{\prime} \in \Nset{\numSample}} \frac{N_{\ell^{\prime}}}{N} \times \Ex_{{\ell}}[F_{{\ell^{\prime}}}(X)]
		+
		\frac{1}{2N}.
	\end{align*}
	Notably,
	$\Ex_{\ell}[F_{\ell}(X)]
	=
	\tfrac{1}{2}$.
\end{proposition}

\begin{proposition}[Multi-block rank variances]
	\label[proposition]{prop:ranks_multi_variance}
	Assume the setup in \cref{sec:ranks_multi_sample}. For each $\ell \in \Nset{\numSample}$ and $i \in \Nset{n}$, the variance of $\intermRank$ when $X_{i} \sim F_{\ell}$ is given by
	\begin{align*}
		\Var_{\ell}\left[ \intermRank \right]
		&= 
		\left(\sum_{\ell^{\prime}, \ell^{\prime\prime} \in \Nset{\numSample}} \frac{N_{\ell^{\prime}}}{N}\frac{N_{\ell^{\prime\prime}}}{N} \times \Ex_{\ell}[F_{\ell^{\prime}}(X)F_{\ell^{\prime\prime}}(X)]\right)
		-
		\left(\sum_{\ell^{\prime} \in \Nset{\numSample}}\frac{N_{\ell^{\prime}}}{N}\times\Ex_{{\ell}}[F_{{\ell^{\prime}}}(X)]\right)^{2}\\
		&\hspace{2em}+
		\left(\sum_{\ell^{\prime} \in \Nset{\numSample}}\frac{N_{\ell^{\prime}}}{N} \times \Ex_{\ell}\left[2F_{\ell^{\prime}}(X) - 2F_{\ell}(X)F_{\ell^{\prime}}(X) - F_{\ell^{\prime}}^{2}(X) \right]\right)\frac{1}{N}\\
		&\hspace{2em}-
		\frac{1}{12}\frac{1}{N^{2}}.
	\end{align*}
	Notably,
	$\Ex_{\ell}[F_{\ell}^{2}(X)] - \Ex_{\ell}[F_{\ell}(X)]^{2}
	=
	\tfrac{1}{3} - (\tfrac{1}{2})^{2}
	=
	\tfrac{1}{12}$.	
\end{proposition}

In order to state the multi-block covariance formula, define $g_{(X,Y)} : \R \rightarrow [0, 1]$ to be the map $g_{(X,Y)}(z) = \int_{-\infty}^{z}F_{X}(y)f_{Y}(y)\,\textsf{d}y$, in terms of the cumulative distribution function and probability density of the underlying independent random variables $X$ and $Y$, respectively. Let $g_{(\ell, \ell^{\prime})}(z)$ denote shorthand for the distributional assumption that the independent random variables satisfy $X \sim F_{\ell}$ and $Y \sim F_{\ell^{\prime}}$.

\begin{proposition}[Multi-block rank covariances]
	\label[proposition]{prop:ranks_multi_covariance}
	Assume the setup in \cref{sec:ranks_multi_sample}. For each $\ell, \ell^{\prime} \in \Nset{\numSample}$ and $i, j \in \Nset{N}$ where $i \neq j$, the covariance of $\widetilde{R}_{Ni}^{\prime}$ and $\widetilde{R}_{Nj}^{\prime}$ when $X_{i} \sim F_{\ell}$ and $X_{j} \sim F_{\ell^{\prime}}$ is given by
	\begin{align*}
		&\Cov_{(\ell,\ell^{\prime})}\left[\widetilde{R}_{Ni}^{\prime},\widetilde{R}_{Nj}^{\prime}\right]\\
		&\hspace{1em}=
		\sum_{m^{\prime} \in \Nset{\numSample}}\frac{N_{m^{\prime}}}{N} \times 
		\Bigg(\Ex_{m^{\prime}}[(1-F_{\ell}(X))(1-F_{\ell^{\prime}}(X))]
		+ \Ex_{\ell}[g_{(m^{\prime},\ell^{\prime})}(X)]
		+ \Ex_{\ell^{\prime}}[g_{(m^{\prime},\ell)}(X)]\\
		&\hspace{4em}-
		\Ex_{\ell}[F_{m^{\prime}}(X)]\Ex_{\ell^{\prime}}[F_{m^{\prime}}(X)]
		- \Ex_{\ell}[F_{m^{\prime}}(X)]\Ex_{\ell^{\prime}}[F_{\ell}(X)]
		- \Ex_{\ell}[F_{\ell^{\prime}}(X)]\Ex_{\ell^{\prime}}[F_{m^{\prime}}(X)]
		\Bigg)\frac{1}{N}\\
		&\hspace{3em}+
		\Bigg( \Ex_{\ell}\left[2F_{\ell^{\prime}}(X)-F_{\ell}(X)F_{\ell^{\prime}}(X)-g_{(\ell,\ell^{\prime})}(X) - g_{(\ell^{\prime},\ell^{\prime})}(X) - \tfrac{1}{3}\right]\Bigg)\frac{1}{N^{2}}\\
		&\hspace{3em}+
		\Bigg(\Ex_{\ell^{\prime}}\left[2F_{\ell}(X)-F_{\ell}(X)F_{\ell^{\prime}}(X)-g_{(\ell,\ell)}(X) - g_{(\ell^{\prime},\ell)}(X) - \tfrac{1}{3}\right]\Bigg)\frac{1}{N^{2}}\\
		&\hspace{3em}+ \Bigg( \Ex_{\ell}[F_{\ell^{\prime}}(X)] \Ex_{\ell^{\prime}}[F_{\ell}(X)] - \tfrac{1}{4} \Bigg)\frac{1}{N^{2}}\\
		&\hspace{3em}-
		\frac{1}{12}\frac{1}{N^{2}}.
	\end{align*}
	Notably, 
	$\Ex_{\ell}[g_{(\ell,\ell)}(X)]
	=
	\Ex_{\ell}[\tfrac{1}{2}F_{\ell}^{2}(X)]
	=
	\tfrac{1}{6}$.
\end{proposition}

\cref{prop:ranks_multi_covariance} implies a loose yet user-friendly two-sided uniform bound for the intermediary normalized ranks given by
$-\frac{8}{N}
\le
\Cov_{(\ell,\ell^{\prime})}\left[\nrk{Ni}^{\prime},\nrk{Nj}^{\prime}\right]
\le
\frac{7}{N}$.
The same bound holds for the normalized rank statistics $\nrk{Ni}$ since $\nrk{Ni} = \frac{N}{N+1}\nrk{Ni}^{\prime}$. Moreover,
\begin{equation*}
	\left(\frac{-3}{N} + \frac{-5}{N^{2}} \right)\left(\frac{N}{N+1}\right)^{2} 
	\le
	\Cov_{(\ell,\ell^{\prime})}\left[\nrk{Ni},\nrk{Nj}\right]
	\le
	\left(\frac{3}{N} + \frac{4}{N^{2}} \right)\left(\frac{N}{N+1}\right)^{2}.
\end{equation*}

\cref{prop:ranks_multi_expectation,prop:ranks_multi_variance,prop:ranks_multi_covariance} illustrate how generalizing the expectation, variance, and covariance beyond the one-block setting amounts to the following:
\begin{enumerate}
	\item Leading-order terms arise as functionals of the distributions $\{F_{\ell}\}_{\ell \in \Nset{\numSample}}$ via appropriate integration and weighting.
	\item The higher-order terms in the one-block setting persist in the multi-block setting.
	\item The multi-block variance and covariance expressions introduce additional higher order interaction-type residual terms. These additional terms arise as a consequence of the heterogeneity encountered in the multi-block setting and conversely vanish in the homogeneous one-block setting.
\end{enumerate}

\subsection{Proofs: weak consistency}
\label{app:weak_consistency}
\begin{remark}[The moment method and bounding the matrix trace]
	\label[remark]{remark:moment_method}
	For any $n \times n$ symmetric matrix $M$ with eigenvalues $\lambda_{1}, \dots, \lambda_{n}$, necessarily real, it holds that $\|M\|_{\op} = \max_{1 \le i \le n} |\lambda_{i}|$ and $\trace(M) = \sum_{i=1}^{n}\lambda_{i}$. Moreover, for every positive integer $p \ge 1$ it holds that $\trace(M^{p}) = \sum_{i=1}^{n} \lambda_{i}^{p}$, which follows directly from Schur's triangularization theorem together with the cyclic property of $\trace(\cdot)$. Consequently, for any positive integer $p \ge 1$, it holds that $\|M\|_{\op}^{2p} \le \trace(M^{2p}) \le n \|M\|_{\op}^{2p}$, which in turn yields the general inequality
	\begin{equation}
		\label{eq:moment_method_inequality}
		\|M\|_{\op}
		\le
		\sqrt[2p]{\trace(M^{2p})}
		\le
		\sqrt[2p]{n} \cdot \|M\|_{\op}.
	\end{equation}
	For a more detailed treatment of the moment method and its applications throughout random matrix theory, see for example \citet[Chapter~2]{tao2012topics}.
\end{remark}

\begin{lemma}[Spectral norm concentration of the centered normalized ranks matrix]
	\label[lemma]{result:operator_norm_concentration}
	Let $A$ be as in \cref{def:data_matrices} and $\nrk{A}$ be as in \cref{alg:ptr_baseline}. There exist absolute constants $C_{1}, C_{2} > 0$ such that for any choice $\epsP > 0$, provided $K \le n^{1/2}$ and $n_{k} \ge C_{1}$ for all $k \in \Nset{K}$, then
	\begin{equation}
		\Pr\left[\|\wt{R}_{A} - \rankspecBig\|_{\op} > n^{3/4 + \epsP} \right]
		\le
		C_{2} \cdot n^{-\epsP}.
	\end{equation}
	Alternatively, the operator norm exceeds $2C_{2} n^{3/4+\epsP}$ with probability at most $\tfrac{1}{2}n^{-\epsP}$. Furthermore, the conclusion also holds for $\|\wt{R}_{A} - \Ex[\nrk{A}]\|_{\op}$.
\end{lemma}

\begin{proof}[Proof of~\cref{result:operator_norm_concentration}]
	Consider the matrices $\nrkmtx$ and $\rankspecBig = \Ex[\nrkmtx] + \diag(\rankspecBig)$. By the triangle inequality,
	\begin{equation*}
		\|\nrkmtx - \rankspecBig\|_{\op}
		\le
		\|\nrkmtx - \Ex[\nrkmtx]\|_{\op} + \|\diag(\rankspecBig)\|_{\op}.
	\end{equation*}	
	Furthermore,
	$\|\diag(\rankspecBig)\|_{\op}
	=
	\max_{1 \le k \le K}|\rankspec_{kk}|
	\le
	1$.
	
	By Markov's inequality, for any $t > 0$ it holds that $\Pr\left[\|M\|_{\op} > t\right] \le \frac{\Ex\|M\|_{\op}}{t}$. The map $t \mapsto \sqrt[4]{t}$ is concave on the domain $(0,\infty)$, hence by the concave version of Jensen's inequality it holds that $\Ex\left[ \sqrt[4]{\trace(M^{4})}\right] \le \sqrt[4]{\Ex\left[\trace(M^{4})\right]}$. Combining these observation with \cref{eq:moment_method_inequality} in \cref{remark:moment_method} yields the bound
	\begin{equation}
		\label{eq:temp_apply_Markov}
		\Pr\left[\|\nrkmtx - \rankspecBig\|_{\op} > t \right]
		\le
		\frac{1 + \sqrt[4]{\Ex\left[\trace\left((\nrkmtx - \Ex[\nrkmtx])^{4}\right)\right]}}{t}.
	\end{equation}
	
	By applying \cref{lem:ms_trace}, the key element of this proof, there exists a constant $C_{1} > 0$ such that if $n_{k} \ge C_{1}$ for all $k \in \Nset{K}$, then
	\begin{equation*}
		1 + \sqrt[4]{\Ex\left[\trace\left((\nrkmtx - \Ex[\nrkmtx])^{4}\right)\right]}
		\le
		1 + \sqrt[4]{66K^{2} n^{2} + 2n^{3}}
		\le
		C_{2} \cdot n^{3/4},
	\end{equation*}
	where $C_{2} > 0$ denotes a new constant, using the hypothesis that $K \le n^{1/2}$. Finally, set $t = n^{3/4 + \epsP}$ in \cref{eq:temp_apply_Markov}. This establishes \cref{result:operator_norm_concentration}.
\end{proof}

\begin{remark}[Conditions and concentration bounds]
	\label[remark]{rem:conditions_concentration_bounds}
		The requirement $K \le n^{1/2}$ in \cref{result:operator_norm_concentration} yielding $\|\nrk{A} - \Ex[\nrk{A}]\|_{\op} = \Ohp(n^{3/4 + \epsP})$ is due to our proof approach and the challenge of working with normalized rank statistics. Alternatively, we could have allowed $K = \oh(n)$ but at the expense of more complicated upper bound statements. For simplicity and ease of presentation, we have chosen to simply require $K \le n^{1/2}$.
\end{remark}

\begin{lemma}[Eigen-structure of block matrices, see Lemma~2.1 in \cite{lei2015consistency}]
	\label[lemma]{result:blockmodel_eigenvectors}
	Let $\Theta \in \membmtx^{n \times K}$ and a symmetric matrix $B \in \R^{K \times K}$ each have full rank equal to $K$. Let $U D U^{\tp}$ denote the eigen-decomposition of $\Theta B \Theta^{\tp}$. Then, $U = \Theta T$ for some $T \in \R^{K \times K}$ where $\|T_{k\star} - T_{\ell\star}\|_{\ell_{2}} = \sqrt{n_{k}^{-1} + n_{\ell}^{-1}}$ for all $1 \le k < \ell \le K$.
\end{lemma}
Further below, \cref{result:blockmodel_eigenvectors} will be applied in conjunction with the $K \times K$ symmetric invertible matrix $\rankspec$.

\begin{lemma}[Principal subspace perturbation, see Lemma~5.1 in \cite{lei2015consistency}]
	\label[lemma]{result:subspace_perturbation_Davis-Kahan}
	Let $M^{(1)}$ be any real-valued symmetric $n \times n$ matrix. Let $M^{(2)} \in \R^{n \times n}$ be a symmetric matrix with rank $K$ and with smallest nonzero singular value $\gamma_{n}$. Let $U^{(1)}, U^{(2)} \in \R^{n \times K}$ denote orthonormal matrices whose columns are eigenvectors for the $K$ largest-in-magnitude eigenvalues of $M^{(1)}$ and $M^{(2)}$, respectively. Then, there exists a $K \times K$ orthogonal matrix $Q$ such that
	\begin{equation}
		\label{eq:subspace_perturbation_Davis-Kahan_bound}
		\|U^{(1)} - U^{(2)}Q\|_{\frob}
		\le
		\frac{2\sqrt{2K}}{\gamma_{n}}\|M^{(1)} - M^{(2)}\|_{\op}.
	\end{equation}
\end{lemma}
We remark that $\sqrt{K}$ appearing \cref{eq:subspace_perturbation_Davis-Kahan_bound} cannot be removed in general.

\cref{result:subspace_perturbation_Davis-Kahan} is stated above for convenience and as a reference. However, it will be useful to consider the following variant.
\begin{lemma}
	\label[lemma]{result:subspace_perturbation_Davis-Kahan_modified}
	Assume the hypothesis in \cref{result:subspace_perturbation_Davis-Kahan}.
	If $\gamma_{n} > (1-1/\sqrt{2})^{-1} \cdot \|M^{(1)} - M^{(2)}\|_{\op}$, then there exists a $K \times K$ orthogonal matrix $Q$ such that
	\begin{equation}
		\label{eq:subspace_perturbation_Davis-Kahan_bound_modified}
		\|U^{(1)} - U^{(2)}Q\|_{\frob}
		\le
		\frac{2}{\gamma_{n}}\|(M^{(1)} - M^{(2)})U^{(2)}\|_{\frob}.
	\end{equation}
\end{lemma}
\begin{proof}[Proof of \cref{result:subspace_perturbation_Davis-Kahan_modified}]
	Let $Q$ be an orthogonal matrix that minimizes the objective function $\|U^{(1)} - U^{(2)}W\|_{\frob}^{2}$ over all orthogonal matrices $W$. In other words, let $Q$ be a solution to the orthogonal Procrustes problem with input matrices $U^{(1)}$ and $U^{(2)}$. Applying Corollary~2.8 in \cite{chen2021spectral} yields $ \|U^{(1)} - U^{(2)}Q\|_{\frob} = \inf_{W \textnormal{orthogonal}} \|U^{(1)} - U^{(2)}W\|_{\frob} \equiv \operatorname{dist}_{\frob}(U^{(1)},U^{(2)}) \le \frac{2}{\gamma_{n}}\|(M^{(1)} - M^{(2)})U^{(2)}\|_{\frob}$, as desired.
\end{proof}

\begin{lemma}[Approximate $k$-means error bound, see Lemma~5.3 in \cite{lei2015consistency}]
	\label[lemma]{result:kmeans_subspace_connection}
	For $\epsK > 0$ and any two full column rank matrices $U^{(1)}, U^{(2)} \in \R^{n \times K}$ such that $U^{(2)} = \Theta T$ with $\Theta \in \membmtx^{n \times K}$, $T \in \R^{K \times K}$, let $(\wh{\Theta}, \wh{T})$ be a $(1+\epsK)$-approximate solution to the $k$-means problem and $\overline{U}^{(1)} = \wh{\Theta}\wh{T}$. For any $\delta_{k} \le \min_{\ell \neq k} \|T_{\ell\star} - T_{k\star}\|_{\ell_{2}}$, define $\mathcal{S}_{k} = \{i \in \mathcal{G}_{k}(\Theta) : \|\overline{U}^{(1)}_{i\star} - U^{(2)}_{i\star}\|_{\ell_{2}} \ge \delta_{k}/2\}$. Then,
	\begin{equation}
		\label{eq:kmeans_subspace_connection_bound_setSk}
		\sum_{k=1}^{K} |\mathcal{S}_{k}| \cdot \delta_{k}^{2}
		\le
		8 \cdot (2+\epsK) \cdot \|U^{(1)} - U^{(2)}\|_{\frob}^{2}.
	\end{equation}
	Moreover, if
	\begin{equation}
		\label{eq:kmeans_subspace_connection_bound_Frobenius}
		8 \cdot (2+\epsK) \cdot \|U^{(1)} - U^{(2)}\|_{\frob}^{2} \, / \, \delta_{k}^{2}
		<
		n_{k} \quad \textnormal{ for all } k,
	\end{equation}
	then there exists a $K \times K$ permutation matrix $\Pi$ such that $\wh{\Theta}_{\mathcal{G}\star} = \Theta_{\mathcal{G}\star}\Pi$, where $\mathcal{G} = \bigcup_{k=1}^{K} (\mathcal{G}_{k} \, \backslash \, \mathcal{S}_{k})$.
\end{lemma}

\begin{center}
	$(\star)$
\end{center}

We are now prepared to prove \cref{result:weak_consistency_kmeans}, our general weak consistency result.

\begin{proof}[Proof of \cref{result:weak_consistency_kmeans}]
	This proof modifies the approach used to establish Theorem~3.1 in \cite{lei2015consistency}. Keep in mind that we consider the truncated eigendecomposition of $\nrk{A}$ rather than of $A$.
	
	Given $\epsP > 0$, define the event $\mathcal{E}_{1,\epsP,n} = \{\|\nrk{A} - \rankspecBig\|_{\op} \le 2 C_{2} n^{3/4+\epsP}\}$, where $C_{2} > 0$ denotes a sufficiently large universal constant. By \cref{result:operator_norm_concentration}, $\Pr[\mathcal{E}_{1,\epsP,n}] \ge 1 - \frac{1}{2}n^{-\epsP}$. Similarly, define the event $\mathcal{E}_{2,\epsP,n,K} = \{\|(\nrk{A} - \rankspecBig)U\|_{\frob} \le 32 K^{1/2}n^{1/2+\epsP/2}\}$. By the proof of \cref{result:bound_EU_op}, $\Pr[\mathcal{E}_{2,\epsP,n,K}] \ge 1-\frac{1}{2}n^{-\epsP}$. Consequently, the intersection of these ``good events'', denoted by $\mathcal{E} \equiv \mathcal{E}_{3,\epsP,n,K}$, holds with probability at least $1-n^{-\epsP}$ for all values of $n$, $n_{k}$, and $K$ satisfying the hypotheses in \cref{result:operator_norm_concentration}.
	
	Let $\gamma_{n} > C_{2}^{\prime}n^{3/4+\epsP}$, per the hypothesis. Then, on the event $\mathcal{E}$, it follows from \cref{result:subspace_perturbation_Davis-Kahan_modified,result:bound_EU_op} that there exists an orthogonal matrix $Q$ and a positive constant $C > 0$ such that
	\begin{equation*}
		\|\widehat{U} - UQ\|_{\frob}
		\le
		\frac{2}{\gamma_{n}}\|(\nrk{A}-\rankspecBig)U\|_{\frob}
		\le
		2\sqrt{2} \cdot C \cdot \left(\frac{K^{1/2} \cdot n^{1/2 + \epsP/2}}{\gamma_{n}}\right).
	\end{equation*}
	Importantly, $\widehat{U} \in \R^{n \times K}$ is a matrix whose columns are $K$ orthonormal eigenvectors for the top-$K$ eigenspace of $\nrk{A}$.
	
	Next, apply \cref{result:kmeans_subspace_connection} to $\widehat{U}$ and $UQ$. \cref{result:blockmodel_eigenvectors} yields $U Q = \Theta T Q = \Theta T^{\prime}$ such that $\|T_{k\star}^{\prime} - T_{\ell\star}^{\prime}\|_{\ell_{2}} = \sqrt{\frac{1}{n_{k}} + \frac{1}{n_{\ell}}}$ for $k \neq \ell$. Choose $\delta_{k} = \sqrt{\frac{1}{n_{k}} + \frac{1}{\max\{n_{\ell} : \ell \neq k\}}}$ in \cref{result:kmeans_subspace_connection}, which yields $n_{k} \delta_{k}^{2} \ge 1$ for all $1 \le k \le K$. Using the above display equation, a sufficient condition for \cref{eq:kmeans_subspace_connection_bound_Frobenius} to be valid for $\widehat{U}$ and $UQ$ with probability at least $1-n^{-\epsP}$ is
	\begin{equation}
		8^{2} C^{2} \cdot (2+\epsK) \frac{K n^{1 +\epsP}}{\gamma_{n}^{2}} 
		<
		1
		\le
		\min_{1 \le k \le K} n_{k} \delta_{k}^{2}.
	\end{equation}
	Now, set $c = \frac{1}{64C^{2}} > 0$, whence with probability at least $1 - n^{-\epsP}$, we have
	\begin{align*}
		\sum_{k=1}^{K} \frac{|\mathcal{S}_{k}|}{n_{k}}
		&\le
		\sum_{k=1}^{K} |\mathcal{S}_{k}|\left(\frac{1}{n_{k}} + \frac{1}{\max\{n_{\ell}:\ell \neq k\}}\right)
		=
		\sum_{k=1}^{K} |\mathcal{S}_{k}| \delta_{k}^{2}\\
		&\le
		8 \cdot (2+\epsK) \|\wh{U} - U Q\|_{\frob}^{2}\\
		&\le
		8^{2} C^{2} \cdot (2+\epsK) \left(\frac{K n^{1 + \epsP}}{\gamma_{n}^{2}}\right)
		=
		c^{-1} (2+\epsK) \left(\frac{K n^{1 + \epsP}}{\gamma_{n}^{2}}\right).
	\end{align*}
	\cref{result:weak_consistency_kmeans} is thereby established, since \cref{result:kmeans_subspace_connection} guarantees that the block memberships are correctly recovered on the complement of the set $\bigcup_{1 \le k \le K} \mathcal{S}_{k}$.
\end{proof}

\begin{proof}[Proof of \cref{result:corollary_weak_consistency}]
	% Corollary analogue to 3.2 in Lei and Rinaldo
	Recall that $\rankspec$ is a full rank symmetric matrix with eigenvalues $|\lambda_{1}(\rankspec)| \ge \dots \ge |\lambda_{K}(\rankspec)| > 0$. Observe that
	\begin{equation*}
		\gamma_{n}
		=
		|\lambda_{K}(\Delta \rankspec \Delta)|
		=
		\frac{1}{\|\Delta^{-1}\rankspec^{-1}\Delta^{-1}\|_{\op}}
		\ge
		\frac{1}{n_{\min}^{-1} \cdot \|\rankspec^{-1}\|_{\op}}
		=
		n_{\min} \cdot |\lambda_{K}(\rankspec)|.
	\end{equation*}
	By hypothesis, $n_{\min} \cdot |\lambda_{K}(\rankspec)| > C_{2}^{\prime}n^{3/4+\epsP}$, hence \cref{result:weak_consistency_kmeans} applies. Thus, with probability at least $1 - n^{-\epsP}$,
	\begin{align*}
		L(\wh{\Theta}, \Theta)
		\le
		\frac{2}{n} \sum_{k=1}^{K} |\mathcal{S}_{k}|
		&\le
		\frac{2 n_{\max}^{\prime}}{n} \sum_{k=1}^{K} |\mathcal{S}_{k}| \left(\frac{1}{n_{k}} + \frac{1}{\max\{n_{\ell} : \ell \neq k\}}\right)\\
		&\le
		2c^{-1} (2+\epsK) \left(\frac{K n^{\epsP} n_{\max}^{\prime}}{n_{\min}^{2} \lambda_{K}^{2}(\rankspec)}\right).
	\end{align*}
	Furthermore,
	\begin{align*}
		\wt{L}(\wh{\Theta},\Theta)
		\le
		2\max_{1 \le k \le K} \frac{|\mathcal{S}_{k}|}{n_{k}}
		&\le
		2\sum_{k=1}^{K} |\mathcal{S}_{k}| \left(\frac{1}{n_{k}} + \frac{1}{\max\{n_{\ell} : \ell \neq k\}}\right)\\
		&\le
		2c^{-1} (2+\epsK) \left(\frac{K n^{1+\epsP}}{n_{\min}^{2} \lambda_{K}^{2}(\rankspec)}\right).
	\end{align*}
	This completes the proof of \cref{result:corollary_weak_consistency}.
\end{proof}

\begin{proof}[Proof of \cref{result:cor_top_eigen_relative_error}]
	We follow the proof approach used to establish \cref{result:weak_consistency_kmeans}. Under the stated hypotheses, with probability at least $1-n^{-\epsP}$, the bound
	\begin{equation*}
		\frac{\|\sampVecs\sampVecs^{\tp} - \popVecs\popVecs^{\tp} \|_{\frob}}{\| \popVecs\popVecs^{\tp} \|_{\frob}}
		\le
		\frac{\sqrt{2}\|\sampVecs - \popVecs Q\|_{\frob}}{K^{1/2}}
		\le
		\frac{C}{K^{1/2}}\left(\frac{K^{1/2} \cdot n^{1/2 +\epsP/2}}{\gamma_{n}}\right)
		\le
		\frac{C}{|\lambda_{K}(\rankspec)|} \left(\frac{n^{1/2+\epsP/2}}{n_{\min}} \right)
	\end{equation*}
	holds for some absolute constant $C > 0$. We have used the fact that $\|\popVecs\popVecs^{\tp}\|_{\frob} = \sqrt{K}$. Furthermore, $\| \sampVecs\sampVecs^{\tp} - \popVecs\popVecs^{\tp} \|_{\frob} \le \sqrt{2} \times {\inf}_{Q \textnormal{ orthogonal }}\|\sampVecs - \popVecs Q\|_{\frob}$ by basic properties of principal subspaces (e.g.,~\cite{cape2019two,cai2018rate}). This establishes \cref{eq:cor_top_eigen_relative_error} and hence proves \cref{result:cor_top_eigen_relative_error}.
\end{proof}

\begin{proof}[Proof of \cref{eq:cor_trace_correlation}]
	The stated equivalence holds as a consequence of basic properties of subspace perturbations (e.g.,~\cite{cape2019two,cai2018rate}). In particular,
	\begin{align*}
		\frac{\| \sampVecs\sampVecs^{\tp} - \popVecs\popVecs^{\tp} \|_{\frob}}{\| \popVecs\popVecs^{\tp} \|_{\frob}}
		&=
		\frac{\sqrt{2}}{\sqrt{K}} \|\sin\Theta(\sampVecs,\popVecs)\|_{\frob}
		=
		\frac{\sqrt{2}}{\sqrt{K}} \sqrt{K - \|\cos\Theta(\sampVecs,\popVecs)\|_{\frob}^{2}}\\
		&=
		\frac{\sqrt{2}}{\sqrt{K}} \sqrt{K - \|\popVecs^{\tp}\sampVecs\|_{\frob}^{2}}
		=
		\frac{\sqrt{2}}{\sqrt{K}} \sqrt{K - \|\popVecs\popVecs^{\tp}\sampVecs\sampVecs^{\tp}\|_{\frob}^{2}}\\
		&=
		\frac{\sqrt{2}}{\sqrt{K}} \sqrt{K - \trace(P_{\sampVecs} P_{\popVecs})}
		=
		\sqrt{2}\sqrt{1 - r_{K}^{2}(\sampVecs,\popVecs)}.
	\end{align*}
	The result follows for asymptotic regimes in which weak consistency holds, per \cref{eq:cor_top_eigen_relative_error}.
\end{proof}

\subsection{Proofs: node-specific strong consistency}
\label{app:strong_consistency}
\begin{remark}[Towards node-specific strong consistency:~notation and setup]
	\label[remark]{remark:towards_strong_consistency}
	It will be convenient to define the temporary notation $\sampMtx = \nrkmtx$, $\popMtx = \rankspecBig$, and $\noiseMtx = (\nrkmtx - \Ex[\nrkmtx]) - \diag(\rankspecBig)$. Hence, $\sampMtx = \popMtx + \noiseMtx$ since $\Ex[\nrkmtx] = \rankspecBig - \diag(\rankspecBig)$. The matrices $\sampMtx$, $\popMtx$, and $\noiseMtx$ are all $n \times n$ and symmetric.

	Write the full eigendecomposition of $\sampMtx$ as $\sampMtx = \sampVecs \sampVals \sampVecs^{\tp} + \sampVecs_{\perp} \sampVals_{\perp} \sampVecs_{\perp}^{\tp}$, where the diagonal matrix $\sampVals$ contains the $K$ largest-in-magnitude eigenvalues of $\sampMtx$ with corresponding orthonormal eigenvectors constituting the columns of $\sampVecs$. Using analogous notation, write the eigendecomposition of $\popMtx$ with $\rank(\popMtx) = K$ as $\popMtx = \popVecs \popVals \popVecs^{\tp}$.
\end{remark}

This section begins by establishing several technical preliminaries. We write $N_{k\ell} \equiv N(k,\ell) = |\{(i, j): i < j, g_{i} = k, g_{j} = \ell\}|$ to denote the number of random variables with block membership tuple $(k, \ell)$. Similarly, write $n_{k} = |\{i : g_{i} = k\}|$ for each $1 \le k \le K$.

\begin{lemma}
	\label[lemma]{result:bound_UTEU_op}
		Let $\epsP > 0$. For any choice of positive integers $K \ge 1$, $n_{k} \ge 1$ for all $1 \le k \le K$, and $n = \sum_{k=1}^{K} n_{k}$, it holds that
		\begin{equation}
			\label{eq:bound_UTEU_op}
			\Pr\left[ \|\popVecs^{\tp} \noiseMtx \popVecs\|_{\op} > 1 + K \cdot n^{\epsP/2} \right]
			\le
			8 n^{-\epsP}.
		\end{equation}
	Consequently,
	$\|\popVecs^{\tp} \noiseMtx \popVecs\|_{\op}
	=
	\Ohp(K \cdot n^{\epsP/2})$. 
\end{lemma}
\begin{proof}[Proof of \cref{result:bound_UTEU_op}]
	First, by the triangle inequality it holds that $\|\popVecs^{\tp} \noiseMtx \popVecs\|_{\op} \le \| \popVecs^{\tp} (\nrkmtx - \Ex[\nrkmtx]) \popVecs\|_{\op} + \|\diag(\rankspecBig)\|_{\op}$, where $ \|\diag(\rankspecBig)\|_{\op} = \max_{1 \le k \le K}|\rankspec_{kk}| \le 1$. Given the basic norm relationship $\|\cdot\|_{\op} \le \|\cdot\|_{\frob}$, consider the squared Frobenius norm quantity
	\begin{align*}
		&\| \popVecs^{\tp} (\nrkmtx - \Ex[\nrkmtx]) \popVecs\|_{\frob}^{2}\\
		&\hspace{1em}=
		\| \Delta^{-1} \Theta^{\tp} (\nrkmtx - \Ex[\nrkmtx]) \Theta \Delta^{-1}\|_{\frob}^{2}\\
		&\hspace{1em}=
		\sum_{1\le i,j \le K} \left(\sum_{1 \le k,\ell \le n} (\nrkmtx - \Ex[\nrkmtx])_{k \ell} \cdot \Theta_{k i} \Theta_{\ell j} \cdot n_{i}^{-1/2} n_{j}^{-1/2}\right)^{2} \\
		&\hspace{1em}=
		\sum_{1\le i,j \le K} \left(\sum_{k : g_{k} = i}\sum_{\ell : g_{\ell} = j} (\nrkmtx - \Ex[\nrkmtx])_{k \ell} \cdot n_{i}^{-1/2} n_{j}^{-1/2}\right)^{2} \\
		&\hspace{1em}\eqqcolon
		\sum_{1 \le i,j \le K} \xi_{i j}^{2}.
	\end{align*}
	Observe that $\Ex[\xi_{ij}] = 0$ for all pairs $(i,j)$. Moreover, \cref{prop:ranks_multi_variance,prop:ranks_multi_covariance} together yield
	\begin{equation*}
		\Ex[\xi_{ij}^{2}]
		=
		\Var[\xi_{ij}] \le \left( \frac{n_{i}n_{j}}{n_{i} n_{j}} + \frac{7n_{i}n_{j}}{N} \cdot \frac{n_{i}n_{j}}{n_{i} n_{j}} \right)
		\le
		8.
	\end{equation*}
	Finally, putting all the pieces together and applying Markov's inequality yields that for any choice $\epsP > 0$,
	\begin{align*}
		&\Pr\left[ \|\popVecs^{\tp} \noiseMtx \popVecs\|_{\op} > 1 + K \cdot n^{\epsP/2} \right]\\
		&\hspace{2em}\le
		\Pr\left[ \|\popVecs^{\tp} \noiseMtx \popVecs\|_{\frob} > 1 + K \cdot n^{\epsP/2} \right]\\
		&\hspace{2em}\le
		\Pr\left[ \| \popVecs^{\tp} (\nrkmtx - \Ex[\nrkmtx]) \popVecs\|_{\frob} + 1 > 1 + K \cdot n^{\epsP/2} \right]\\
		&\hspace{2em}=
		\Pr\left[\left(\sum_{1 \le i,j \le K}\xi_{ij}^{2} \right)^{1/2} > K \cdot n^{\epsP/2} \right]\\
		&\hspace{2em}=
		\Pr\left[\sum_{1 \le i,j \le K}\xi_{ij}^{2} > K^{2} \cdot n^{\epsP} \right]\\
		&\hspace{2em}\le
		\frac{\Ex\left[\sum_{1 \le i,j \le K}\xi_{ij}^{2} \right]}{K^{2} \cdot n^{\epsP}}\\
		&\hspace{2em}\le
		8n^{-\epsP}.
	\end{align*}
	This establishes \cref{result:bound_UTEU_op}.
\end{proof}

\begin{lemma}
	\label[lemma]{result:bound_EU_op}
	Let $\epsP > 0$. For any choice of positive integers $K \ge 1$, $n_{k} \ge 1$ for all $1 \le k \le K$, and $n = \sum_{k=1}^{K} n_{k}$, it holds that
	\begin{equation}
		\label{eq:bound_EU_op}
		\Pr\left[\|\noiseMtx \popVecs\|_{\op} > 1 + K^{1/2} \cdot n^{1/2 + \epsP/2}\right]
		\le
		\Pr\left[\|\noiseMtx \popVecs\|_{\frob} > K^{1/2} + K^{1/2} \cdot n^{1/2 + \epsP/2}\right]
		\le
		8n^{-\epsP}.
	\end{equation}
	Consequently,
	$\|\noiseMtx \popVecs\|_{\op}
	=
	\Ohp(K^{1/2} \cdot n^{1/2 + \epsP/2})$.
\end{lemma}
\begin{proof}[Proof of \cref{result:bound_EU_op}]
	We mirror the proof of \cref{result:bound_UTEU_op}. First, by the triangle inequality, $\|\noiseMtx \popVecs\|_{\op} \le \|(\nrkmtx - \Ex[\nrkmtx]) \popVecs\|_{\op} + 1$, since $\|\diag(\rankspecBig)\popVecs\|_{\op} \le 1$. Furthermore,
	\begin{align*}
		&\|(\nrkmtx - \Ex[\nrkmtx]) \popVecs\|_{\frob}^{2}\\
		&\hspace{1em}=
		\|(\nrkmtx - \Ex[\nrkmtx]) \Theta \Delta^{-1}\|_{\frob}^{2}\\
		&\hspace{1em}=
		\sum_{1 \le i \le n, 1 \le j \le K} \left( \sum_{\ell:g_{\ell}=j} (\nrkmtx - \Ex[\nrkmtx])_{i \ell} \cdot n_{j}^{-1/2} \right)^{2}\\
		&\hspace{1em}\eqqcolon
		\sum_{1 \le i \le n, 1 \le j \le K} \xi_{ij}^{2},
	\end{align*}
	where it follows from \cref{prop:ranks_multi_variance,prop:ranks_multi_covariance} that uniformly in the tuples $(i,j)$,
	\begin{equation*}
		\Ex[\xi_{ij}^{2}]
		=
		\Var[\xi_{ij}]
		\le
		\frac{n_{j}}{n_{j}} + \frac{7n_{j}^{2}}{n_{j} N}
		\le
		8.
	\end{equation*}
	By an application of Markov's inequality, we conclude that for any choice $\epsP > 0$,
	\begin{equation*}
		\Pr\left[\|\noiseMtx \popVecs\|_{\op} > 1 + K^{1/2} \cdot n^{1/2 + \epsP/2}\right]
		\le
		8n^{-\epsP}.
	\end{equation*}
	This establishes the first part of \cref{result:bound_EU_op}. The second part follows by the same approach, since $\|\noiseMtx \popVecs\|_{\frob} \le \|(\nrkmtx - \Ex[\nrkmtx]) \popVecs\|_{\frob} + \|\diag(\rankspecBig)\popVecs\|_{\frob}$ and $\|\diag(\rankspecBig)\popVecs\|_{\frob} \le \|\diag(\rankspecBig)\|_{\op}\|\popVecs\|_{\frob} \le K^{1/2}$.
\end{proof}

\begin{lemma}
	\label[lemma]{result:bound_EU_rowvec}
		Let $\epsP > 0$. For any choice of positive integers $K \ge 1$, $n_{k} \ge 1$ for all $1 \le k \le K$, and $n = \sum_{k=1}^{K} n_{k}$, it holds for each fixed row index $i$ that
	\begin{equation}
		\label{eq:bound_EU_rowvec}
		\Pr\left[\|\noiseMtx \popVecs\|_{i,\ell_{2}} > n_{g_{i}}^{-1/2} + K^{1/2} \cdot n^{\epsP/2}\right]
		\le
		8n^{-\epsP}.
	\end{equation}
	Consequently,
	$\|\noiseMtx \popVecs\|_{i,\ell_{2}}
	=
	\Ohp(K^{1/2} \cdot n^{\epsP/2})$.
\end{lemma}
\begin{proof}[Proof of \cref{result:bound_EU_rowvec}]
	We again mirror the proof of \cref{result:bound_UTEU_op}. For any choice of row index $i$, by the triangle inequality $\|\noiseMtx \popVecs\|_{i,\ell_{2}} = \|\noiseMtx \Theta \Delta^{-1}\|_{i,\ell_{2}} \le \|(\nrkmtx - \Ex[\nrkmtx]) \Theta \Delta^{-1}\|_{i,\ell_{2}} + n_{g_{i}}^{-1/2}$ since $\|\diag(\rankspecBig)\Theta \Delta^{-1}\|_{i,\ell_{2}} \le \|\Theta \Delta^{-1}\|_{i,\ell_{2}} = n_{g_{i}}^{-1/2}$. Furthermore,
	\begin{align*}
		&\|(\nrkmtx - \Ex[\nrkmtx]) \Theta \Delta^{-1}\|_{i,\ell_{2}}^{2}\\
		&\hspace{1em}=
		\sum_{1 \le j \le K} \left(\sum_{\ell : g_{\ell}=j} (\nrkmtx - \Ex[\nrkmtx])_{i\ell} \cdot n_{j}^{-1/2}\right)^{2}\\
		&\hspace{1em}\eqqcolon
		\sum_{1 \le j \le K} \xi_{ij}^{2}.
	\end{align*}
	Here, \cref{prop:ranks_multi_variance,prop:ranks_multi_covariance} imply
	\begin{equation*}
		\Ex[\xi_{ij}^{2}]
		=
		\Var[\xi_{ij}]
		\le
		\frac{n_{j}}{n_{j}} + \frac{7n_{j}^{2}}{n_{j} N}
		\le
		8.
	\end{equation*}
	By an application of Markov's inequality, for any choice $\epsP > 0$,
	\begin{equation*}
		\Pr\left[\|\noiseMtx \popVecs\|_{i,\ell_{2}} > n_{g_{i}}^{-1/2} + K^{1/2} \cdot n^{\epsP/2}\right]
		\le
		8n^{-\epsP}.
	\end{equation*}
	This establishes \cref{result:bound_EU_rowvec}.
\end{proof}

\begin{lemma}
	\label[lemma]{result:bound_E_rowvec}
	For any choice of row index $i$, the bound $\|\noiseMtx\|_{i,\ell_{2}} \le 2n^{1/2}$ holds surely.
\end{lemma}
\begin{proof}[Proof of \cref{result:bound_E_rowvec}]
	All the entries of $\nrkmtx - \Ex[\nrkmtx]$ and $\diag(\rankspecBig)$ are bounded in modulus by one, hence for any choice of row index $i$, by the triangle inequality $\|\noiseMtx\|_{i,\ell_{2}} = \|(\nrkmtx - \Ex[\nrkmtx]) - \diag(\rankspecBig)\|_{i,\ell_{2}} \le 2n^{1/2}$.
\end{proof}

\begin{lemma}
	\label[lemma]{result:bounds_misc_evec}
	Assume the setting and notation in \cref{sec:strong_consistency}. Fix $\epsP > 0$. Write $|\lambda_{K}| \equiv |\lambda_{K}(\rankspecBig)|$. If $|\lambda_{K}| \ge C n^{3/4+\epsP}$ for all $n \ge n_{0}$ where $n_{0}, C > 0$ are large constants, then
	\begin{align}
		\|\popVecs\|_{i,\ell_{2}}
			&=
			n_{g_{i}}^{-1/2}, \label{eq:tech_lemma_delocalize} \\
		\|(I - \popVecs \popVecs^{\tp})\sampVecs\|_{\op}
			&=
			\Ohp(K^{1/2} \cdot n^{1/2+\epsP/2} \cdot |\lambda_{K}|^{-1}), \label{eq:tech_lemma_sine_bound} \\
		\|\sampVecs - \popVecs W_{\star}\|_{\op} 
			&=
			\Ohp(K^{1/2} \cdot n^{1/2+\epsP/2} \cdot |\lambda_{K}|^{-1}), \label{eq:tech_lemma_sine1_bound} \\			
		\|\popVecs^{\tp} \sampVecs - W_{\star}\|_{\op}
			&=
			\Ohp(K \cdot n^{1 + \epsP} \cdot |\lambda_{K}|^{-2}), \label{eq:tech_lemma_sine2_bound}
	\end{align}
	where $W_{\star} \in \underset{W \textnormal{ orthogonal }}{\arg \inf}\|\sampVecs - \popVecs W\|_{\frob}^{2}$.
\end{lemma}
\begin{proof}[Proof of \cref{result:bounds_misc_evec}]
	\cref{eq:tech_lemma_delocalize} follows directly from \cref{result:blockmodel_eigenvectors}.
	
	\cref{eq:tech_lemma_sine_bound} holds since $\|(I - \popVecs \popVecs^{\tp})\sampVecs\|_{\op} = \|\sin\Theta(\sampVecs,\popVecs)\|_{\op} \le \frac{\sqrt{2}\|\noiseMtx\popVecs\|_{\op}}{|\lambda_{K}| - |\lambda_{K+1}|} = \Ohp(K^{1/2} \cdot n^{1/2+\epsP/2} \cdot |\lambda_{K}|^{-1})$, using Lemma~2.5 and Corollary~2.8 in \cite{chen2021spectral} together with our \cref{result:bound_EU_op} to bound $\|\noiseMtx\popVecs\|_{\op}$. For context, here $\lambda_{K+1} = 0$, and by slight overload of notation, $\Theta(\sampVecs,\popVecs)$ is the ordered diagonal matrix of canonical angles between the subspaces spanned by $\sampVecs$ and $\popVecs$.
	
	\cref{eq:tech_lemma_sine1_bound} holds since $\|\sampVecs - \popVecs W_{\star}\|_{\op} \le \|(I-\popVecs\popVecs^{\tp})\sampVecs\|_{\op} + \|\popVecs^{\tp} \sampVecs - W_{\star}\|_{\op} \le 2\|(I-\popVecs\popVecs^{\tp})\sampVecs\|_{\op} = \Ohp(K^{1/2} \cdot n^{1/2+\epsP/2} \cdot |\lambda_{K}|^{-1})$, where we have applied the triangle inequality with the orthogonal decomposition $I = \popVecs\popVecs^{\tp} + (I-\popVecs\popVecs^{\tp})$ and \cref{eq:tech_lemma_sine_bound,eq:tech_lemma_sine2_bound}.
	
	\cref{eq:tech_lemma_sine2_bound} holds since $\|\popVecs^{\tp} \sampVecs - W_{\star}\|_{\op} \le \|\sin\Theta(\sampVecs,\popVecs)\|_{\op}^{2} = \|(I-\popVecs\popVecs^{\tp})\sampVecs\|_{\op}^{2} = \Ohp(K \cdot n^{1 + \epsP} \cdot |\lambda_{K}|^{-2})$, which holds from Lemma~6.7 in \cite{cape2019two} followed by an application of \cref{eq:tech_lemma_sine_bound}.
		
	This completes the proof of \cref{result:bounds_misc_evec}.
\end{proof}

\begin{center}
	$(\star)$
\end{center}

We are now prepared to prove \cref{result:strong_consistency_general}, our general strong consistency result.

\begin{proof}[Proof of \cref{result:strong_consistency_general}]
	We adopt the notation in \cref{remark:towards_strong_consistency}. For \cref{result:operator_norm_concentration} and each technical lemma proved above, say the $i$-th, associate to it the ``good event'' $\mathcal{E}_{i}$ on which the stated $\Ohp(\cdot)$ bound holds. Applying DeMorgan's law and Boole's inequality yields that the ``bad'' event $\overline{\cap_{i}\mathcal{E}_{i}}$ satisfies $\Pr[\overline{\cap_{i}\mathcal{E}_{i}}] = \Pr[\cup_{i}\overline{\mathcal{E}}_{i}] \le \sum_{i=1}^{n_{\mathcal{E}}} \Pr[\overline{\mathcal{E}}_{i}] \le C n^{-\epsP}$ since $n_{\mathcal{E}}$ is an absolute constant and $\Pr[\overline{\mathcal{E}_{i}}] \le C_{i} n^{-\epsP}$ for each $i$. Finally, by taking complements, $\Pr[\cap_{i}\mathcal{E}_{i}] \ge 1-Cn^{-\epsP}$. Consequently, the analysis that follows holds with high probability on an underlying ``good event''.

	Weyl's inequality \citep[Theorem VIII.4.8]{bhatia2013matrix} guarantees that
	\begin{equation*}
		\max_{1 \le i \le n}\left| |\lambda_{i}(\sampMtx)| - |\lambda_{i}(\popMtx)| \right|
		\le
		\max_{1 \le i \le n}\left| \lambda_{i}(\sampMtx) - \lambda_{i}(\popMtx) \right|
		\le
		\|\noiseMtx\|_{\op},
	\end{equation*}
	hence it follows from the triangle inequality that $|\lambda_{i}(\sampMtx)| \ge |\lambda_{i}(\popMtx)| - \|\noiseMtx\|_{\op}$ for $1 \le i \le n$. Furthermore, $\|\noiseMtx\|_{\op} = \Ohp(n^{3/4+\epsP})$ by \cref{result:operator_norm_concentration}, and $|\lambda_{K}| \ge C n^{3/4 + \epsP}$ implies $|\wh{\lambda}_{i}|^{-1} = \Ohp(|\lambda_{i}|^{-1})$ for each $1 \le i \le K$.

	First, rewriting $\sampVecs$ and applying the relationship $\sampMtx = \popMtx + \noiseMtx$ yields
	\begin{equation}
		\sampVecs
		=
		\sampMtx \sampVecs \sampVals^{-1}
		=
		\popMtx \sampVecs \sampVals^{-1} + \noiseMtx \sampVecs \sampVals^{-1}.
	\end{equation}
	Expanding the first term on the right-hand side yields
	\begin{align*}
		\popMtx \sampVecs \sampVals^{-1}
		&=
		\popVecs \popVals \popVecs^{\tp} \sampVecs \sampVals^{-1}\\
		&=
		\popVecs \popVecs^{\tp} \sampVecs - \popVecs \popVecs^{\tp} \noiseMtx \sampVecs \sampVals^{-1}\\
		&=
		\popVecs W_{\operatorname{\star}} + \popVecs (\popVecs^{\tp} \sampVecs - W_{\star}) - \popVecs \popVecs^{\tp} \noiseMtx \sampVecs \sampVals^{-1}\\
		&=
		\popVecs W_{\operatorname{\star}} + \popVecs (\popVecs^{\tp} \sampVecs - W_{\star}) - \popVecs \popVecs^{\tp} \noiseMtx \popVecs \popVecs^{\tp} \sampVecs \sampVals^{-1} - \popVecs \popVecs^{\tp} \noiseMtx (I - \popVecs \popVecs^{\tp}) \sampVecs \sampVals^{-1},
	\end{align*}
	where the first equality invokes the low-rank structure of $\popMtx = \popVecs \popVals \popVecs^{\tp}$, the second equality invokes the Sylvester-style observation that
	\begin{equation}
		\label{eq:mod_sylvester}
		\popVals \popVecs^{\tp} \sampVecs
		=
		\popVecs^{\tp} \sampVecs \sampVals - \popVecs^{\tp} \noiseMtx \sampVecs,
	\end{equation}
	the third equality introduces the Frobenius-optimal orthogonal Procrustes matrix,
	\begin{equation}
		\label{eq:Wstar_frob}
		W_{\star}
		\in
		\underset{W \textnormal{orthogonal}}{\arg \inf}\|\sampVecs - \popVecs W\|_{\frob}^{2},
	\end{equation}
	and the final equality invokes the expansion
	\begin{equation}
		\label{eq:expand_identity_UUt}
		I
		=
		\popVecs \popVecs^{\tp} + (I - \popVecs \popVecs^{\tp})
		=
		\popVecs \popVecs^{\tp} + \popVecs_{\perp} \popVecs_{\perp}^{\tp}.
	\end{equation}
	Applying the above lemmas, together with basic properties of matrix norms, yields
	\begin{align*}
		&\|\popVecs (\popVecs^{\tp} \sampVecs - W_{\star})\|_{i,\ell_{2}}\\
			&\hspace{1em}\le \|\popVecs\|_{i,\ell_{2}} \|\popVecs^{\tp} \sampVecs - W_{\star}\|_{\op}
			= \Ohp( K \cdot n_{g_{i}}^{-1/2} \cdot n^{1 + \epsP} \cdot |\lambda_{K}|^{-2}),\\
		&\|\popVecs \popVecs^{\tp} \noiseMtx \popVecs \popVecs^{\tp} \sampVecs \sampVals^{-1}\|_{i,\ell_{2}}\\
			&\hspace{1em} \le \|\popVecs\|_{i,\ell_{2}} \|\popVecs^{\tp} \noiseMtx \popVecs\|_{\op} \|\popVecs^{\tp} \sampVecs\|_{\op} \|\sampVals^{-1}\|_{\op}
			= \Ohp\left( K \cdot n_{g_{i}}^{-1/2} \cdot n^{\epsP/2} \cdot |\lambda_{K}|^{-1}\right) ,\\
		&\|\popVecs \popVecs^{\tp} \noiseMtx (I - \popVecs \popVecs^{\tp}) \sampVecs \sampVals^{-1}\|_{i,\ell_{2}}\\
			&\hspace{1em} \le \|\popVecs\|_{i,\ell_{2}} \|\popVecs^{\tp} \noiseMtx \|_{\op} \|(I - \popVecs \popVecs^{\tp}) \sampVecs\|_{\op} \|\sampVals^{-1}\|_{\op}
			= \Ohp\left( K \cdot n_{g_{i}}^{-1/2} \cdot n^{1 +\epsP} \cdot |\lambda_{K}|^{-2} \right).
	\end{align*}
	Together, the sum of these terms is bounded in the manner
	\begin{equation}
		\label{eq:strong_term1}
		\Ohp\left( K \cdot n_{g_{i}}^{-1/2} \cdot n^{1 +\epsP} \cdot |\lambda_{K}|^{-2} \right).
	\end{equation}
	It remains to analyze $\noiseMtx \sampVecs \sampVals^{-1}$. Applying \cref{eq:mod_sylvester,eq:Wstar_frob,eq:expand_identity_UUt} yields the decomposition
	\begin{align*}
		&\noiseMtx \sampVecs \sampVals^{-1}\\
		&\hspace{1em}= \left\{ \noiseMtx \popVecs \popVecs^{\tp} \sampVecs \sampVals^{-1} \right\} + \left\{ \noiseMtx \popVecs_{\perp} \popVecs_{\perp}^{\tp} \sampVecs \sampVals^{-1} \right\}\\
		&\hspace{1em}= \left\{ \noiseMtx \popVecs \popVals^{-1} \left( \popVals \popVecs^{\tp} \sampVecs\right) \sampVals^{-1}\right\} + \left\{ \noiseMtx \popVecs_{\perp} \popVecs_{\perp}^{\tp} \sampVecs \sampVals^{-1} \right\}\\
		&\hspace{1em}= \left\{\noiseMtx \popVecs \popVals^{-1} W_{\star} + \noiseMtx \popVecs \popVals^{-1}(\popVecs^{\tp} \sampVecs - W_{\star}) - \noiseMtx \popVecs \popVals^{-1}(\popVecs^{\tp} \noiseMtx \sampVecs \sampVals^{-1}) \right\} + \left\{ \noiseMtx \popVecs_{\perp} \popVecs_{\perp}^{\tp} \sampVecs \sampVals^{-1} \right\}.
	\end{align*}
	The rightmost term above can be further decomposed as
	\begin{equation}
		\label{eq:Esq_resid}
		\noiseMtx \popVecs_{\perp} \popVecs_{\perp}^{\tp} \sampVecs \sampVals^{-1}
		= \noiseMtx \popVecs_{\perp} \popVecs_{\perp}^{\tp} \noiseMtx \sampVecs \sampVals^{-2}
		= \noiseMtx^{2} \sampVecs \sampVals^{-2} - \noiseMtx \popVecs \popVecs^{\tp} \noiseMtx \sampVecs \sampVals^{-2}.
	\end{equation}
	\cref{lem:ho_cctrtn} is an important technical result that enables this proof and the subsequent derivation of asymptotic normality. Without it, we are merely able to deduce that
	\begin{align*}
		\|\noiseMtx^{2} \sampVecs \sampVals^{-2}\|_{i,\ell_{2}}
		&= \|\noiseMtx^{2} (\sampVecs - \popVecs W_{\star} + \popVecs W_{\star} ) \sampVals^{-2}\|_{i,\ell_{2}}\\
		&\le \|\noiseMtx\|_{i,\ell_{2}} \|\noiseMtx\|_{\op} \|\sampVecs - \popVecs W_{\star}\|_{\op} \|\sampVals^{-2}\|_{\op} + \|\noiseMtx\|_{i,\ell_{2}} \|\noiseMtx \popVecs W_{\star}\|_{\op} \|\sampVals^{-2}\|_{\op}\\
		&= \Ohp\left(K^{1/2} \cdot n^{1 + 3/4 + 3\epsP/2} \cdot |\lambda_{K}|^{-3} + K^{1/2} \cdot n^{1+\epsP/2} \cdot |\lambda_{K}|^{-2}\right)\\
		&= \Ohp\left(K^{1/2} \cdot n^{1+\epsP/2} \cdot |\lambda_{K}|^{-2}\right).
	\end{align*}
	Instead, consider
	\begin{align*}
		\|\noiseMtx^{2} \sampVecs \sampVals^{-2}\|_{i,\ell_{2}}
		&= \|\noiseMtx^{2} (\popVecs\popVecs^{\tp} + \popVecs_{\perp}\popVecs_{\perp}^{\tp})\sampVecs \sampVals^{-2}\|_{i,\ell_{2}}\\
		&\le \|\noiseMtx^{2} \popVecs\popVecs^{\tp} \sampVecs \sampVals^{-2}\|_{i,\ell_{2}} + \|\noiseMtx^{2} \popVecs_{\perp}\popVecs_{\perp}^{\tp} \sampVecs \sampVals^{-2}\|_{i,\ell_{2}}\\
		&\le \|\noiseMtx^{2} \popVecs\|_{i,\ell_{2}} \| \sampVals^{-2}\|_{\op} + \|\noiseMtx\|_{i,\ell_{2}} \|\noiseMtx\popVecs_{\perp}\|_{\op}\|\popVecs_{\perp}^{\tp} \sampVecs\|_{\op} \|\sampVals^{-2}\|_{\op}\\
		&\le \|\noiseMtx^{2} \popVecs\|_{i,\ell_{2}} \| \sampVals^{-2}\|_{\op} + \Ohp\left(K^{1/2} \cdot n^{1+3/4+3\epsP/2} \cdot |\lambda_{K}|^{-3} \right).
	\end{align*}
	We need to bound $\|\noiseMtx^{2}\popVecs\|_{i,\ell_{2}}$. Expanding terms yields
	\begin{align*}
		&\|\noiseMtx^{2}\popVecs\|_{i,\ell_{2}}\\
		&\hspace{1em}=
		\left\|\left(\nrk{} - \Ex[\nrk{}] - \diag(\rankspecBig)\right)^{2}\popVecs\right\|_{i,\ell_{2}}\\
		&\hspace{1em}=
		\Big\|\Big((\nrk{} - \Ex[\nrk{}])^{2} - (\nrk{} - \Ex[\nrk{}])\diag(\rankspecBig)\\
		&\hspace{5em}-
		\diag(\rankspecBig)(\nrk{} - \Ex[\nrk{}]) + \diag^{2}(\rankspecBig)\Big)\popVecs\Big\|_{i,\ell_{2}},
	\end{align*}
	where
	\begin{align*}
		\|\diag^{2}(\rankspecBig)\popVecs\|_{i,\ell_{2}} &\le n_{g_{i}}^{-1/2},\\
		\|\diag(\rankspecBig)(\nrk{} - \Ex[\nrk{}])\popVecs\|_{i,\ell_{2}} &= \Ohp(K^{1/2}n^{\epsP/2}),\\
		\|(\nrk{} - \Ex[\nrk{}])\diag(\rankspecBig)\popVecs\|_{i,\ell_{2}} &= \Ohp(K^{1/2}n^{\epsP/2}).
	\end{align*}
	We must control $\|(\nrk{} - \Ex[\nrk{}])^{2} \popVecs\|_{i,\ell_{2}} = \|(\nrk{} - \Ex[\nrk{}])^{2} \Theta\Delta^{-1}\|_{i,\ell_{2}}$ whose square is given by
	\begin{align*}
		&\|(\nrk{} - \Ex[\nrk{}])^{2} \Theta\Delta^{-1}\|_{i,\ell_{2}}^{2}\\
		&\hspace{1em}= \sum_{1 \le j \le K}\left(e_{i}^{\tp} (\nrk{} - \Ex[\nrk{}])^{2}\Theta\Delta^{-1}e_{j}\right)^{2}\\
		&\hspace{1em}= \sum_{1 \le j \le K}\left( \sum_{1 \le k,\ell \le n} (\nrk{} - \Ex[\nrk{}])_{i k} (\nrk{} - \Ex[\nrk{}])_{k \ell} (\Theta\Delta^{-1})_{\ell j} \right)^{2}\\
		&\hspace{1em}= \sum_{1 \le j \le K} \frac{1}{n_{j}} \sum_{1 \le k,k^{\prime} \le n} \sum_{\ell,\ell^{\prime} : g_{\ell} = j, g_{\ell^{\prime}} = j} (\nrk{} - \Ex[\nrk{}])_{i k} (\nrk{} - \Ex[\nrk{}])_{i k^{\prime}} (\nrk{} - \Ex[\nrk{}])_{k \ell} (\nrk{} - \Ex[\nrk{}])_{k^{\prime} \ell^{\prime}}.
	\end{align*}
	By \cref{lem:ho_cctrtn}, its expectation is $\Oh\left(\max\left\{\frac{n^{2}}{n_{g_i}},\sqrt{\frac{K^{3}n^{3}}{n_{g_i}}} \right\}\right)$. Markov's inequality thus gives
	\begin{equation*}
		\|(\nrk{} - \Ex[\nrk{}])^{2} \Theta\Delta^{-1}\|_{i,\ell_{2}} = \Ohp\left(\max\left\{\frac{n}{n_{g_i}^{1/2}},\frac{K^{3/4}n^{3/4}}{n_{g_i}^{1/4}} \right\} \cdot n^{\epsP/2} \right).
	\end{equation*}
	Therefore,
	\begin{align}
		&\|\noiseMtx^{2}\sampVecs\sampVals^{-2}\|_{i,\ell_{2}} \nonumber\\
		&\hspace{1em}\le \|\noiseMtx^{2} \popVecs\popVecs^{\tp} \sampVecs \sampVals^{-2}\|_{i,\ell_{2}} + \|\noiseMtx^{2} \popVecs_{\perp}\popVecs_{\perp}^{\tp} \sampVecs \sampVals^{-2}\|_{i,\ell_{2}} \nonumber\\
		&\hspace{1em}= \Ohp\Bigg(\left\{\max\left\{\frac{n}{n_{g_i}^{1/2}},\frac{K^{3/4}n^{3/4}}{n_{g_i}^{1/4}} \right\} \cdot n^{\epsP/2} + K^{1/2} n^{\epsP/2} + n_{g_{i}}^{-1/2} \right\} \cdot |\lambda_{K}|^{-2} \\
		&\hspace{6em} +
		\left\{K^{1/2} n^{1+3/4+3\epsP/2}\right\} \cdot |\lambda_{K}|^{-3} \Bigg) \nonumber\\
		&\hspace{1em}= \Ohp\Bigg(\left\{\max\left\{\frac{n}{n_{g_i}^{1/2}},\frac{K^{3/4} \cdot n^{3/4}}{n_{g_i}^{1/4}} \right\} \cdot n^{\epsP/2} \right\} \cdot |\lambda_{K}|^{-2} \label{eq:strong_term2} \\
		&\hspace{6em} +
		\left\{K^{1/2} \cdot n^{1+3/4+3\epsP/2}\right\} \cdot |\lambda_{K}|^{-3} \Bigg). \nonumber
	\end{align}

	For the remaining term in \cref{eq:Esq_resid}, a preliminary bound is given by
	\begin{align*}
		\|\noiseMtx \popVecs \popVecs^{\tp} \noiseMtx \sampVecs \sampVals^{-2}\|_{i,\ell_{2}}
		&\le \|\noiseMtx \popVecs\|_{i,\ell_{2}} \|\noiseMtx \popVecs\|_{\op} \|\sampVals^{-2}\|_{\op} = \Ohp\left(K \cdot n^{1/2+\epsP} \cdot |\lambda_{K}|^{-2}\right).
	\end{align*}
	Alternatively, the above term can be bounded in the manner
	\begin{align*}
		&\|\noiseMtx \popVecs \popVecs^{\tp} \noiseMtx \sampVecs \sampVals^{-2}\|_{i,\ell_{2}}\\
		&\hspace{1em}\le \|\noiseMtx \popVecs \popVecs^{\tp} \noiseMtx \popVecs \popVecs^{\tp} \sampVecs \sampVals^{-2}\|_{i,\ell_{2}} + \|\noiseMtx \popVecs \popVecs^{\tp} \noiseMtx \popVecs_{\perp} \popVecs_{\perp}^{\tp} \sampVecs \sampVals^{-2}\|_{i,\ell_{2}}\\
		&\hspace{1em}\le \|\noiseMtx \popVecs\|_{i,\ell_{2}} \left\{ \|\popVecs^{\tp} \noiseMtx \popVecs\|_{\op} \|\popVecs^{\tp} \sampVecs\|_{\op} + \|\noiseMtx \popVecs\|_{\op} \|\popVecs_{\perp} \popVecs_{\perp}^{\tp} \sampVecs\|_{\op} \right\} \|\sampVals^{-2}\|_{\op}\\
		&\hspace{1em}= \Ohp\left(K^{1/2} \cdot n^{\epsP/2} \cdot \left\{ K \cdot n^{\epsP/2} + K \cdot n^{1+\epsP} \cdot |\lambda_{K}|^{-1} \right\} \cdot |\lambda_{K}|^{-2}\right)\\
		&\hspace{1em}= \Ohp\left(K^{3/2} \cdot n^{1 + 3\epsP/2} \cdot |\lambda_{K}|^{-3}\right).
	\end{align*}
	Hence, for this term,
		\begin{equation}
			\label{eq:strong_term3}
			\|\noiseMtx \popVecs \popVecs^{\tp} \noiseMtx \sampVecs \sampVals^{-2}\|_{i,\ell_{2}}
			= \Ohp\left(\min\left\{K \cdot n^{1/2+\epsP} \cdot |\lambda_{K}|^{-2}, K^{3/2} \cdot n^{1 + 3\epsP/2} \cdot |\lambda_{K}|^{-3}\right\}\right).
		\end{equation}
	Under the current assumptions, $|\lambda_{K}| \ge C n^{3/4+\epsP}$ and $K \le n^{1/2}$, whence it follows that $|\lambda_{K}| \ge C K^{1/2} n^{1/2+\epsP/2}$. Finally, a preliminary analysis of the earlier leading order term yields
	\begin{align}
		&\|\noiseMtx \popVecs \popVals^{-1} W_{\star} + \noiseMtx \popVecs \popVals^{-1}(\popVecs^{\tp} \sampVecs - W_{\star}) - \noiseMtx \popVecs \popVals^{-1}(\popVecs^{\tp} \noiseMtx \sampVecs \sampVals^{-1})\|_{i,\ell_{2}} \nonumber\\
		&\hspace{1em}\le \|\noiseMtx \popVecs\|_{i,\ell_{2}} \cdot \|\popVals^{-1}\|_{\op} \cdot \left\{ 1 + \|\popVecs^{\tp} \sampVecs - W_{\star}\|_{\op} + \|\popVecs^{\tp} \noiseMtx \sampVecs \sampVals^{-1}\|_{\op}\right\} \nonumber\\
		&\hspace{1em}= \Ohp\left(K^{1/2} \cdot n^{\epsP/2} \cdot |\lambda_{K}|^{-1} \cdot \left\{1 + \min\{1, K \cdot n^{1+\epsP} \cdot |\lambda_{K}|^{-2}\} + K^{1/2} \cdot n^{1/2 + \epsP/2} \cdot |\lambda_{K}|^{-1}\right\}\right) \nonumber\\
		&\hspace{1em}= \Ohp\left(K^{1/2} \cdot n^{\epsP/2} \cdot |\lambda_{K}|^{-1} \right). \label{eq:strong_term4}
	\end{align}
	Combining \cref{eq:strong_term1,eq:strong_term2,eq:strong_term3,eq:strong_term4} yields
	\begin{align*}
		\|\sampVecs-\popVecs W_{\star}\|_{i,\ell_{2}}
		&= \Ohp\Bigg( K^{1/2} \cdot n^{\epsP/2} \cdot |\lambda_{K}|^{-1}\\
		&\qquad\qquad + \min\left\{K \cdot n^{1/2+\epsP} \cdot |\lambda_{K}|^{-2}, K^{3/2} \cdot n^{1 + 3\epsP/2} \cdot |\lambda_{K}|^{-3}\right\}\\
		&\qquad\qquad + \max\left\{n_{g_{i}}^{-1/2} \cdot n,K^{3/4} \cdot n_{g_{i}}^{-1/4} \cdot n^{3/4} \right\} \cdot n^{\epsP/2} \cdot |\lambda_{K}|^{-2}\\
		&\qquad\qquad + K^{1/2} \cdot n^{1+3/4+3\epsP/2} \cdot |\lambda_{K}|^{-3}\\
		&\qquad\qquad + K \cdot n_{g_{i}}^{-1/2} \cdot n^{1 +\epsP} \cdot |\lambda_{K}|^{-2}\Bigg).
	\end{align*}
	Rewriting terms gives \cref{result:strong_consistency_general}, as desired.
\end{proof}

\begin{proof}[Proof of \cref{result:strong_consistency_corollary}]
	Since $n_{k}$ is order $n / K$ for each $1 \le k \le K$, the condition $|\lambda_{K}| \ge C n^{3/4+\epsP}$ further implies $K \le C^{\prime} n^{1/4-\epsP}$. We therefore have
	\begin{align*}
		\begin{split}
			&\|\sampVecs-\popVecs W_{\star}\|_{i,\ell_{2}}\\ %%%
			&\hspace{2em}= \Ohp\Bigg( K^{3/2} \cdot n^{-1+\epsP/2}\\
			&\hspace{6em}+ \min\left\{K^{3} \cdot n^{-3/2+\epsP}, K^{9/2} \cdot n^{-2 + 3\epsP/2}\right\}\\
			&\hspace{6em}+ \left\{\max\left\{K^{1/2}\cdot n^{1/2}, K \cdot n^{1/2} \right\} \cdot K^{2} \cdot n^{-2+\epsP/2} \right\} + \left\{K^{7/2} \cdot n^{-1-1/4+3\epsP/2}\right\}\\
			&\hspace{6em}+ K^{7/2} \cdot n^{-3/2 +\epsP}\Bigg)\\ %%%
			&\hspace{2em}= \Ohp\Bigg( K^{3/2} \cdot n^{-1+\epsP/2}\\
			&\hspace{6em}+ \min\left\{K^{3} \cdot n^{-3/2+\epsP}, K^{9/2} \cdot n^{-2 + 3\epsP/2}\right\}\\
			&\hspace{6em}+ K^{3} \cdot n^{-3/2 +\epsP/2} + K^{7/2} \cdot n^{-1-1/4+3\epsP/2}\Bigg)\\
			&\hspace{2em}= \Ohp\Bigg( K^{3/2} \cdot n^{-1+\epsP/2}\\
			&\hspace{6em}+ \min\left\{K^{3} \cdot n^{-3/2+\epsP}, K^{9/2} \cdot n^{-2 + 3\epsP/2}\right\}\\
			&\hspace{6em}+ K^{7/2} \cdot n^{-1-1/4+3\epsP/2}\Bigg)\\ %%%
			&\hspace{2em}= \Ohp\Bigg( K^{3/2} \cdot n^{-1+\epsP/2} + K^{7/2} \cdot n^{-1-1/4+3\epsP/2}\Bigg).
		\end{split}
	\end{align*}
	Consequently,
	\begin{equation*}
		\|\sampVecs-\popVecs W_{\star}\|_{i,\ell_{2}}
		=
		\Ohp\left( K^{3/2} \cdot n^{-1+\epsP/2} \cdot \left\{1 + K^{2} \cdot n^{-1/4+\epsP}\right\}\right),
	\end{equation*}
	and
	\begin{equation*}
		n_{g_{i}}^{1/2}\|\sampVecs-\popVecs W_{\star}\|_{i,\ell_{2}}
		=
		\Ohp\left( K \cdot n^{-1/2+\epsP/2} \cdot \left\{1 + K^{2} \cdot n^{-1/4+\epsP}\right\}\right)
		=
		\ohp(1).
	\end{equation*}
	This establishes \cref{result:strong_consistency_corollary}.
\end{proof}

A further refinement of the leading order term in the above proof is possible. In fact, it will be needed to establish asymptotic normality. Namely, we shall characterize $\noiseMtx \popVecs \popVals^{-1} W_{\star}$ and $\|\noiseMtx \popVecs \popVals^{-1}\|_{i,\ell_{2}} = \|\noiseMtx \popVecs \popVals^{-1} W_{\star}\|_{i,\ell_{2}}$ appearing in
\begin{equation*}
	\|\noiseMtx \popVecs \popVals^{-1} W_{\star} + \noiseMtx \popVecs \popVals^{-1}(\popVecs^{\tp} \sampVecs - W_{\star}) - \noiseMtx \popVecs \popVals^{-1}(\popVecs^{\tp} \noiseMtx \sampVecs \sampVals^{-1})\|_{i,\ell_{2}}
	=
	\|\noiseMtx \popVecs \popVals^{-1} W_{\star}\|_{i,\ell_{2}} \cdot (1 + \ohp(1)).
\end{equation*}

\begin{lemma}
	\label[lemma]{result:bound_EULambdaInvWstar_rownorm}
		Let $\epsP > 0$. For any choice of positive integers $K \ge 1$, $n_{k} \ge 1$ for all $1 \le k \le K$, and $n = \sum_{k=1}^{K} n_{k}$ with $\rank(\rankspec)=K$, it holds for each fixed row index $i$ that
	\begin{equation*}
		\Pr\left[\|\noiseMtx \popVecs \popVals^{-1}\|_{i,\ell_{2}}
		>
		\left\{ n_{\min}^{-1} + \sqrt{\sum_{j=1}^{K} n_{j}^{-1}} \cdot n^{\epsP/2} \right\} \cdot \|\rankspec^{-1}\|_{\op} \cdot n_{\min}^{-1/2} \right]
		\le
		8 n^{-\epsP}.
	\end{equation*}
\end{lemma}
\begin{proof}[Proof of~\cref{result:bound_EULambdaInvWstar_rownorm}]
	Let $\popMtx^{\dagger}$ denote the necessarily unique Moore--Penrose inverse of $\popMtx$, given by
	\begin{equation*}
		\popMtx^{\dagger}
		=
		(\Theta \rankspec \Theta^{\tp})^{\dagger}
		=
		\Theta (\Theta^{\tp} \Theta)^{-1} \rankspec^{-1} (\Theta^{\tp}\Theta)^{-1} \Theta^{\tp},
	\end{equation*}
	which when squared equals
	\begin{equation*}
		(\popMtx^{\dagger})^{2}
		=
		\Theta (\Theta^{\tp} \Theta)^{-1} \rankspec^{-1} (\Theta^{\tp}\Theta)^{-1} \rankspec^{-1} (\Theta^{\tp} \Theta)^{-1} \Theta^{\tp}.
	\end{equation*}
	The preceding matrix is positive semi-definite, hence there exists an orthogonal matrix $W_{\Theta}$ such that
	\begin{align*}
		\noiseMtx \popVecs \popVals^{-1} W_{\star}
		&=
		\noiseMtx \cdot \left\{ \popVecs \popVals^{-1} W_{\Theta}\right\} \cdot(W_{\Theta}^{\tp} W_{\star})\\
		&=
		\noiseMtx \cdot \left\{ \Theta (\Theta^{\tp} \Theta)^{-1} \rankspec^{-1} (\Theta^{\tp}\Theta)^{-1/2} \right\} \cdot (W_{\Theta}^{\tp} W_{\star}).
	\end{align*}
	Here, the treatment of orthogonal transformations is crucial but does not modify the $\ell_{2}$ row norm values.
	
	From the above expansion, it holds by deterministic properties of norms that
	\begin{align*}
		&\|\noiseMtx \popVecs \popVals^{-1} W_{\star}\|_{i,\ell_{2}}\\
		&\hspace{1em}= \left\|\noiseMtx \cdot \left\{ \Theta (\Theta^{\tp} \Theta)^{-1} \rankspec^{-1} (\Theta^{\tp}\Theta)^{-1/2} \right\}\right\|_{i,\ell_{2}}\\
		&\hspace{1em}= \left\|\left( (\nrkmtx - \Ex[\nrkmtx]) - \diag(\rankspecBig) \right) \cdot \left\{ \Theta \Delta^{-2} \rankspec^{-1} \Delta^{-1} \right\}\right\|_{i,\ell_{2}}\\
		&\hspace{1em}\le \left\{\|(\nrkmtx - \Ex[\nrkmtx]) \Theta \Delta^{-2}\|_{i,\ell_{2}} + n_{\min}^{-1} \right\} \cdot \|\rankspec^{-1} \Delta^{-1}\|_{\op}.
	\end{align*}
	Furthermore,
	\begin{align*}
		&\|(\nrkmtx - \Ex[\nrkmtx]) \Theta \Delta^{-2}\|_{i,\ell_{2}}^{2}\\
		&\hspace{1em}= \sum_{1 \le j \le K} \left(\sum_{\ell:g_{\ell}=j} (\nrkmtx 	- \Ex[\nrkmtx])_{i\ell} \cdot n_{j}^{-1}\right)^{2}\\
		&\hspace{1em}\eqqcolon \sum_{1 \le j \le K} \xi_{ij}^{2}.
	\end{align*}
	For each value of $j$, invoking \cref{prop:ranks_multi_variance,prop:ranks_multi_covariance} yields
	\begin{equation*}
		\Ex[\xi_{ij}^{2}]
		=
		\Var[\xi_{ij}]
		\le
		\frac{n_{j}}{n_{j}^{2}} + \frac{n_{j}^{2}}{n_{j}^{2}} \cdot \frac{7}{N}
		\le
		\frac{8}{n_{j}}.
	\end{equation*}
	Consequently, for $\epsP > 0$, by Markov's inequality,
	\begin{equation*}
		\Pr\left[\|(\nrkmtx - \Ex[\nrkmtx]) \Theta \Delta^{-2}\|_{i,\ell_{2}}
		> \sqrt{\sum_{j=1}^{K} n_{j}^{-1}} \cdot n^{\epsP/2}\right]
		\le
		8 n^{-\epsP}.
	\end{equation*}
	Combining the above observations yields
	\begin{equation*}
		\Pr\left[\|\noiseMtx \popVecs \popVals^{-1} W_{\star}\|_{i,\ell_{2}} > 	\left\{ n_{\min}^{-1} + \sqrt{\sum_{j=1}^{K} n_{j}^{-1}} \cdot n^{\epsP/2} \right\} \cdot \|\rankspec^{-1}\|_{\op} \cdot n_{\min}^{-1/2}\right]
		\le
		8 n^{-\epsP}.
	\end{equation*}
	This concludes the proof of \cref{result:bound_EULambdaInvWstar_rownorm} since $\|\cdot\|_{i,\ell_{2}}$ is right-orthogonal invariant.
\end{proof}

\subsection{Proofs: asymptotic normality}
\label{app:asymptotic_normality}
We divide the proof of our theorem into multiple parts.

\begin{proof}[Proof of \cref{lem:asymp_part1}]
	Assume the hypotheses and fix a tuple $(\alpha,\beta) \in \Nset{n} \times \Nset{K}$. Define the associated random variable
	\begin{equation}
		\label{eq:CLT:temp:rv1}
		\xi_{\alpha,\beta}
		=
		(\nrkmtx \Theta)_{\alpha,\beta}
		=
		\sum_{\ell=1}^{n} (\nrkmtx)_{\alpha, \ell} \Theta_{\ell, \beta}
		=
		\sum_{\ell : g_{\ell} = \beta} (\nrkmtx)_{\alpha,\ell}.
	\end{equation}
	It will be convenient to work with the following intermediary notation for converting properties of the vector $(X_{1}, \dots, X_{N})$ into matrix form. Specifically, define the hollow symmetric $n \times n$ matrix $Y = (Y_{ij})$ as $Y = \vech_{0}^{-1}(\{X_{i}\}_{i=1}^{N})$, formed by the inverse half-vectorization operation taking the upper-triangular elements by row and ignoring the main diagonal. Define the hollow symmetric matrix $\Upsilon = (\upsilon_{ij}) \in \{0,1\}^{n \times n}$ according to the binary indicator function $1\{\cdot\}$ as $\upsilon_{ij} = 1\{i=\alpha, g_{j}=\beta\}$, hence $\Upsilon$ depends on $\alpha$ and $\beta$. Further, let $c = (c_{1}, \dots, c_{N})^{\tp} = \vech_{0}(\Upsilon)$, again defined in terms of the upper-triangular entries excluding the diagonal, hence $c_{\iota} \in \{0,1\}$ for all $1 \le \iota \le N$. Define $a_{N}(\iota) = \varphi(\iota/(N+1))$ with $\varphi(t) = t$ for $0 < t < 1$.
	
	Per \cite{hajek1968asymptotic}, $c_{1}, \dots, c_{N}$ denote the \emphh{regression constants}, while $a_{N}(1),\dots,a_{N}(N)$ denote the \emphh{scores} generated by $\varphi$. We seek to verify the asymptotic normality of the simple linear rank statistic \citep[Equation 2.3]{hajek1968asymptotic} given by \cref{eq:CLT:temp:rv1}, equivalently written as
	\begin{equation}
		\xi_{\alpha,\beta}
		=
		\sum_{1 \le i, j \le n} \upsilon_{ij} \cdot (\nrk{A})_{ij}
		=
		\sum_{\iota = 1}^{N} c_{\iota} \cdot a_{N}(R_{\iota}),
	\end{equation}
	where $R_{\iota}$ denotes the rank associated with the random variable $X_{\iota}$ having distribution function $F_{\iota}$.
	
	As in \citet[Equation 2.6]{hajek1968asymptotic}, define $\overline{c}$ and $H(y)$, respectively, as
	\begin{equation}
		\overline{c} 
		=
		\frac{1}{N}\sum_{\iota=1}^{N} c_{\iota}
		\asymp
		\frac{n_{\beta}}{N},
		\quad \quad \quad
		H(y)
		=
		\frac{1}{N} \sum_{\iota=1}^{N} F_{\iota}(y)
		=
		\sum_{m^{\prime} \in \Nset{m}} \frac{N_{m^{\prime}}}{N} F_{m^{\prime}}(y),
	\end{equation}
	where the second expression (property) holds for the present setting.
	
	Fix $\epsilon > 0$. For $n$ sufficiently large (hence $N$ sufficiently large), the following deterministic relationship holds for some $K_{\epsilon} > 0$.
	\begin{align*}
		\Var[\xi_{\alpha,\beta}]
		&=
		\sum_{\ell:g_{\ell}=\beta} \Var[(\nrkmtx)_{\alpha,\ell}]
		+
		\sum_{1 \le \ell < \ell^{\prime} \le n: g_{\ell} = \beta, g_{\ell^{\prime}}=\beta} 2 \cdot \Cov[(\nrkmtx)_{\alpha,\ell},(\nrkmtx)_{\alpha,\ell^{\prime}}]\\
		&\ge
		C^{\prime} n_{\beta} \cdot \delta_{\operatorname{Var}}
		-
		C^{\prime\prime} \cdot\frac{n_{\beta}^{2}}{N}\\
		&\ge
		C^{\prime\prime\prime} n_{\beta}\\
		&\gg
		K_{\epsilon} \max_{1 \le \iota \le N}(c_{\iota} - \bar{c})^{2}.
	\end{align*}
	The first line holds by the law of total variance. The second line holds by invoking a uniform lower bound on the (co)variances and $\liminf_{n\rightarrow\infty}\min_{i,j}\Var[(\nrk{A})_{ij}] \ge \delta_{\operatorname{Var}} > 0$ and $n_{\beta} \rightarrow \infty$. The third line holds provided $n$ is sufficiently large. For the fourth line, in the language of \citet[Equation 5.5]{hajek1968asymptotic},
	\begin{equation*}
		K_{\epsilon}
		=
		\left[2 \cdot \delta_{\epsilon}^{-1} \cdot 1 + (2 \cdot \epsilon^{-1/2} \cdot \beta_{\epsilon}^{-1} + 1) \sqrt{39} \right]^{2},
		\quad \quad
		0 <
		\max_{1\le \iota \le N} (c_{\iota} - \overline{c})^{2}
		<
		1,
	\end{equation*}
	using the fact that $\delta \equiv \delta_{\epsilon} > 0$ is chosen according to an application of the Lindeberg theorem \citep[Page 340]{hajek1968asymptotic}, that $\sup_{0 < t < 1}|\varphi^{\prime}(t)| = 1$, that $\beta \equiv \beta_{\epsilon} > 0$ is governed by the cumulative distribution function of the standard normal distribution \citep[Page 340]{hajek1968asymptotic}, and $M^{1/2} = \sqrt{39}$ where $M = 2K + 36 K_{2}^{2}$, $K = K_{1} + \frac{1}{2}K_{2}$, and $K_{i}$ denotes an upper bound on the squared $i$-th derivative of $\varphi$, with $K_{1} = K_{2} = 1$.
	
	Thus, invoking \citet[Theorem 2.1]{hajek1968asymptotic} for the aforementioned arbitrary choice $\epsilon > 0$ yields asymptotic normality of the form
	\begin{equation}
		\label{eq:Hajek_normality}
		\max_{x \in \R} \left| \Pr\left[\frac{\xi_{\alpha,\beta} - \Ex[\xi_{\alpha,\beta}]}{\sqrt{\Var[\xi_{\alpha,\beta}]}} \le x\right] - \Phi_{1}(x)\right| < \epsilon,
	\end{equation}
	where $\Phi_{1}$ denotes the cumulative distribution function of the standard normal distribution in $\R$. This concludes the proof of \cref{lem:asymp_part1} modulo a change of notation.
\end{proof}

\begin{proof}[Proof of \cref{lem:asymp_part2}]
	Under the stated assumptions, the argument in the proof of \cref{lem:asymp_part1} holds for each choice of $\beta \in \Nset{K}$ with index $\alpha$ held fixed. Now, consider the $K$-dimensional random vector $\xi_{\alpha} \equiv \xi_{\alpha}^{(n)} = (\xi_{\alpha,1}^{(n)}, \dots, \xi_{\alpha,K}^{(n)})^{\tp}$. Let $a \in \R^{K}$ be any specified non-zero (say unit norm) vector of constants. The random variable $a^{\tp} \xi_{\alpha}$ is itself a linear combination of simple linear rank statistics under so-called multi-sample (block) alternatives, so by another application of \citet[Theorem 2.1]{hajek1968asymptotic} and the above proof approach for \cref{eq:Hajek_normality}, we obtain asymptotic (univariate) normality of $a^{\tp} \xi_{\alpha}$ after standardization. An application of the Cram\'{e}r--Wold theorem \citep[Theorem 29.4]{billingsley2012probability} yields the asymptotic multivariate normality of $\xi_{\alpha}$ after standardization, namely
	\begin{equation*}
		\max_{x \in \R^{K}} \left| \Pr\left[ \Cov(\xi_{\alpha})^{-1/2}(\xi_{\alpha} - \Ex[\xi_{\alpha}]) \le x \right] - \Phi_{K}(x)\right| < \epsilon.
	\end{equation*}
	This concludes the proof of \cref{lem:asymp_part2}. 
\end{proof}

\begin{center}
	$(\star)$
\end{center}

We are now in a position to prove \cref{thrm:asymp_norm}, the main asymptotic normality result.
\begin{proof}[Proof of \cref{thrm:asymp_norm}]
	\newcommand{\resid}{E^{\prime}}
	From the proof of \cref{result:strong_consistency_general}, there exists a random matrix $\resid \in \R^{n \times K}$ such that with probability one, \begin{equation*}
		\sampVecs - \popVecs W_{\star}
		=
		\noiseMtx \popVecs \popVals^{-1} W_{\star} + \resid.
	\end{equation*}
	Here, $K$ is fixed, $|\lambda_{K}|$ is order $n/K$, and $\|\resid\|_{i,\ell_{2}} = \ohp(n^{-1})$ for each row index $i$.
	
	From the proof of \cref{result:bound_EULambdaInvWstar_rownorm}, there exist orthogonal matrices $W_{\star}$ and $W_{\Theta}$, each depending on $n$, such that
	\begin{align*}
		\noiseMtx \popVecs \popVals^{-1} W_{\star}
		&= \left\{\noiseMtx \Theta\right\} \cdot \left\{ (\Theta^{\tp} \Theta)^{-1} \rankspec^{-1} (\Theta^{\tp}\Theta)^{-1/2} \right\} \cdot (W_{\Theta}^{\tp} W_{\star}).
	\end{align*}
	Write $W^{\tp} = W_{\Theta}^{\tp}W_{\star}$. Recall that $\noiseMtx = (\nrkmtx - \Ex[\nrkmtx]) - \diag(\rankspecBig)$, whence these observations together yield
	\begin{align*}
		&\left(n\left(\sampVecs - \popVecs W_{\star}\right) W\right)_{\operatorname{row} i}\\
		&\hspace{1em}= \left(\left\{n^{-1/2}(\nrk{A} - \Ex[\nrk{A}]) \Theta\right\} \cdot \left\{ n(\Theta^{\tp} \Theta)^{-1} \cdot \rankspec^{-1} \cdot n^{1/2}(\Theta^{\tp}\Theta)^{-1/2} \right\}\right)_{\operatorname{row} i} + \ohp(1).
	\end{align*}
	By hypothesis,
	\begin{align*}
		\Gamma^{(k)}
		&=
		\lim_{n \rightarrow\infty} \Cov\left[n^{-1/2}((\nrk{A} - \Ex[\nrk{A}])\Theta)_{\operatorname{row} i} : g_{i}=k\right],\\
		\Xi
		&=
		\lim_{n \rightarrow\infty} \left\{n^{1/2}(\Theta^{\tp}\Theta)^{-1/2} \right\} \cdot \rankspec^{-1} \cdot \left\{n(\Theta^{\tp} \Theta)^{-1}\right\},
	\end{align*}
	where each limiting matrix has full rank. Finally, by an application of \cref{lem:asymp_part2} and Slutsky's theorem, the sequence of row vectors of the form
	\begin{equation*}
		\left(n\left(\sampVecs - \popVecs W_{\star}\right) W\right)_{\operatorname{row} i} : g_{i}=k
	\end{equation*}
	converges in distribution to a $K$-dimensional multivariate normal random vector with mean zero and covariance matrix $\Xi \cdot \Gamma^{(k)} \cdot \Xi^{\tp}$. This concludes the proof of \cref{thrm:asymp_norm}.
\end{proof}

\subsection{Proofs: dimension selection}
\label{app:additional_proofs}
\begin{proof}[Proof of \cref{lem:selectK}]
	By the proof of \cref{result:operator_norm_concentration}, for $n$ sufficiently large it holds that
	\begin{equation*}
		\Pr\left[\|\nrk{A} - \rankspecBig\|_{\op} > 4n^{3/4+\epsP}\right]
		\le
		n^{-\epsP}.
	\end{equation*}
	Furthermore, by the proof of \cref{result:corollary_weak_consistency} it holds that
	\begin{equation*}
		|\lambda_{K}(\rankspecBig)|
		\ge
		n_{\min} \cdot |\lambda_{K}(\rankspec)|
		>
		0.
	\end{equation*}
	Weyl's inequality \citep[Theorem VIII.4.8]{bhatia2013matrix} guarantees that for each $i \in \Nset{n}$,
	\begin{equation*}
		\left| |\lambda_{i}(\nrk{A})| - |\lambda_{i}(\rankspecBig)|\right|
		\le
		\|\nrk{A} - \rankspecBig\|_{\op}.
	\end{equation*}
	Hence, for each $k \in \Nset{K}$ and $n$ large, it holds with probability at least $1-n^{-\epsP}$ that
	\begin{equation*}
		|\lambda_{k}(\nrk{A})|
		\ge
		|\lambda_{k}(\rankspecBig)| - \|\nrk{A} - \rankspecBig\|_{\op}
		\ge
		n_{\min} \delta - 4n^{3/4+\epsP}
		\gg
		4n^{3/4+\epsP}.
	\end{equation*}
	Similarly, for each $i \ge 1$ and $n$ sufficiently large, we have $|\lambda_{K+i}(\nrk{A})| > 0 + 4n^{3/4+\epsP}$ with probability at most $n^{-\epsP}$. Taking limits with $n\rightarrow\infty$ establishes \cref{lem:selectK} by the squeeze principle.
\end{proof}

\subsection{Proofs: properties of rank statistics}
\label{app:rank_statistic_proofs}
This section begins by establishing notation and several preliminary observations.

Let $X, Y, Z$ be independent absolutely continuous random variables with cumulative distribution functions $F_{X}, F_{Y}, F_{Z}$ and densities $f_{X}, f_{Y}, f_{Z}$. A basic but important observation is that
\begin{equation*}
	\Pr_{(X,Y)}[X \le Y]
	=
	\Ex_{Y}[F_{X}(Y)].
\end{equation*}
Let $g_{(X,Y)} : \R \rightarrow \R$ be the deterministic map $g_{(X,Y)}(z) = \int_{-\infty}^{z}F_{X}(y)f_{Y}(y)\,\textsf{d}y$. One can show that
\begin{equation}
	\Pr_{(X,Y,Z)}[X \le Y \le Z]
	=
	\Ex_{Z}[g_{(X,Y)}(Z)].
\end{equation}
Also, keep in mind that
$\Ex_{X}[F_{X}(X)] = \tfrac{1}{2}$
and
$\Ex_{X}[F_{X}^{2}(X)]=\tfrac{1}{3}$.

As before, let $\{X_{i}\}_{i \in \Nset{N}}$ be a collection of $N$ independent random variables, and let $\{f_{\ell}\}_{\ell \in \Nset{\numSample}}$ be a collection of $\numSample$ density functions, where $\{N_{\ell}\}_{\ell=1}^{\numSample}$ is a deterministic collection of integer counts satisfying $N_{\ell} = |\{i : X_{i} \sim f_{\ell}\}|$ with $ \sum_{\ell=1}^{\numSample} N_{\ell} = N$. Let $F_{\ell}$ denote the cumulative distribution function associated with density $f_{\ell}$, and as before let $\Ex_{\ell}[X_{i}]$ denote taking expectation when $X_{i} \sim f_{\ell}$, at times also written as $X_{i} \sim F_{\ell}$.

Let $H : \R \rightarrow \{0, \tfrac{1}{2}, 1\}$ denote the Heaviside step function using the half-maximum convention, namely
\begin{equation*}
	H(x)
	=
	\begin{cases}
		1 & \text{if } x>0,\\
		\tfrac{1}{2} & \text{if } x=0,\\
		0 & \text{if } x<0.
	\end{cases}
\end{equation*}

\begin{proof}[Proof of \cref{prop:ranks_one_sample}]	
	The stated result follows directly from numerous textbooks and published papers. For example, see Example~2.1 in \cite{viana2001covariance}.
\end{proof}

\begin{proof}[Proof of \cref{prop:ranks_multi_expectation}]
	Let the random variable $\intermRank$ denote the normalized rank value of $X_{i}$ among $X_{1}, \dots, X_{N}$, such that the possible outcomes of $\intermRank$ are given by $\left\{\frac{1}{N}, \frac{2}{N}, \dots, 1\right\}$. Conveniently, $\intermRank$ is expressible in the form
	\begin{equation}
		\label{eq:nnrank}
		\intermRank
		=
		\frac{1}{N}\sum_{i^{\prime} \in \Nset{N}} H(X_{i} - X_{i^{\prime}}) + \frac{1}{2N}.
	\end{equation}
	Above, the term $\tfrac{1}{2N}$ is a consequence of $H(0) = \frac{1}{2}$ when $i^{\prime} = i$ and having specified that $\intermRank \in \left\{\frac{1}{N}, \tfrac{2}{N}, \dots, 1\right\}$.
	
	For $k \in \Nset{m}$, it follows from the above display equation that
	\begin{align*}
		\Ex_{k}\left[\intermRank\right]
		=
		\sum_{\ell \in \Nset{\numSample}}\frac{N_{\ell}}{N}\times\Ex_{{k}}[F_{{\ell}}(X)]
		+
		\frac{1}{2N}.
	\end{align*}
	Finally, replace $\ell$ with $\ell^{\prime}$ and $k$ with $\ell$ in the index notation. This establishes \cref{prop:ranks_multi_expectation}.
\end{proof}

\begin{proof}[Proof of \cref{prop:ranks_multi_variance}]	
	We pursue an explicit variance calculation. Towards this end, consider
	\begin{equation*}
		\frac{1}{N^{2}} \times \Ex_{k}\left[\sum_{i^{\prime}, j^{\prime} \in \Nset{N}} H(X_{i}-X_{i^{\prime}}) \cdot H(X_{i}-X_{j^{\prime}})\right], \quad X_{i} \sim F_{k}.
	\end{equation*}
	Applying linearity of expectation, we shall count and characterize the various enumerations of $\Ex_{k}[H(X_{i} - X_{i^{\prime}}) \cdot H(X_{i} - X_{j^{\prime}})]$. Here, $X_{i} \sim F_{k}$ is a representative of $N_{k}$ i.i.d.~random variables, whereas $X_{i^{\prime}}$ and $X_{j^{\prime}}$ range over all $N$ variables and $\numSample$ distributions.

\begin{enumerate}
	\item \label{case:var_pf_1} For $i = i^{\prime} = j^{\prime}$, there is one term of the form
	\begin{equation*}
		\Ex_{k}\left[H(X_{i}-X_{i^{\prime}})H(X_{i}-X_{j^{\prime}})\right]
		=
		\Ex_{k}[H^{2}(0)]
		=
		\tfrac{1}{4}.
	\end{equation*}
	
	\item \label{case:var_pf_2} For $i = i^{\prime}, i^{\prime} \neq j^{\prime}$ where $X_{j^{\prime}}\sim F_{\ell}$, there are, by symmetry, $2N - 2$ terms of the form
	\begin{equation*}
		\Ex_{k}\left[H(X_{i}-X_{i^{\prime}})H(X_{i}-X_{j^{\prime}})\right]
		=
		\Ex_{k}\left[H(0)H(X_{i}-X_{j^{\prime}})\right]
		=
		\tfrac{1}{2}\Ex_{{k}}[F_{{\ell}}(X)].
	\end{equation*}
	
	\item \label{case:var_pf_3} For $i \neq i^{\prime}, i^{\prime}=j^{\prime}$ with $X_{i^{\prime}}\sim F_{\ell}$, there are $N-1$ terms of the form 
	\begin{equation*}
		\Ex_{k}\left[H(X_{i}-X_{i^{\prime}})H(X_{i}-X_{j^{\prime}})\right]
		=
		\Ex_{k}\left[H^{2}(X_{i}-X_{i^{\prime}})\right]
		=
		\Ex_{{k}}[F_{{\ell}}(X)].
	\end{equation*}
	
	\item \label{case:var_pf_4} For distinct triples $\{i,i^{\prime}, j^{\prime}\}$ where $X_{i^{\prime}} \sim F_{\ell}$ and $X_{j^{\prime}} \sim F_{m}$, there are $N^{2} - 3N + 2$ terms of the form
	\begin{align*}
		\Ex_{k}\left[H(X_{i}-X_{i^{\prime}})H(X_{i}-X_{j^{\prime}})\right]
		&= \Prob[\{X_{i^{\prime}} < X_{i}\} \cap \{X_{j^{\prime}} < X_{i}\}]\\
		&= \Prob[\max(X_{i^{\prime}},X_{j^{\prime}}) < X_{i}]\\
		&= \Ex_{{k}}[F_{{\ell}}(X)F_{{m}}(X)].
	\end{align*}
	Importantly, when $\ell=\numSample=k$, then $\Ex_{{k}}[F_{{k}}(X)F_{k}(X)]
	=\Ex_{{k}}[F_{{k}}^{2}(X)]
	=\frac{1}{3}$.
\end{enumerate}
\cref{case:var_pf_1} above contributes the (scaled) term $\frac{1}{4}\frac{1}{N^{2}}$. For each of \cref{case:var_pf_2} and \cref{case:var_pf_3} above, scaling and aggregating yields $\frac{1}{N^{2}}\times\left((N_{k}-1)\frac{1}{2}+\sum_{\ell \neq k}N_{\ell}\Ex_{k}[F_{\ell}(X)]\right)$. In total, these three cases contribute
\begin{equation}
	\frac{2}{N}\sum_{\ell \in \Nset{\numSample}}\frac{N_{\ell}}{N}\Ex_{k}[F_{\ell}(X)]
	- \frac{3}{4}\frac{1}{N^{2}}.
\end{equation}
\cref{case:var_pf_4} above, the primary term of interest, contributes
\begin{align*}
	&\frac{1}{N^{2}} \times \Bigg((N_{k}-1)(N_{k}-2)\Ex_{k}[F_{k}^{2}(X)]
	+ 2\sum_{\ell \neq k}(N_{k}-1)N_{\ell}\Ex_{k}[F_{k}(X)F_{\ell}(X)]\\
	&\hspace{2em}+ \sum_{\ell \neq k, \ell^{\prime} \neq k, \ell \neq \ell^{\prime}} N_{\ell}N_{\ell^{\prime}}\Ex_{k}[F_{\ell}(X)F_{\ell^{\prime}}(X)]
	+ \sum_{\ell \neq k, \ell^{\prime} \neq k, \ell = \ell^{\prime}} N_{\ell}(N_{\ell^{\prime}}-1)\Ex_{k}[F_{\ell}(X)F_{\ell^{\prime}}(X)] \Bigg).
\end{align*}
The above expression can be viewed as an incomplete expansion of a quadratic form. To complete the above expansion, we need to invoke a plus-zero trick with
\begin{equation}
	\pm\frac{1}{N^{2}}\left\{\left((3N_{k}-2)\frac{1}{3}\right)
	+ \left(2\sum_{\ell \neq k}N_{\ell}\Ex_{k}[F_{k}(X)F_{\ell}(X)]\right)
	+ \left(\sum_{\ell\neq k}N_{\ell} \Ex_{k}[F_{\ell}^{2}(X)]\right) \right\}.
\end{equation}
The ``plus'' part of the zero trick yields the completed quadratic form
\begin{equation}
	\sum_{\ell, \ell^{\prime} \in \Nset{m}}\frac{N_{\ell}}{N}\frac{N_{\ell^{\prime}}}{N} \times \Ex_{k}[F_{\ell}(X)F_{\ell^{\prime}}(X)].
\end{equation}
The variance computation is nearly complete. Observe that the squared expectation of $\intermRank$ is given by
\begin{equation*}
	\left(\Ex_{k}[\intermRank]\right)^{2}
	= \left(\sum_{\ell \in \Nset{m}}\frac{N_{\ell}}{N}\times\Ex_{{k}}[F_{{\ell}}(X)]\right)^{2} + \frac{1}{N}\sum_{\ell \in \Nset{m}}\frac{N_{\ell}}{N}\times\Ex_{{k}}[F_{{\ell}}(X)] + \frac{1}{4}\frac{1}{N^{2}},
\end{equation*}
while the second moment is given by
\begin{align*}
	&\Ex_{k}\left[(\intermRank)^{2}\right]\\
	&= \Ex_{k}\left[\left(\frac{1}{N}\sum_{i^{\prime} \in \Nset{N}} H(X_{i}-X_{i^{\prime}})+\frac{1}{2N}\right)^{2}\right]\\
	&= \frac{1}{N^{2}} \times \Ex_{k}\left[\left(\sum_{i^{\prime} \in \Nset{N}} H(X_{i}-X_{i^{\prime}})\right)^{2} + \sum_{i^{\prime}\in\Nset{\numSample}} H(X_{i}-X_{i^{\prime}}) + \frac{1}{4}\right]\\
	&= \frac{1}{N^{2}} \times \Ex_{k}\left[\sum_{i^{\prime}, j^{\prime} \in \Nset{N}} H(X_{i}-X_{i^{\prime}})H(X_{i}-X_{j^{\prime}})\right] + \frac{1}{N}\sum_{\ell \in \Nset{\numSample}}\frac{N_{\ell}}{N}\times\Ex_{{k}}[F_{{\ell}}(X)]
	+ \frac{1}{4}\frac{1}{N^{2}}.
\end{align*}
More immediately, by the shift-invariance of variance,
\begin{equation*}
	\Var_{k}[\intermRank]
	= \frac{1}{N^{2}} \times \Ex_{k}\left[\sum_{i^{\prime}, j^{\prime} \in \Nset{N}} H(X_{i}-X_{i^{\prime}})H(X_{i}-X_{j^{\prime}})\right]
	- \left(\sum_{\ell \in \Nset{\numSample}}\frac{N_{\ell}}{N}\times\Ex_{{k}}[F_{{\ell}}(X)]\right)^{2}.
\end{equation*}
Thus,
\begin{align*}
	\Var_{k}[\intermRank]
	&= 
	\left(\sum_{\ell, \ell^{\prime} \in \Nset{\numSample}}\frac{N_{\ell}}{N}\frac{N_{\ell^{\prime}}}{N} \times\Ex_{k}[F_{\ell}(X)F_{\ell^{\prime}}(X)]
	- \left(\sum_{\ell \in \Nset{\numSample}}\frac{N_{\ell}}{N}\times\Ex_{{k}}[F_{{\ell}}(X)]\right)^{2}\right)\\
	&\hspace{2em}+ \left(\sum_{\ell \in \Nset{\numSample}}\frac{N_{\ell}}{N}\times\left\{\Ex_{k}[2F_{\ell}(X)] - \Ex_{k}[2F_{k}(X)F_{\ell}(X)] - \Ex_{k}[F_{\ell}^{2}(X)]\right\}\right)\frac{1}{N}\\
	&\hspace{2em}-
	\frac{1}{12}\frac{1}{N^{2}}.
\end{align*}
The result follows by replacing $\ell^{\prime}$ with $\ell^{\prime\prime}$, $\ell$ with $\ell^{\prime}$, and $k$ with $\ell$ in the notation above. This establishes \cref{prop:ranks_multi_variance}.
\end{proof}

\begin{proof}[Proof of \cref{prop:ranks_multi_covariance}]
	By the shift-invariance of covariance,
	\begin{align*}
		&\operatorname{Cov}_{(k,\ell)}(\wt{R}_{N i}^{\prime}, \wt{R}_{N j}^{\prime})\\
		&= \Ex[\wt{R}_{N i}^{\prime} \wt{R}_{N j}^{\prime}] - \Ex[\wt{R}_{N i}^{\prime}]\Ex[\wt{R}_{N j}^{\prime}]\\
		&= \Ex\left[\frac{1}{N^{2}}\sum_{i^{\prime}, j^{\prime} \in \Nset{N}}H(X_{i} - X_{i^{\prime}})H(X_{j} - X_{j^{\prime}})\right]\\
		&\hspace{5em}-
		\Ex\left[\frac{1}{N} \sum_{i^{\prime} \in \Nset{N}}H(X_{i} - X_{i^{\prime}})\right]
		\times \Ex\left[\frac{1}{N}\sum_{j^{\prime} \in \Nset{N}}H(X_{j} - X_{j^{\prime}})\right].
	\end{align*}
	We proceed to consider $\wt{R}_{N i}^{\prime}$ and $\wt{R}_{N j}^{\prime}$ for $i \neq j$ where $X_{i} \sim F_{k}, X_{j} \sim F_{\ell}$ and $X_{i^{\prime}} \sim F_{k^{\prime}}, X_{j^{\prime}} \sim F_{\ell^{\prime}}$. Namely, we pursue both within-block covariances and between-block covariances.
	Keeping these underlying distributions in mind, we analyze $\Ex[H(X_{i}-X_{i^{\prime}})H(X_{j}-X_{j^{\prime}})]$.
\begin{enumerate}
	\item \label{item-case-ii} For $(i^{\prime},j^{\prime})=(i,i)$,
	\begin{equation*}
		\Ex[H(X_{i}-X_{i^{\prime}})H(X_{j}-X_{j^{\prime}})]
		= \Ex[H(0)H(X_{j}-X_{i})]
		= \tfrac{1}{2}\Ex_{\ell}[F_{k}(X)].
	\end{equation*}
	\item \label{item-case-jj} For $(i^{\prime},j^{\prime})=(j,j)$,
	\begin{equation*}
		\Ex[H(X_{i}-X_{i^{\prime}})H(X_{j}-X_{j^{\prime}})]
		= \Ex[H(X_{i}-X_{j})H(0)]
		= \tfrac{1}{2}\Ex_{k}[F_{\ell}(X)].
	\end{equation*}	
	\item \label{item-case-ij} For $(i^{\prime},j^{\prime})=(i,j)$,
	\begin{equation*}
		\Ex[H(X_{i}-X_{i^{\prime}})H(X_{j}-X_{j^{\prime}})]
		= \Ex[H(0)H(0)]
		= \tfrac{1}{4}.
	\end{equation*}
	\item \label{item-case-ji} For $(i^{\prime},j^{\prime})=(j,i)$,
	\begin{equation*}
		\Ex[H(X_{i}-X_{i^{\prime}})H(X_{j}-X_{j^{\prime}})]
		= \Prob[\{X_{j} < X_{i}\}\cap\{X_{i} < X_{j}\}]
		= 0.
	\end{equation*}	
	\item \label{item-case-iPrime=i-partial} For $i^{\prime}=i$ and $j^{\prime} \notin \{i,j\}$, occurring $N-2$ times,
	\begin{equation*}
		\Ex[H(X_{i}-X_{i^{\prime}})H(X_{j}-X_{j^{\prime}})]
		= \Ex[H(0)H(X_{j}-X_{j^{\prime}})]
		= \tfrac{1}{2}\Ex_{\ell}[F_{\ell^{\prime}}(X)].
	\end{equation*}
	\item \label{item-case-iPrime=j-partial} For $i^{\prime}=j$ and $j^{\prime} \notin \{i,j\}$, occurring $N-2$ times,
	\begin{equation*}
		\Ex[H(X_{i}-X_{i^{\prime}})H(X_{j}-X_{j^{\prime}})]
		= \Prob[X_{j^{\prime}} < X_{j} < X_{i}]
		= \Ex_{k}[g_{(\ell^{\prime},\ell)}(X)].
	\end{equation*}
	\item \label{item-case-jPrime=j-partial} For $i^{\prime} \notin \{i,j\}$ and $j^{\prime}=j$, occurring $N-2$ times,
	\begin{equation*}
		\Ex[H(X_{i}-X_{i^{\prime}})H(X_{j}-X_{j^{\prime}})]
		= \Ex[H(X_{i}-X_{i^{\prime}})H(0)]
		= \tfrac{1}{2}\Ex_{k}[F_{k^{\prime}}(X)].
	\end{equation*}
	\item \label{item-case-jPrime=i-partial} For $i^{\prime} \notin \{i,j\}$ and $j^{\prime} = i$, occurring $N-2$ times,
	\begin{equation*}
		\Ex[H(X_{i}-X_{i^{\prime}})H(X_{j}-X_{j^{\prime}})]
		= \Prob[X_{i^{\prime}} < X_{i} < X_{j}] = \Ex_{\ell}[g_{(k^{\prime},k)}(X)].
	\end{equation*}
	\item \label{item-case-diag-partial} For $i^{\prime} = j^{\prime}$ with $i^{\prime} \notin \{i,j\}$, occurring $N-2$ times,
	\begin{align*}
		\Ex[H(X_{i}-X_{i^{\prime}})H(X_{j}-X_{j^{\prime}})]
		&= \Prob[\{X_{i^{\prime}} < X_{i}\}\cap\{X_{i^{\prime}} < X_{j}\}]\\
		&= \Prob[X_{i^{\prime}} < \min(X_{i},X_{j})]\\
		&= 1 - \Prob[\min(X_{i},X_{j}) \le X_{i^{\prime}}]\\
		&= \Ex_{k^{\prime}}[(1-F_{k}(X))(1-F_{\ell}(X))].
	\end{align*}
	\item \label{item-case-PartialCompleteSquare} In all remaining cases, the independence and distinctness of $i, j, i^{\prime}, j^{\prime}$ yields
	\begin{equation*}
		\Ex[H(X_{i}-X_{i^{\prime}})H(X_{j}-X_{j^{\prime}})]
		= \Ex[H(X_{i}-X_{i^{\prime}})]\Ex[H(X_{j}-X_{j^{\prime}})]
		= \Ex_{k}[F_{k^{\prime}}(X)] \Ex_{\ell}[F_{\ell^{\prime}}(X)].
	\end{equation*}		
\end{enumerate}
	Aggregating over $i^{\prime}, j^{\prime}$ in Case~\ref{item-case-PartialCompleteSquare} yields a partial collection of terms in the expansion of $\left(\sum_{k^{\prime} \in \Nset{\numSample}}N_{k^{\prime}}\Ex_{k}[F_{k^{\prime}}(X)]\right) \times \left(\sum_{\ell^{\prime} \in \Nset{\numSample}}N_{\ell^{\prime}}\Ex_{\ell}[F_{\ell^{\prime}}(X)]\right)$. We proceed to complete this expansion by adding and subtracting required terms, keeping the additional terms as remainders which themselves can be suitably completed and organized.

	First, in light of Case~\ref{item-case-diag-partial}, introducing
	\begin{equation}
		\pm\Big\{\Ex_{k}[(1-F_{k}(X))(1-F_{\ell}(X))] + \Ex_{\ell}[(1-F_{k}(X))(1-F_{\ell}(X))]\Big\}
	\end{equation}
	contributes a (subtracted) residual term and (via addition) completes the expansion
	\begin{equation}
		\sum_{m^{\prime} \in \Nset{\numSample}}N_{m^{\prime}}\Ex_{m^{\prime}}[(1-F_{k}(X))(1-F_{\ell}(X))].
	\end{equation}
	Furthermore, introducing
	\begin{equation}
		\pm \Bigg\{ \sum_{m^{\prime} \in \Nset{\numSample}} N_{m^{\prime}} \Ex_{k}[F_{m^{\prime}}(X)]\Ex_{\ell}[F_{m^{\prime}}(X)] \Bigg\}
	\end{equation}
	contributes a (subtracted) residual term and (via addition) completes the diagonal portion of the complete-the-square problem.

	Next, in light of Case~\ref{item-case-iPrime=j-partial}, introducing
	\begin{equation}
		\pm\Big\{\Ex_{k}[g_{(k,\ell)}(X)] + \Ex_{k}[g_{(\ell,\ell)}(X)]\Big\}
	\end{equation}
	contributes a (subtracted) residual term and (via addition) completes the expansion
	\begin{equation}
		\sum_{m^{\prime} \in \Nset{\numSample}} N_{m^{\prime}} \Ex_{k}[g_{(m^{\prime},\ell)}(X)].
	\end{equation}
	Furthermore, introducing
	\begin{equation}
		\pm\Bigg\{ \sum_{m^{\prime} \in \Nset{\numSample}} N_{m^{\prime}} \Ex_{k}[F_{\ell}(X)]\Ex_{\ell}[F_{m^{\prime}}(X)] - \tfrac{1}{2}\Ex_{k}[F_{\ell}(X)] \Bigg\}
	\end{equation}
	contributes a (subtracted) residual term and (via) addition completes the corresponding partial row of the complete-the-square problem.

	Next, in light of Case~\ref{item-case-jPrime=i-partial}, introducing
	\begin{equation}
		\pm\Big\{ \Ex_{\ell}[g_{(k,k)}(X)] + \Ex_{\ell}[g_{(\ell,k)(X)}] \Big\}
	\end{equation}
	contributes a (subtracted) residual term and (via addition) completes the expansion
	\begin{equation}
		\sum_{m^{\prime} \in \Nset{\numSample}} N_{m^{\prime}} \Ex_{\ell}[g_{(m^{\prime},k)}(X)].
	\end{equation}
	Furthermore, introducing
	\begin{equation}
		\pm\Bigg\{ \sum_{m^{\prime} \in \Nset{\numSample}} N_{m^{\prime}} \Ex_{k}[F_{m^{\prime}}(X)]\Ex_{\ell}[F_{k}(X)] - \tfrac{1}{2}\Ex_{\ell}[F_{k}(X)] \Bigg\}
	\end{equation}
	contributes a (subtracted) residual term and (via) addition completes the corresponding partial row of the complete-the-square problem.

	Putting all the pieces together yields the completed square, namely all the terms in the expansion of 
	\begin{equation}
		\left(\sum_{k^{\prime} \in \Nset{\numSample}}N_{k^{\prime}}\Ex_{k}[F_{k^{\prime}}(X)]\right)
		\times
		\left(\sum_{\ell^{\prime} \in \Nset{\numSample}}N_{\ell^{\prime}}\Ex_{\ell}[F_{\ell^{\prime}}(X)]\right).
	\end{equation}
	Additionally, the remaining residual can be organized into two expressions. The first, dominant expression is given by
\begin{align*}
	&\sum_{m^{\prime} \in \Nset{\numSample}}N_{m^{\prime}}
	\Bigg[
	\Ex_{m^{\prime}}[(1-F_{k}(X))(1-F_{\ell}(X))]
	+ \Ex_{k}[g_{(m^{\prime},\ell)}(X)]
	+ \Ex_{\ell}[g_{(m^{\prime},k)}(X)]\\
	&\hspace{6em}- \Ex_{k}[F_{m^{\prime}}(X)]\Ex_{\ell}[F_{m^{\prime}}(X)]
	- \Ex_{k}[F_{m^{\prime}}(X)]\Ex_{\ell}[F_{k}(X)]
	- \Ex_{k}[F_{\ell}(X)]\Ex_{\ell}[F_{m^{\prime}}(X)]
	\Bigg].
\end{align*} 
The second expression is given by
\begin{align*}
	&\Big\{
	\tfrac{1}{2}\Ex_{\ell}[F_{k}(X)]
	+ \tfrac{1}{2}\Ex_{k}[F_{\ell}(X)]
	+ \tfrac{1}{4}
	+ 0
	\Big\}\\
	&\hspace{4em}-
	\Big\{
	\Ex_{k}[(1-F_{k}(X))(1-F_{\ell}(X))]
	+ \Ex_{\ell}[(1-F_{k}(X))(1-F_{\ell}(X))]
	\Big\}\\
	&\hspace{4em}-
	\Big\{
	\Ex_{k}[g_{(k,\ell)}(X)]
	+ \Ex_{k}[g_{(\ell,\ell)}(X)]
	\Big\}\\
	&\hspace{4em}-
	\Big\{
	\Ex_{\ell}[g_{(k,k)}(X)]
	+ \Ex_{\ell}[g_{(\ell,k)}(X)]
	\Big\}\\
	&\hspace{4em}+
	\Big\{
	\tfrac{1}{2}\Ex_{k}[F_{\ell}(X)]
	+ \tfrac{1}{2}\Ex_{\ell}[F_{k}(X)]
	\Big\}\\
	&\hspace{4em}+ \Big\{\Ex_{k}[F_{\ell}(X)]\Ex_{\ell}[F_{k}(X)] - \tfrac{1}{4}\Big\}
\end{align*}
which can be written in the more interpretable form
\begin{align*}
	&\Bigg\{ \Ex_{k}\left[2F_{\ell}(X)-F_{k}(X)F_{\ell}(X)-g_{(k,\ell)}(X) - g_{(\ell,\ell)}(X) - \tfrac{1}{3}\right]\Bigg\}\\
	&\hspace{2em}+
	\Bigg\{\Ex_{\ell}\left[2F_{k}(X)-F_{k}(X)F_{\ell}(X)-g_{(k,k)}(X) - g_{(\ell,k)}(X) - \tfrac{1}{3}\right]\Bigg\}\\
	&\hspace{4em}+ \Big\{\Ex_{k}[F_{\ell}(X)]\Ex_{\ell}[F_{k}(X)] - \tfrac{1}{4}\Big\}\\
	&\hspace{6em}-
	\Bigg\{\frac{1}{12}\Bigg\}.
\end{align*}
\cref{prop:ranks_multi_covariance} follows by scaling and re-indexing the above expressions. 
\end{proof}

\subsection{Proofs: a combinatorial trace bound}
\label{app:trace_bound_proof}
\newcommand{\Diff}{\Gamma}
\newcommand{\diff}{\gamma}
\newcommand{\nrke}{\widetilde{r}}
\newcommand{\set}{\mathcal{S}}
\newcommand{\grp}{\mathcal{G}}
\newcommand{\fset}{\Psi}
\newcommand{\prob}{{\rm Pr}}

\newcommand{\bra}[1]{\left( #1\right)}
\newcommand{\sqbra}[1]{\left[ #1\right]}
\newcommand{\crbra}[1]{\left\{ #1\right\}}
\newcommand{\ABS}[1]{\left\vert #1\right\vert}

\begin{proof}[Proof of \cref{lem:os_trace}]
	Define the matrix $\Diff = \nrk{A} - \Ex[\nrk{A}]$ where $\Diff \equiv (\diff_{ij})$, hence $\diff_{ij} = \nrke_{ij} - \Ex[\nrke_{ij}] = \frac{r_{ij}}{N+1} - \frac{1}{2}$ for $i \neq j$. We focus on
	\begin{equation}
		\label{equ::trace}
		\trace\left({\Diff^{4}}\right)
		=
		\sum_{1 \leq i_{1}, i_{2}, i_{3}, i_{4} \leq n} \diff_{i_{1} i_{2}}\diff_{i_{2} i_{3}}\diff_{i_{3} i_{4}}\diff_{i_{4} i_{1}}.
	\end{equation}
	Here, $\Diff$ is a random symmetric matrix with zero-valued diagonal entries and is thus determined by $N = n(n-1)/2$ random variables. For different choices of indices $i_{1}, i_{2}, i_{3}, i_{4} \in \Nset{n}$, the tuple $(\diff_{i_{1} i_{2}}, \diff_{i_{2} i_{3}}, \diff_{i_{3} i_{4}}, \diff_{i_{4} i_{1}})$ may be entrywise comprised of either four, two, or a single unique random variable(s). We shall proceed to characterize all possible cases of $(\diff_{i_{1} i_{2}}, \diff_{i_{2} i_{3}}, \diff_{i_{3} i_{4}}, \diff_{i_{4} i_{1}})$. In particular, we only need to consider situations when $i_{1} \neq i_{2}$, $i_{2} \neq i_{3}$, $i_{3} \neq i_{4}$, and $i_{4} \neq i_{1}$, since otherwise $\diff_{i_{1} i_{2}}\diff_{i_{2} i_{3}}\diff_{i_{3} i_{4}}\diff_{i_{4} i_{1}} = 0$ almost surely.
	
	In what follows, it will be convenient to use the following new definition and notation.
	\begin{definition}
		\label[definition]{def:abcd_vec_unif_distrib}
		Let $(a,b,c,d)$ denote a random vector generated by uniformly selecting a sequence of four distinct elements from the set of shifted, normalized rank statistic values
		\begin{equation*}
			\left\{\frac{i}{N+1} - \frac{1}{2}
			\quad
			:
			\quad
			i \in \Nset{N}\right\}.
		\end{equation*}
	\end{definition}

	\noindent{\bf Case 1: $i_1$, $i_2$, $i_3$, $i_4$ are pairwise different.} In this case, the distribution of the tuple $(\diff_{i_1 i_2}, \diff_{i_2 i_3}, \diff_{i_3 i_4}, \diff_{i_4 i_1})$ is the same as that of $(a,b,c,d)$. By elementary counting, there are $n(n-1)(n-2)(n-3)$ such tuples of indices $i_1, i_2, i_3, i_4$.

	\noindent{\bf Case 2: $i_1\neq i_3$, $i_2=i_4$.} In this case, the distribution of the tuple $(\diff_{i_1 i_2}, \diff_{i_2 i_3}, \diff_{i_3 i_4}, \diff_{i_4 i_1})$ is the same as that of $(a,b,b,a)$. By elementary counting, there are $n(n-1)(n-2)$ such tuples of indices $i_1, i_2, i_3, i_4$.

	\noindent{\bf Case 3: $i_1 = i_3$, $i_2 \neq i_4$.} In this case, the distribution of the tuple $(\diff_{i_1 i_2}, \diff_{i_2 i_3}, \diff_{i_3 i_4}, \diff_{i_4 i_1})$ is the same as that of $(a,a,b,b)$. By elementary counting, there are $n(n-1)(n-2)$ such tuples of indices $i_1, i_2, i_3, i_4$.

	\noindent{\bf Case 4: $i_1 = i_3$, $i_2 = i_4$.} In this case, the distribution of the tuple $(\diff_{i_1 i_2}, \diff_{i_2 i_3}, \diff_{i_3 i_4}, \diff_{i_4 i_1})$ is the same as that of $(a,a,a,a)$. By elementary counting, there are $n(n-1)$ such sequences of indices $i_1, i_2, i_3, i_4$.

	By considering all possible configurations of $i_1, i_2, i_3, i_4$ in \cref{equ::trace}, we obtain
	\begin{equation}
		\label{equ::trace_terms}
		\begin{split}
		&\Ex\left[\trace({\Diff^4})\right]\\
		&\hspace{1em}=
		\sum_{1 \leq i_1,i_2,i_3,i_4 \leq n} \Ex[\diff_{i_1 i_2}\diff_{i_2 i_3}\diff_{i_3 i_4}\diff_{i_4 i_1}]\\
		&\hspace{1em}=
		n(n-1)(n-2)(n-3)\Ex[abcd] + 2n(n-1)(n-2)\Ex [a^{2}b^{2}] + n(n-1)\Ex[a^{4}],
		\end{split}
	\end{equation}
	where expectations $\Ex[\cdot]$ are taken with respect to the uniform distribution over $(a,b,c,d)$ per \cref{def:abcd_vec_unif_distrib}.
	
	Next, we explicitly compute $\Ex[abcd]$, $\Ex [a^{2}b^{2}]$, and $\Ex[a^{4}]$. First,
	
	{\small
	\begin{equation}
		\label{equ::expro}
	\begin{split}
		&\Ex[abcd]\\
		&= \tfrac{1}{N(N-1)(N-2)(N-3)}\sum_{i\in\Nset{N}} \left(\tfrac{i}{N+1}-\tfrac{1}{2}\right)\sum_{j\in\Nset{N}-\{i\}}\left(\tfrac{j}{N+1}-\tfrac{1}{2}\right)\sum_{k\in\Nset{N}-\{i,j\}}\left(\tfrac{k}{N+1}-\tfrac{1}{2}\right)\\
			&\hspace{2em} \sum_{l\in\Nset{N}-\{i,j,k\}}\left(\tfrac{l}{N+1}-\tfrac{1}{2}\right)\\
		&= \tfrac{1}{N(N-1)(N-2)(N-3)}\sum_{i\in\Nset{N}} \left(\tfrac{i}{N+1}-\tfrac{1}{2}\right)\sum_{j\in\Nset{N}-\{i\}}\left(\tfrac{j}{N+1}-\tfrac{1}{2}\right)\sum_{k\in\Nset{N}-\{i,j\}}\left(\tfrac{k}{N+1}-\tfrac{1}{2}\right)\\
			&\hspace{2em} \left[N\Ex[a]-\left(\tfrac{i}{N+1}-\tfrac{1}{2}\right)-\left(\tfrac{j}{N+1}-\tfrac{1}{2}\right)-\left(\tfrac{k}{N+1}-\tfrac{1}{2}\right)\right]\\
		&= \tfrac{1}{N(N-1)(N-2)(N-3)}\sum_{i\in\Nset{N}} \left(\tfrac{i}{N+1}-\tfrac{1}{2}\right)\sum_{j\in\Nset{N}-\{i\}}\left(\tfrac{j}{N+1}-\tfrac{1}{2}\right)\\
			&\hspace{2em} \left\{\left[N\Ex[a]-\left(\tfrac{i}{N+1}-\tfrac{1}{2}\right)-\left(\tfrac{j}{N+1}-\tfrac{1}{2}\right)\right]^2-N\Ex[a^2]+\left(\tfrac{i}{N+1}-\tfrac{1}{2}\right)^2+\left(\tfrac{j}{N+1}-\tfrac{1}{2}\right)^2\right\}\\
		&= \tfrac{1}{N(N-1)(N-2)(N-3)}\sum_{i\in\Nset{N}} \left(\tfrac{i}{N+1}-\tfrac{1}{2}\right)\sum_{j\in\Nset{N}-\{i\}}\left(\tfrac{j}{N+1}-\tfrac{1}{2}\right)\\
			&\hspace{2em} \left\{2\left(\tfrac{j}{N+1}-\tfrac{1}{2}\right)^2- 2\left[N\Ex[a]-\left(\tfrac{i}{N+1}-\tfrac{1}{2}\right)\right]\left(\tfrac{j}{N+1}-\tfrac{1}{2}\right)-N\Ex[a^2]+\left(\tfrac{i}{N+1}-\tfrac{1}{2}\right)^2\right.\\
			&\hspace{3em}\left.+\left[N\Ex[a]-\left(\tfrac{i}{N+1}-\tfrac{1}{2}\right)\right]^2\right\}\\
		&= \tfrac{1}{N(N-1)(N-2)(N-3)}\sum_{i\in\Nset{N}} \left(\tfrac{i}{N+1}-\tfrac{1}{2}\right)\left\{2N\Ex[a^3]-2\left(\tfrac{i}{N+1}-\tfrac{1}{2}\right)^3\right.\\
			&\hspace{2em}\left.-3\left[N\Ex[a]-\left(\tfrac{i}{N+1}-\tfrac{1}{2}\right)\right]\left[N\Ex[a^2]-\left(\tfrac{i}{N+1}-\tfrac{1}{2}\right)^2\right]+\left[N\Ex[a]-\left(\tfrac{i}{N+1}-\tfrac{1}{2}\right)\right]^3\right\}\\
		&= \tfrac{1}{N(N-1)(N-2)(N-3)}\sum_{i\in\Nset{N}} \left(\tfrac{i}{N+1}-\tfrac{1}{2}\right)\left\{-6\left(\tfrac{i}{N+1}-\tfrac{1}{2}\right)^3+6N\Ex[a]\left(\tfrac{i}{N+1}-\tfrac{1}{2}\right)^2\right.\\
			&\hspace{2em}\left.+3\left[N\Ex[a^2]-N^2(\Ex[a])^2\right]\left(\tfrac{i}{N+1}-\tfrac{1}{2}\right)+N^3(\Ex[a])^3+2N\Ex[a^3]-3N^2\Ex[a^2]\Ex[a]\right\}\\
		&= \cfrac{8N^2\Ex[a]\Ex[a^3]+N^4(\Ex[a])^4-6N^3\Ex[a^2](\Ex[a])^2-6N\Ex[a^4]+3N^2(\Ex[a^2])^2}{N(N-1)(N-2)(N-3)}.
	\end{split}
	\end{equation}
	}
	Since $\Ex[a]=0$, the above expression simplifies to the form
	\begin{equation*}
		\Ex[abcd]
		=
		\frac{3N^2(\Ex[a^2])^2-6N\Ex[a^4]}{N(N-1)(N-2)(N-3)}.
	\end{equation*}
	Next, again by direct computation, we evaluate $\Ex[a^{2}b^{2}]$ in the manner
	\begin{equation*}
	\begin{split}
		\Ex[a^2b^2] &= \frac{1}{N(N-1)}\sum_{i\in\Nset{N}}\left(\frac{i}{N+1}-\frac{1}{2}\right)^2\sum_{j\in\Nset{N}-\{i\}}\left(\frac{j}{N+1}-\frac{1}{2}\right)^2\\
		&= \frac{1}{N(N-1)}\sum_{i\in\Nset{N}}\left(\frac{i}{N+1}-\frac{1}{2}\right)^2\left\{N\Ex[a^2]-\left(\frac{i}{N+1}-\frac{1}{2}\right)^2\right\}\\
		&= \frac{N(\Ex[a^2])^2-\Ex[a^4]}{N-1}.
	\end{split}
	\end{equation*}
	Thus, from \cref{equ::trace_terms},
	\begin{equation*}
	\begin{split}
		\Ex\left[\trace({\Diff^{4}})\right]
		&= (\Ex[a^2])^2\left\{\frac{3Nn(n-1)(n-2)(n-3)}{(N-1)(N-2)(N-3)}+\frac{2Nn(n-1)(n-2)}{N-1}\right\}\\
				&\hspace{2em}- \Ex[a^4]\left\{\frac{6n(n-1)(n-2)(n-3)}{(N-1)(N-2)(N-3)}+\frac{2n(n-1)(n-2)}{N-1}-n(n-1)\right\}\\
		&= (\Ex[a^2])^2\left(\frac{2n^2(n-1)^2}{n+1}\right)\left\{\frac{3}{(N-2)(n+2)}+1\right\}\\
				&\hspace{2em}+ \Ex[a^4]\left(\frac{n(n-1)(n-3)}{n+1}\right)\left\{1-\frac{24}{(N-2)(n+2)(n-3)}\right\}.
	\end{split}
	\end{equation*}
	Recall that $N = n(n-1)/2$, hence by direct computation,
	\begin{align*}
		\Ex[a^{2}]
		&=
		\frac{1}{N}\sum_{i \in \Nset{N}}\left(\frac{i}{N+1}-\frac{1}{2}\right)^{2} = \frac{N-1}{12(N+1)} \le \frac{1}{12},\\
		\Ex[a^{4}]
		&=
		\frac{1}{N}\sum_{i \in \Nset{N}}\left(\frac{i}{N+1}-\frac{1}{2}\right)^4
		=
		\frac{3N^{3}-3N^{2}-7N+7}{240(N+1)^{3}}
		\le
		\frac{3}{240}
		=
		\frac{1}{80}.
	\end{align*}
	The hypothesis $n \ge 4$ implies $n^{2} \le \frac{1}{4}n^{3}$, whence plugging in the above observations yields
		\begin{align*}
			\Ex\left[\trace({\Diff^{4}})\right]
			&\le
			\left(\frac{1}{12^{2}} \cdot 2 \cdot \frac{9}{8}\right)n^{3} + \left(\frac{1}{80}\right) n^{2}
			\le
			\frac{1}{50}n^{3}.
		\end{align*}
	This concludes the proof of \cref{lem:os_trace}.
\end{proof}

\begin{proof}[Proof of \cref{lem:ms_trace}]
	Define the matrix $\Diff = \nrk{A} - \Ex[\nrk{A}]$ where $\Diff \equiv (\diff_{ij})$. Thus, $\diff_{ij} = \nrke_{ij} - \Ex[\nrke_{ij}] = \frac{r_{ij}}{N+1} - \rankspec_{g_i g_j}$ for $i \neq j$. We focus on
	\begin{equation}
		\label{equ::trace2}
		\trace\left({\Diff^{4}}\right)
		=
		\sum_{1 \leq i_1, i_2, i_3, i_4 \leq n} \diff_{i_1 i_2}\diff_{i_2 i_3}\diff_{i_3 i_4}\diff_{i_4 i_1}.
	\end{equation}
	In this proof, we shall mirror the case-by-case discussion in the proof of \cref{lem:os_trace}. We first observe that there are at most $2n^{3}$ terms in the summation \cref{equ::trace2} that satisfy the condition $i_1 \neq i_2, i_2 \neq i_3, i_3 \neq i_4$ and $i_4 \neq i_1$ and with $(i_1,i_2,i_3,i_4)$ having fewer than four distinct entries. Note that $\vert\diff_{ij}\vert \leq 1$ for all $i, j \in \Nset{n}$, which implies that $\vert\diff_{i_1 i_2}\diff_{i_2 i_3}\diff_{i_3 i_4}\diff_{i_4 i_1}\vert\leq 1$ for any choice of indices $i_1,i_2,i_3,i_4 \in \Nset{n}$. Hence,
	\begin{equation}
		\label{inequ::xy1}
		\left| \sum_{1\leq i_1,i_2,i_3,i_4\leq n} \diff_{i_1 i_2}\diff_{i_2 i_3}\diff_{i_3 i_4}\diff_{i_4 i_1}-\sum_{
			\begin{array}{c}
				i_1,i_2,i_3,i_4\in\Nset{n}\\
				\text{pairwise different}
			\end{array}}
		\diff_{i_1 i_2}\diff_{i_2 i_3}\diff_{i_3 i_4}\diff_{i_4 i_1}\right| \leq 2n^{3}.
	\end{equation} 
	To bound the second summation term in \cref{inequ::xy1}, we shall characterize the distribution of $\diff_{ij} = \nrke_{ij}-\Ex[\nrke_{ij}]$ by using the population-level block structure governing $A$ and $\nrk{A}$.

	For any choices $k_{1}, k_{2} \in \Nset{K}$, define an associated collection of random variables $\set_{k_1 k_2} = \{\nrke_{ij} : i \in \grp_{k_1}, j \in \grp_{k_2}, i < j\}$. By the definition of $\nrke_{ij}$, the sets $\set_{k_1k_2}$ for $1 \leq k_1 \leq k_2 \leq K$ are pairwise disjoint and collectively satisfy
	\begin{equation*}
		\bigcup_{1 \leq k_1 \leq k_2 \leq K} \set_{k_1 k_2}
		=
		\left\{\frac{i}{N+1} : i \in \Nset{N} \right\}.
	\end{equation*}
	Next, in order to facilitate counting, for any $k_1, k_2 \in \Nset{K}$, define
	\begin{equation*}
		N(k_1,k_2) =
		\begin{cases}
			n_{k_1}n_{k_2}
			&\text{ if } k_1\neq k_2,\\
			\cfrac{n_{k_1}(n_{k_1}-1)}{2}
			&\text{ if } k_1 = k_2.
		\end{cases}
	\end{equation*}
	We emphasize that $N(k_1,k_2)$ is the cardinality of the set $\set_{k_1k_2}$. Further below, we use the notation $\operatorname{card}(\cdot) \equiv |\cdot|$ to indicate the cardinality of an underlying set.

	For any partition of the set $\{i/(N+1) : i \in \Nset{N}\}$ that satisfies $\vert\fset_{k_1k_2}\vert=N(k_1,k_2)$, denoted by $\{\fset_{k_1 k_2} : 1 \leq k_1 \leq k_2 \leq K\}$, condition on the event that $\set_{k_1 k_2} = \fset_{k_1 k_2}$ for all $1 \leq k_1 \leq k_2 \leq K$. Then, for any choice of $k_1, k_2$, the joint conditional distribution of the random variables in $\set_{k_1 k_2}$ is simply the distribution of uniformly sampling elements in $\fset_{k_1 k_2}$ without replacement. Moreover, by conditioning on the event that $\set_{k_1 k_2} = \fset_{k_1 k_2}$ for all $1 \leq k_1 \leq k_2 \leq K$, the conditional distribution is independent between blocks. This conditional within-block uniform distribution property is due to the fact that the entries of $A$ are i.i.d.~within blocks and are independent across blocks. These observations will enable us to analyze the summation in \cref{equ::trace2}.

	Next, we shall characterize several moments of the within-block uniform conditional distribution. Let $(a,b,c,d)$ be a random vector defined as selecting a four-element subsequence from $\fset_{k_1 k_2}$ uniformly at random, for an arbitrary, finite set satisfying $\vert\fset_{k_1k_2}\vert = N(k_1,k_2) \geq 4$. Write $\Ex_{\fset_{k_1k_2}}[\cdot]$ to denote taking an expectation with respect to this underlying distribution. Consequently,
	\begin{equation*}
		\Ex_{\fset_{k_1k_2}}\sqbra{a - \rankspec_{k_1k_2}}
		=
		\overline{\fset}_{k_1k_2} - \rankspec_{k_1k_2},
	\end{equation*}
	where $\overline{\fset}_{k_1k_2}$ denotes the average of all elements in $\fset_{k_1k_2}$. We compute
	\begin{equation}\label{equ::xy3}
		\begin{split}
			&\Ex_{\fset_{k_1k_2}}\sqbra{(a-\rankspec_{k_1k_2})(b-\rankspec_{k_1k_2})}\\
			&\hspace{1em}= \cfrac{\sum_{x\in \fset_{k_1k_2}}(x-\rankspec_{k_1k_2})\left\{\sum_{y\in\fset_{k_1k_2}-\{x\}}(y-\rankspec_{k_1k_2})\right\}}{N(k_1,k_2)[N(k_1,k_2)-1]}\\
			&\hspace{1em}= \cfrac{\sum_{x\in \fset_{k_1k_2}}(x-\rankspec_{k_1k_2})\left\{N(k_1,k_2)\Ex_{\fset_{k_1k_2}}[a-\rankspec_{k_1k_2}]-(x-\rankspec_{k_1k_2})\right\}}{N(k_1,k_2)[N(k_1,k_2)-1]}\\
			&\hspace{1em}= \cfrac{N(k_1,k_2)\left(\Ex_{\fset_{k_1k_2}}[a-\rankspec_{k_1k_2}]\right)^2-\Ex_{\fset_{k_1k_2}}[(a-\rankspec_{k_1k_2})^2]}{N(k_1,k_2)-1}.
		\end{split}
	\end{equation}
	The same, direct approach yields the expansion
	\begin{equation}\label{equ::xy4}
		\begin{split}
		&\Ex_{\fset_{k_1k_2}}\sqbra{(a-\rankspec_{k_1k_2})^2(b-\rankspec_{k_1k_2})(c-\rankspec_{k_1k_2})}\\
		&\hspace{1em}=\cfrac{\sum_{x\in \fset_{k_1k_2}}(x-\rankspec_{k_1k_2})^2\sum_{y\in\fset_{k_1k_2}-\{x\}}(y-\rankspec_{k_1k_2})\sum_{z\in\fset_{k_1k_2}-\{x,y\}}(z-\rankspec_{k_1k_2})}{N(k_1,k_2)[N(k_1,k_2)-1][N(k_1,k_2)-2]}\\
		&\hspace{1em}=\frac{1}{N(k_1,k_2)\sqbra{N(k_1,k_2)-1}\sqbra{N(k_1,k_2)-2}}\\
		&\hspace{2em}\times \left\{N(k_1,k_2)^3\bra{\Ex_{\fset_{k_1k_2}}[a-\rankspec_{k_1k_2}]}^2\Ex_{\fset_{k_1k_2}}[(a-\rankspec_{k_1k_2})^2]\right.\\
		&\hspace{3em}-N(k_1,k_2)^2\bra{\Ex_{\fset_{k_1k_2}}[(a-\rankspec_{k_1k_2})^2]}^2+2N(k_1,k_2) \Ex_{\fset_{k_1k_2}}[(a-\rankspec_{k_1k_2})^4]\\
		&\hspace{3em}\left.-2N(k_1,k_2)^2\Ex_{\fset_{k_1k_2}}[a-\rankspec_{k_1k_2}]\Ex_{\fset_{k_1k_2}}[(a-\rankspec_{k_1k_2})^3]\right\}
		\end{split}
	\end{equation}
	and
	\begin{equation}
		\label{equ::xy1}
		\begin{split}
			&\Ex_{\fset_{k_1k_2}}\sqbra{(a-\rankspec_{k_1k_2})(b-\rankspec_{k_1k_2})(c-\rankspec_{k_1k_2})(d-\rankspec_{k_1k_2})}\\
			&\hspace{1em}=\frac{1}{N(k_1,k_2)\sqbra{N(k_1,k_2)-1}\sqbra{N(k_1,k_2)-2}\sqbra{N(k_1,k_2)-3}}\\
				&\hspace{2em}\times \Big\{ 8N(k_1,k_2)^2\Ex_{\fset_{k_1k_2}}[a-\rankspec_{k_1k_2}]\Ex_{\fset_{k_1k_2}}[(a-\rankspec_{k_1k_2})^3]\\
				&\hspace{3em}+ N(k_1,k_2)^4\bra{\Ex_{\fset_{k_1k_2}}[a-\rankspec_{k_1k_2}]}^4\\
				&\hspace{3em}- 6N(k_1,k_2)^3\Ex_{\fset_{k_1k_2}}[(a-\rankspec_{k_1k_2})^2]\bra{\Ex_{\fset_{k_1k_2}}[a-\rankspec_{k_1k_2}]}^2\\
				&\hspace{3em}- 6N(k_1,k_2)\Ex_{\fset_{k_1k_2}}[(a-\rankspec_{k_1k_2})^4]\\
				&\hspace{3em}+ 3N(k_1,k_2)^2\bra{\Ex_{\fset_{k_1k_2}}[(a-\rankspec_{k_1k_2})^2]}^2\Big\}\\
			&\hspace{1em}=\frac{N(k_1,k_2)^4\bra{\Ex_{\fset_{k_1k_2}}[a-\rankspec_{k_1k_2}]}^4}{N(k_1,k_2)\sqbra{N(k_1,k_2)-1}\sqbra{N(k_1,k_2)-2}\sqbra{N(k_1,k_2)-3}}+\eta(\fset_{k_1k_2},\rankspec_{k_1k_2}),
		\end{split}
	\end{equation}
	where $\eta(\cdot,\cdot)$ is a constant depending on $\fset_{k_1k_2}$ and $\rankspec_{k_1k_2}$ that is bounded in the manner
	\begin{equation}
		\label{inequ::xy2}
		\left\vert \eta(\fset_{k_1k_2},\rankspec_{k_1k_2})\right\vert
		\leq \frac{6\sqbra{N(k_1,k_2)+1}^2}{\sqbra{N(k_1,k_2)-1}\sqbra{N(k_1,k_2)-2}\sqbra{N(k_1,k_2)-3}}.\\
	\end{equation}
	The first expression in \cref{equ::xy1} is obtained from deductions that mirror those in \cref{equ::expro}. \cref{inequ::xy2} holds since $\vert a-\rankspec_{k_1k_2}\vert\leq 1$ almost surely.

	Below, we bound the second summation term in \cref{inequ::xy1} by enumerating four different cases for $g_{i_1},g_{i_2},g_{i_3},g_{i_4}$ when $i_1,i_2,i_3,i_4$ are pairwise different.

	\noindent\textbf{Case~1: $k_1=k_3$, $k_2=k_4$ and $g_{i_1}=k_1$, $g_{i_2}=k_2$, $g_{i_3}=k_3$, $g_{i_4}=k_4$.} In this case, the variables $\diff_{i_1 i_2},\diff_{i_2 i_3},\diff_{i_3 i_4},\diff_{i_4 i_1}$ are all entries in the block $\Diff_{\grp_{k_1}\grp_{k_2}}$. Thus,
	\begin{equation}
		\label{equ::xy2}
		\begin{split}
			&\Ex\left[\diff_{i_1 i_2}\diff_{i_2 i_3}\diff_{i_3 i_4}\diff_{i_4 i_1}\right]\\
			&=
			\sum_{\fset_{k_1k_2}\subset\crbra{\frac{i}{N+1} : i\in\Nset{N}}}\prob[\set_{k_1k_2}=\fset_{k_1k_2}]\\
			&\hspace{5em}\times \Ex_{\fset_{k_1k_2}}[(a-\rankspec_{k_1k_2})(b-\rankspec_{k_1k_2})(c-\rankspec_{k_1k_2})(d-\rankspec_{k_1k_2})]\\
			&= \sum_{\fset_{k_1k_2}\subset\crbra{\frac{i}{N+1} : i\in\Nset{N}}}\frac{\prob[\set_{k_1k_2}=\fset_{k_1k_2}]N(k_1,k_2)^3(\Ex_{\fset_{k_1k_2}}[a-\rankspec_{k_1k_2}])^4}{\sqbra{N(k_1,k_2)-1}\sqbra{N(k_1,k_2)-2}\sqbra{N(k_1,k_2)-3}}\\
			&\hspace{5em}+ \prob[\set_{k_1k_2}=\fset_{k_1k_2}]\eta(\fset_{k_1k_2},\rankspec_{k_1k_2})\\
			&= \sum_{\fset_{k_1k_2}\subset\crbra{\frac{i}{N+1} : i\in\Nset{N}}}\frac{\prob[\set_{k_1k_2}=\fset_{k_1k_2}]N(k_1,k_2)^3(\overline{\fset}_{k_1k_2}-\rankspec_{k_1k_2})^4}{\sqbra{N(k_1,k_2)-1}\sqbra{N(k_1,k_2)-2}\sqbra{N(k_1,k_2)-3}}\\
			&\hspace{5em}+ \Ex\left[\eta(\set_{k_1k_2},\rankspec_{k_1k_2})\right]\\
			&= \frac{N(k_1,k_2)^3\Ex\sqbra{(\overline{\set}_{k_1k_2}-\rankspec_{k_1k_2})^4}}{\sqbra{N(k_1,k_2)-1}\sqbra{N(k_1,k_2)-2}\sqbra{N(k_1,k_2)-3}} + \Ex\left[\eta(\set_{k_1k_2},\rankspec_{k_1k_2})\right],
		\end{split}
	\end{equation}
	where the random variable $\overline{\set}_{k_1k_2}$ denotes the average of all random variables in ${\set}_{k_1k_2}$. In order to bound the above expression further, we note that
	\begin{equation}
		\label{inequ::xy3}
		0
		\leq 
		\Ex\sqbra{(\overline{\set}_{k_1k_2}-\rankspec_{k_1k_2})^4} \leq \Ex\sqbra{(\overline{\set}_{k_1k_2} - \rankspec_{k_1k_2})^2}
		=
		\Var\sqbra{\overline{\set}_{k_1k_2}},
	\end{equation}
	where
	\begin{equation*}
		\Var\sqbra{\overline{\set}_{k_1k_2}}
		=
		\cfrac{1}{N(k_1,k_2)^2} \times
		\begin{cases}
			\Var\sqbra{\sum_{i , j \in \grp_{k_1}, i < j}\nrke_{ij}}
				&\text{ if } k_1 = k_2,\\ \\
			\Var\sqbra{\sum_{i \in \grp_{k_1}, j \in \grp_{k_2}}\nrke_{ij}}
				&\text{ if } k_1 \neq k_2.
		\end{cases}
	\end{equation*}
	The above variance can be expanded as the sum of covariance terms, in the manner
	\begin{equation*}
		\Var\sqbra{\overline{\set}_{k_1k_2}}
		=\cfrac{1}{N(k_1,k_2)^2} \times
		\begin{cases}
			\sum_{i,j\in\grp_{k_1},i<j}\sum_{i',j'\in\grp_{k_1},i'<j'}\Cov\sqbra{\nrke_{ij}\nrke_{i'j'}}
				&\text{ if } k_1 = k_2,\\ \\
			\sum_{i\in\grp_{k_1},j\in\grp_{k_2}}\sum_{i'\in\grp_{k_1},j'\in\grp_{k_2}}\Cov\sqbra{\nrke_{ij}\nrke_{i'j'}}
				&\text{ if } k_1 \neq k_2.
			\end{cases}
	\end{equation*}
	Since $\nrke_{ij} \in (0,1]$,
	\begin{equation*}
		\Var[\nrke_{ij}]
		\leq
		1
		\quad\quad
		\text{ for all } i,j \in \Nset{n}.
	\end{equation*}
	Likewise, as a consequence of \cref{prop:ranks_multi_covariance}, it holds that
	\begin{equation*}
		|\Cov[\nrke_{ij},\nrke_{i'j'}]|
		\le
		\frac{8}{N},
		\quad\quad
		\text{ for all } (i,j) \neq (i^{\prime},j^{\prime}).
	\end{equation*}
	Thus,
	\begin{equation}\label{inequ::var}
			\begin{split}
			\Var\sqbra{\overline{\set}_{k_1k_2}}
			&\le
			\cfrac{1}{N(k_1,k_2)^2}\sqbra{N(k_1,k_2) + \bra{N(k_1,k_2)^2-N(k_1,k_2)} \times \bra{\frac{8}{N}}}\\
			&\le
			\cfrac{1}{N(k_1,k_2)} + \frac{8}{N}.
			\end{split}
	\end{equation}
	Using \cref{equ::xy2,inequ::xy3,inequ::xy2} together with the inequality above yields
	\begin{equation*}
		\left\vert\Ex\left[\diff_{i_1 i_2}\diff_{i_2 i_3}\diff_{i_3 i_4}\diff_{i_4 i_1}\right]\right\vert\leq\cfrac{N(k_1,k_2)^2 + 8\frac{N(k_1,k_2)^3}{N} + 6\sqbra{N(k_1,k_2)+1}^2}{\sqbra{N(k_1,k_2)-1}\sqbra{N(k_1,k_2)-2}\sqbra{N(k_1,k_2)-3}}.
	\end{equation*}
	Provided $n_{k} \geq 5$ for all $k \in \Nset{K}$, then $N(k_1,k_2) \geq 10$ for all $k_1, k_2 \in \Nset{K}$. Furthermore, $N(k_1,k_2) - m \ge (1-\frac{m}{10})N(k_1,k_2)$ for each integer $0 \le m \le 9$, and $N(k_1,k_2) + 1 \le \frac{11}{10}N(k_1,k_2)$. This yields
	\begin{equation*}
		\left\vert\Ex\left[\diff_{i_1 i_2}\diff_{i_2 i_3}\diff_{i_3 i_4}\diff_{i_4 i_1}\right]\right\vert
		\leq \frac{10}{9}\frac{10}{8}\frac{10}{7}\left(\frac{1}{N(k_1,k_2)} + \frac{8}{N} + 6\cdot\frac{11^{2}}{10^{2}}\frac{1}{N(k_1,k_2)}\right)
		\leq \frac{33}{N(k_{1},k_{2})}.
	\end{equation*}
	Given any choice of $k_{1},k_{2},k_{3},k_{4}$ satisfying $k_{1}=k_{3}$ and $k_{2}=k_{4}$, there are two possibilities when counting the number of tuples (terms) with distinct entries $i_{1},i_{2},i_{3},i_{4}\in\Nset{n}$.
	\begin{enumerate}
		\item Provided $k_{1} = k_{2}$, there are $n_{k_{1}}(n_{k_{1}}-1)(n_{k_{1}}-2)(n_{k_{1}}-3)$ terms, and $N(k_{1},k_{2})=n_{k_{1}}(n_{k_{1}}-1)/2$. Since $(n_{k_{1}}-2)(n_{k_{1}}-3) \le \sqrt{n_{k_{1}}n_{k_{2}}n_{k_{3}}n_{k_{4}}}$, it holds that
		\begin{equation*}
			\sum_{\begin{array}{c} 
					i_1,i_2,i_3,i_4\in\Nset{n}\text{ distinct}\\ g_{i_1}=k_1,g_{i_2}=k_2,\\g_{i_3}=k_3,g_{i_4}=k_4
			\end{array}}
			\left\vert\Ex [\diff_{i_1 i_2}\diff_{i_2 i_3}\diff_{i_3 i_4}\diff_{i_4 i_1}] \right\vert\leq 2 \times 33 \times \sqrt{n_{k_1}n_{k_2}n_{k_3}n_{k_4}}.
		\end{equation*}
		\item Provided $k_{1} \neq k_{2}$, there are $n_{k_{1}}(n_{k_{1}}-1)n_{k_{2}}(n_{k_{2}}-1)$ terms, and $N(k_{1},k_{2})=n_{k_{1}}n_{k_{2}}$. Since $(n_{k_{1}}-1)(n_{k_{2}}-1) \le \sqrt{n_{k_{1}}n_{k_{2}}n_{k_{3}}n_{k_{4}}}$, it holds that
		\begin{equation*}
			\sum_{\begin{array}{c} 
					i_1,i_2,i_3,i_4\in\Nset{n}\text{ distinct}\\ g_{i_1}=k_1,g_{i_2}=k_2,\\g_{i_3}=k_3,g_{i_4}=k_4
			\end{array}}
			\left\vert\Ex [\diff_{i_1 i_2}\diff_{i_2 i_3}\diff_{i_3 i_4}\diff_{i_4 i_1}] \right\vert\leq 1 \times 33 \times \sqrt{n_{k_1}n_{k_2}n_{k_3}n_{k_4}}.
		\end{equation*}
	\end{enumerate}

\noindent\textbf{Case~2: $k_1=k_2$, $k_3=k_4$, $k_1\neq k_3$ (or $k_2=k_3$, $k_1=k_4$, $k_1\neq k_2$) and $g_{i_1}=k_1$, $g_{i_2}=k_2$, $g_{i_3}=k_3$, $g_{i_4}=k_4$.}
	We first analyze the situation $k_1=k_2$, $k_3=k_4$, $k_1 \neq k_3$. In this situation, $\diff_{i_1 i_2}$ is an entry in block $\Diff_{\grp_{k_1}\grp_{k_1}}$ while $\diff_{i_3i_4}$ is an entry in block $\Diff_{\grp_{k_3}\grp_{k_3}}$, and both $\diff_{i_2i_3}$ and $\diff_{i_4i_1}$ are entries in block $\Diff_{\grp_{k_1}\grp_{k_3}}$. Consequently, using \cref{equ::xy3} yields
	\begin{equation*}
		\begin{split}
			&\Ex[\diff_{i_1 i_2}\diff_{i_2 i_3}\diff_{i_3 i_4}\diff_{i_4 i_1}]\\
			&= \sum_{\fset_{k_1k_1},\fset_{k_3k_3},\fset_{k_1k_3}\subset\crbra{\frac{i}{N+1} : i\in\Nset{N}}}\prob\sqbra{\set_{k_1k_1}=\fset_{k_1k_1},\set_{k_3k_3}=\fset_{k_3k_3},\set_{k_1k_3}=\fset_{k_1k_3}}\\
			&\hspace{2em}\times \Ex_{\fset_{k_1k_1}}\sqbra{a-\rankspec_{k_1k_1}} \times \Ex_{\fset_{k_3k_3}}\sqbra{a-\rankspec_{k_3k_3}} \times \Ex_{\fset_{k_1k_3}}\sqbra{\bra{a-\rankspec_{k_1k_3}}\bra{b-\rankspec_{k_1k_3}}}\\
			&= \sum_{\fset_{k_1k_1},\fset_{k_3k_3},\fset_{k_1k_3}\subset\crbra{\frac{i}{N+1} : i\in\Nset{N}}}\prob\sqbra{\set_{k_1k_1}=\fset_{k_1k_1},\set_{k_3k_3}=\fset_{k_3k_3},\set_{k_1k_3}=\fset_{k_1k_3}}\\
			&\hspace{2em}\times \bra{\overline{\fset}_{k_1k_1}-\rankspec_{k_1k_1}} \times \bra{\overline{\fset}_{k_3k_3}-\rankspec_{k_3k_3}} \times \left(\tfrac{N(k_1,k_3)\left(\overline{\fset}_{k_1k_3}-\rankspec_{k_1k_3}\right)^2-\Ex_{\fset_{k_1k_3}}[(a-\rankspec_{k_1k_3})^2]}{N(k_1,k_3)-1}\right).
		\end{split}
	\end{equation*}
	All entries of both $\nrk{A}$ and $\rankspec$ are in the interval $[0,1]$, which together with \cref{inequ::var} yields
	\begin{equation*}
		\begin{split}
			\left\vert\Ex\diff_{i_1 i_2}\diff_{i_2 i_3}\diff_{i_3 i_4}\diff_{i_4 i_1} \right\vert
			&\leq \frac{N(k_1,k_3)}{N(k_1,k_3)-1}\Ex\sqbra{\bra{\overline{\set}_{k_1k_3}-\rankspec_{k_1k_3}}^2}+\frac{1}{N(k_1,k_3)-1}\\
			&\leq \frac{N(k_1,k_3)}{N(k_1,k_3)-1}\Var\sqbra{\overline{\set}_{k_1k_3}}+\frac{1}{N(k_1,k_3)-1}\\
			&\leq \frac{N(k_1,k_3)}{N(k_1,k_3)-1}\sqbra{\cfrac{1}{N(k_1,k_3)}+\frac{8}{N}}+\frac{1}{N(k_1,k_3)-1}\\
			&\leq \frac{12}{N(k_{1},k_{3})} \quad \textnormal{ provided } N(k_{1},k_{3}) \ge 10.
		\end{split}
	\end{equation*}
	In this situation, $N(k_{1},k_{3})=n_{k_{1}}n_{k_{3}}$ and there are $n_{k_{1}}(n_{k_{1}}-1)n_{k_{3}}(n_{k_{3}}-1)$ distinct tuples of $i_{1},i_{2},i_{3},i_{4}$. Since $(n_{k_{1}}-1)(n_{k_{3}}-1) \le \sqrt{n_{k_{1}}n_{k_{2}}n_{k_{3}}n_{k_{4}}}$, it holds that
	\begin{equation*}
		\sum_{
			\begin{array}{c} 
				i_1,i_2,i_3,i_4\in\Nset{n}\text{ distinct}\\
				g_{i_1}=k_1,g_{i_2}=k_2,\\g_{i_3}=k_3,g_{i_4}=k_4
			\end{array}}
		\left\vert\Ex \diff_{i_1 i_2}\diff_{i_2 i_3}\diff_{i_3 i_4}\diff_{i_4 i_1}\right \vert \leq 1 \times 12 \times \sqrt{n_{k_1}n_{k_2}n_{k_3}n_{k_4}}.
	\end{equation*}
	The same analysis applies to the remaining situation $k_2=k_3$, $k_1=k_4$, $k_1\neq k_2$, yielding the same bound in the above display equation.
	
\noindent\textbf{Case~3: $k_1=k_3$, $k_2\neq k_4$ (or $k_1\neq k_3$, $k_2= k_4$) and $g_{i_1}=k_1$, $g_{i_2}=k_2$, $g_{i_3}=k_3$, $g_{i_4}=k_4$.}
	We first analyze the situation $k_1=k_3$, $k_2\neq k_4$. Without loss of generality, we assume that $n_{k_2} \geq n_{k_4}$. In this case, $\diff_{i_1 i_2}$, $\diff_{i_2 i_3}$ are entries of block $\Diff_{\grp_{k_1}\grp_{k_2}}$ and $\diff_{i_3 i_4}$, $\diff_{i_4 i_1}$ are entries of block $\Diff_{\grp_{k_1}\grp_{k_4}}$. Thus,
	\begin{equation*}
		\begin{split}
			&\Ex[\diff_{i_1 i_2}\diff_{i_2 i_3}\diff_{i_3 i_4}\diff_{i_4 i_1}]\\
			&= \sum_{\fset_{k_1k_2},\fset_{k_1k_4}\subset\crbra{\frac{i}{N+1} : i\in\Nset{N}}}\prob\sqbra{\set_{k_1k_2}=\fset_{k_1k_2},\set_{k_1k_4}=\fset_{k_1k_4}}\\
			&\hspace{2em} \times \Ex_{\fset_{k_1k_2}}\sqbra{\bra{a-\rankspec_{k_1k_2}}\bra{b-\rankspec_{k_1k_2}}} \times \Ex_{\fset_{k_1k_4}}\sqbra{\bra{a-\rankspec_{k_1k_4}}\bra{b-\rankspec_{k_1k_4}}}
		\end{split}
	\end{equation*}
	and
	\begin{equation*}
		\begin{split}
			&\left\vert\Ex\diff_{i_1 i_2}\diff_{i_2 i_3}\diff_{i_3 i_4}\diff_{i_4 i_1} \right\vert\\
			&\leq \sum_{\fset_{k_1k_2},\fset_{k_1k_4}\subset\crbra{\frac{i}{N+1} : i\in\Nset{N}}}\prob\sqbra{\set_{k_1k_2}=\fset_{k_1k_2},\set_{k_1k_4}=\fset_{k_1k_4}}\\
			&\hspace{2em}\times \cfrac{N(k_1,k_2)\left(\overline{\fset}_{k_1k_2}-\rankspec_{k_1k_2}\right)^2+\Ex_{\fset_{k_1k_2}}[(a-\rankspec_{k_1k_2})^2]}{N(k_1,k_2)-1}.
		\end{split}
	\end{equation*}
	The above inequality is obtained using \cref{equ::xy3} and $\vert a-\rankspec_{k_1k_4} \vert\leq 1$ for any $a\in \fset_{k_1k_4}$. Furthermore, applying \cref{inequ::var} yields 
	\begin{equation*}
		\begin{split}
			\left\vert\Ex\diff_{i_1 i_2}\diff_{i_2 i_3}\diff_{i_3 i_4}\diff_{i_4 i_1} \right\vert
			&\le \frac{N(k_1,k_2)}{N(k_1,k_2)-1}\Var[\overline{\set}_{k_1,k_2}]+\frac{1}{N(k_1,k_2)-1}\\
			&\leq \frac{12}{N(k_{1},k_{2})} \quad \textnormal{ provided } N(k_{1},k_{2}) \ge 10.
		\end{split},
	\end{equation*}
	In particular,
	\begin{equation*}
		\sum_{
			\begin{array}{c} 
				i_1,i_2,i_3,i_4\in\Nset{n}\text{ distinct}\\
				g_{i_1}=k_1,g_{i_2}=k_2,\\
				g_{i_3}=k_3,g_{i_4}=k_4
			\end{array}}
		\left\vert\Ex \diff_{i_1 i_2}\diff_{i_2 i_3}\diff_{i_3 i_4}\diff_{i_4 i_1}\right\vert\leq 2 \times 12 \times \sqrt{n_{k_1}n_{k_2}n_{k_3}n_{k_4}}.
	\end{equation*}
	Here the constant $2$ comes from considering the separate cases $k_2 = k_1$ and $k_2 \neq k_1$.
	
	The above analysis also applies to the situation $k_1\neq k_3$, $k_2=k_4$, yielding the same upper bound in the display equation.

\noindent\textbf{Case~4: Either of the five situations for $k_1,k_2,k_3,k_4$ and $g_{i_1}=k_1$, $g_{i_2}=k_2$, $g_{i_3}=k_3$, $g_{i_4}=k_4$.
	\begin{itemize}
		\item $k_1,k_2,k_3,k_4$ are pairwise different.
		\item $k_2,k_3,k_4$ are pairwise different with $k_1=k_2$.
		\item $k_3,k_4,k_1$ are pairwise different with $k_2=k_3$.
		\item $k_1,k_2,k_3$ are pairwise different with $k_4=k_3$.
		\item $k_1,k_2,k_3$ are pairwise different with $k_4=k_1$.
	\end{itemize}
 }
	In all five situations, $\diff_{i_1 i_2}$, $\diff_{i_2 i_3}$, $\diff_{i_3 i_4}$, $\diff_{i_4 i_1}$ belong to four different blocks of $\Diff$. We have
	{\small
	\begin{equation*}
		\begin{split}
			&\Ex[\diff_{i_1 i_2}\diff_{i_2 i_3}\diff_{i_3 i_4}\diff_{i_4 i_1}]\\
			&= \sum_{\substack{\fset_{k_1k_2},\fset_{k_2k_3},\fset_{k_3k_4},\fset_{k_4k_1} \\
			\subset\crbra{\frac{i}{N+1} : i\in\Nset{N}}}}
			\prob\sqbra{\set_{k_1k_2}=\fset_{k_1k_2},\set_{k_2k_3}=\fset_{k_2k_3},\set_{k_3k_4}=\fset_{k_3k_4},\set_{k_4k_1}=\fset_{k_4k_1}}\\
			&\hspace{4em} \times \Ex_{\fset_{k_1k_2}}\sqbra{a-\rankspec_{k_1k_2}} \times \Ex_{\fset_{k_2k_3}}\sqbra{a-\rankspec_{k_2k_3}} \times \Ex_{\fset_{k_3k_4}}\sqbra{a-\rankspec_{k_3k_4}} \times \Ex_{\fset_{k_4k_1}}\sqbra{a-\rankspec_{k_4k_1}}\\
			&= \Ex\crbra{\sqbra{\overline{\set}_{k_1k_2}-\rankspec_{k_1k_2}}\sqbra{\overline{\set}_{k_2k_3}-\rankspec_{k_2k_3}}\sqbra{\overline{\set}_{k_3k_4}-\rankspec_{k_3k_4}}\sqbra{\overline{\set}_{k_4k_1}-\rankspec_{k_4k_1}}}.
		\end{split}
	\end{equation*}
	}
	Furthermore,
	{\small
	\begin{equation}\label{inequ::xy4}
		\begin{split}
			&\left\vert\Ex\diff_{i_1 i_2}\diff_{i_2 i_3}\diff_{i_3 i_4}\diff_{i_4 i_1}\right\vert\\
			&\hspace{2em}\leq \sqrt{\Ex\crbra{\sqbra{\overline{\set}_{k_1k_2}-\rankspec_{k_1k_2}}^2\sqbra{\overline{\set}_{k_2k_3}-\rankspec_{k_2k_3}}^2}\Ex\crbra{\sqbra{\overline{\set}_{k_3k_4}-\rankspec_{k_3k_4}}^2\sqbra{\overline{\set}_{k_4k_1}-\rankspec_{k_4k_1}}^2}} \\
			&\hspace{2em}\le \sqrt{\Ex\crbra{\sqbra{\overline{\set}_{k_1k_2}-\rankspec_{k_1k_2}}^2}\Ex\crbra{\sqbra{\overline{\set}_{k_3k_4}-\rankspec_{k_3k_4}}^2}}\\
			&\hspace{2em}\leq \sqrt{\Var\sqbra{\overline{\set}_{k_1k_2}}\Var\sqbra{\overline{\set}_{k_3k_4}}}\\
			&\hspace{2em}\leq \frac{9}{\sqrt{N(k_{1},k_{2}) N(k_{3},k_{4})}} \qquad \textnormal{ provided } N(k_{1},k_{2}), N(k_{3},k_{4}) \ge 10.
		\end{split}
	\end{equation}
	}
	Here, the first inequality is obtained via the Cauchy--Schwarz inequality, and the fourth inequality is obtained using \cref{inequ::var}. Consequently, in the present case,
	\begin{equation*}
		\sum_{
			\begin{array}{c} 
				i_1,i_2,i_3,i_4\in\Nset{n}\text{ distinct}\\
				g_{i_1}=k_1,g_{i_2}=k_2,\\
				g_{i_3}=k_3,g_{i_4}=k_4
			\end{array}}
		\left\vert\Ex \diff_{i_1 i_2}\diff_{i_2 i_3}\diff_{i_3 i_4}\diff_{i_4 i_1}\right\vert
		\leq
		2 \times 9 \times \sqrt{n_{k_1}n_{k_2}n_{k_3}n_{k_4}}.
	\end{equation*}

	\begin{center}
		$(\star)$
	\end{center}

	The above analysis considers all possible configurations of $g_{i_1}$, $g_{i_2}$, $g_{i_3}$, $g_{i_4}$ when $i_{1}$, $i_2$, $i_3$, $i_4$ are pairwise different. These four cases cover all possible scenarios for $g_{i_1},g_{i_2},g_{i_3},g_{i_4}$. Provided $n_k\geq 5$ for all $k\in\Nset{K}$, then for any $k_1,k_2,k_3,k_4\in\Nset{K}$, we have established the upper bound
	\begin{equation}
		\sum_{
			\begin{array}{c} 
				i_1,i_2,i_3,i_4\in\Nset{n}\text{ distinct}\\
				g_{i_1}=k_1,g_{i_2}=k_2,\\g_{i_3}=k_3,g_{i_4}=k_4
			\end{array}}
		\left\vert\Ex \diff_{i_1 i_2}\diff_{i_2 i_3}\diff_{i_3 i_4}\diff_{i_4 i_1}\right\vert
		\leq
		66\sqrt{n_{k_1}n_{k_2}n_{k_3}n_{k_4}}
	\end{equation}
	which holds irrespective of the values taken by $k_{1}, k_{2}, k_{3}, k_{4}$.

	Finally,
	\begin{equation*}
		\begin{split}
			&\left\vert\sum_{
				\begin{array}{c}
					i_1,i_2,i_3,i_4\in\Nset{n}\text{}\\
					\text{ pairwise different}
				\end{array}}
			\Ex[\diff_{i_1 i_2}\diff_{i_2 i_3}\diff_{i_3 i_4}\diff_{i_4 i_1}]\right\vert \\
			&\hspace{8em}\leq \sum_{k_1,k_2,k_3,k_4\in\Nset{K}}\sum_{
				\begin{array}{c} 
					i_1,i_2,i_3,i_4\in\Nset{n}\text{}\\
					\text{ pairwise different,}\\
					g_{i_1}=k_1,g_{i_2}=k_2,\\g_{i_3}=k_3,g_{i_4}=k_4
				\end{array}}
			\left\vert\Ex \diff_{i_1 i_2}\diff_{i_2 i_3}\diff_{i_3 i_4}\diff_{i_4 i_1}\right\vert\\
			&\hspace{8em}\leq
			66\bra{\sum_{k_1,k_2,k_3,k_4\in\Nset{K}}\sqrt{n_{k_1}n_{k_2}n_{k_3}n_{k_4}}}\\
			&\hspace{8em}\leq
			66\bra{\sum_{k\in\Nset{K}}\sqrt{n_k}}^4\\
			&\hspace{8em}\leq
			66\bra{\sqrt{K n}}^4\\
			&\hspace{8em}\leq
			66 K^{2} n^{2}.
		\end{split}
	\end{equation*}
	Above, we use the fact that for the $K$-dimensional vector of all ones, $1_{K}$ and $\sqrt{\vec{n}} = (\sqrt{n_{1}},\dots,\sqrt{n_{K}})^{\tp}$, we have $\|\sqrt{\vec{n}}\|_{\ell_{2}}^{2}=n$. By the Cauchy--Schwarz inequality,
	\begin{equation*}
		|\langle \sqrt{\vec{n}},1_{K}\rangle|
		\le
		\|\sqrt{\vec{n}}\|_{\ell_{2}}\|1_{K}\|_{\ell_{2}}
		=
		\sqrt{K n}.
	\end{equation*}
	Using the result above together with \cref{inequ::xy1}, we finally have
	\begin{equation*}
		\Ex\left[\trace({\Diff^4})\right]
		\leq
		66K^{2}n^{2} + 2n^{3}.
	\end{equation*}
	This concludes the proof of \cref{lem:ms_trace}.
\end{proof}

\subsection{Proofs: a combinatorial expectation bound}
\label{app:higher_order_concentration}
\begin{lemma}[A combinatorial expectation bound]\label[lemma]{lem:ho_cctrtn}
	Let $A \sim \operatorname{Blockmodel}\blockmodel$ per \cref{def:data_matrices}. Let $\nrk{A}$ denote the matrix of normalized rank statistics of $A$, per \cref{alg:ptr_baseline}. There exists a universal constant $C > 0$ such that if $n_{k} \geq 4$ for all $k \in \Nset{K}$, then for each $i \in \Nset{n}$,
	{\small
	\begin{equation}
		\label{equ::higher_order_concentration_bound}
	\begin{split}
		&\left\vert 
		\sum_{r\in\Nset{K}}\frac{1}{n_r}\Ex\left[\sum_{h,l\in\Nset{n};s,t\in \grp_{r}} \left\{\nrk{A}-\Ex[\nrk{A}]\right\}_{ih}\left\{\nrk{A}-\Ex[\nrk{A}]\right\}_{il}\left\{\nrk{A}-\Ex[\nrk{A}]\right\}_{hs}\left\{\nrk{A}-\Ex[\nrk{A}]\right\}_{lt} \right] \right\vert\\
		&\leq
		C\max\left\{\frac{n^{2}}{n_{g_{i}}},\sqrt{\frac{K^{3}n^{3}}{n_{g_{i}}}} \right\}.
	\end{split}	
	\end{equation}
	}
\end{lemma}

\begin{proof}[Proof of \cref{lem:ho_cctrtn}]
We adopt the notation used in the proof of \cref{lem:os_trace}, namely $\Diff = \nrk{A} - \Ex[\nrk{A}]$ whose entry $(i,h)$ is denoted by $\gamma_{ih}$. With this convention, we express the left-hand side of \cref{equ::higher_order_concentration_bound} in the more compact form
\begin{equation}
	\label{equ::XY_higher_order_concentration_term}
	\left\vert \sum_{r\in\Nset{K}}\frac{1}{n_r}\sum_{h,l\in\Nset{n};s,t\in \grp _{r}} \Ex[\diff_{ih}\diff_{il}\diff_{hs}\diff_{lt}]
	\right\vert.
\end{equation}
In what follows, we bound $\vert \Ex[\diff_{ih}\diff_{il}\diff_{hs}\diff_{lt}] \vert$ by considering different cases for indices $h,l,s,t$, restricting to situations when $i\neq h$, $i\neq l$, $h\neq s$ and $l\neq t$, since otherwise the expectation vanishes. We group all possible configurations of $h,l,s,t$, under the above restriction, into four cases. Importantly, $\{i,h\}$, $\{i,l\}$, $\{h,s\}$ and $\{l,t\}$ are four distinct sets if and only if $h\neq l$, $s\neq i$, $t\neq i$ and $(s,t)\neq (l,h)$. The four cases of interest for $h,l,s,t$ are given below.

\begin{itemize}
	\item[] \noindent\textbf{Case~1:} $h\neq l$, $s\neq i$, $t\neq i$, and $(s,t)\neq (l,h)$,
	\item[] \noindent\textbf{Case~2:} $h= l$,
	\item[] \noindent\textbf{Case~3:} $h\neq l$, $i\in\{s,t\}$, and $(s,t)\neq (l,h)$,
	\item[]\noindent\textbf{Case~4:} $(s,t)=(l,h)$.
\end{itemize}
We shall handle these cases by further dividing each one into several subcases which can be analyzed separately. In this proof, we will repeatedly use the fact that, when $n_k \geq 4$ for all $k \in \Nset{K}$, then $N(k_1,k_2)\geq 6$ for all $k_1,k_2\in\Nset{K}$, with $N(k_1,k_2)$ denoting the number of unique matrix entries in the $(k_{1},k_{2})$ block, as defined in the proof of \cref{lem:ms_trace}.\\

\noindent\textbf{Case~1.1:} $h\neq l$, $s\neq i$, $t\neq i$, $(s,t)\neq (l,h)$, $g_h=g_l$ and $r\neq g_i$.\\ Denote $g_h$ by $k$. In this case, the variables $\diff_{ih}$, $\diff_{il}$ are entries of the block $\Diff_{\grp _{g_i}\grp _{k}}$, and $\diff_{hs}$, $\diff_{lt}$ are entries of the block $\Diff_{\grp _{k}\grp _{r}}$. Using the notation for $\set_{k_1k_2}$, $\fset_{k_1k_2}$ and $\Ex_{\fset_{k_1k_2}}$ as in the proof of \cref{lem:ms_trace}, we compute
\begin{equation}
	\label{inequ::xy5}
	\begin{split}
		&\ABS{\Ex\sqbra{\diff_{ih}\diff_{il}\diff_{hs}\diff_{lt}}}\\
		&\hspace{1em}=\Bigg\lvert\sum_{\fset_{g_ik},\fset_{kr}\subset \crbra{\frac{i}{N+1} : i\in\Nset{N}}}\prob\sqbra{\set_{g_ik}=\fset_{g_ik},\set_{kr}=\fset_{kr}}\\
		&\hspace{3em}\times \Ex_{\fset_{g_ik}}\sqbra{(a-\rankspec_{g_ik})(b-\rankspec_{g_ik})}\Ex_{\fset_{kr}}\sqbra{(a-\rankspec_{kr})(b-\rankspec_{kr})} \Bigg\rvert\\
		&\hspace{1em}\leq \sum_{\fset_{kr}\subset \crbra{\frac{i}{N+1} : i\in\Nset{N}}}\prob\sqbra{\set_{kr}=\fset_{kr}}\cfrac{N(k,r)\Ex^2_{\fset_{kr}}[a-\rankspec_{kr}]+\Ex_{\fset_{kr}}[(a-\rankspec_{kr})^2]}{N(k,r)-1}\\
		&\hspace{1em}\leq \frac{6}{5}\sum_{\fset_{kr}\subset \crbra{\frac{i}{N+1} : i\in\Nset{N}}}\prob\sqbra{\set_{kr}=\fset_{kr}}(\overline{\fset}_{kr}-\rankspec_{kr})^2+\frac{1}{N(k,r)-1}\\
		&\hspace{1em}\leq \frac{6}{5}\bra{\Var[\overline{\set}_{kr}]+\frac{1}{N(k,r)}}.
	\end{split}
\end{equation}
The first inequality above is established using \cref{equ::xy3} and the fact that $\Ex_{\fset_{g_ik}}[(a-\rankspec_{g_ik})(b-\rankspec_{g_ik})]\in [-1,1]$. The second inequality invokes the relation $N(k,r) - 1 \ge \frac{5}{6}N(k,r)$ since $N(k,r) \ge 6$. We showed in \cref{inequ::var} that there exists a constant $C_{1} > 0$, such that $\Var[\overline{\set}_{k_1 k_2}]\leq C_1/N(k_1,k_2)$ for any $k_1,k_2\in\Nset{K}$. Thus,
\begin{equation}
	\ABS{\Ex\sqbra{\diff_{ih}\diff_{il}\diff_{hs}\diff_{lt}}}\leq \frac{6(C_1+1)}{5N(k,r)}.
\end{equation}
For any fixed $k$ and fixed $r\neq g_i$, the number of $h,l,s,t$ tuples satisfying Case~1.1 does not exceed $n_{k}^{2}n_{r}^{2}$. Therefore,
\begin{equation}
	\begin{split}
		&\left\vert \sum_{r\in\Nset{K}}\frac{1}{n_r}\sum_{\substack{h,l\in\Nset{n};s,t\in \grp _{r}\\ h,l,s,t \text{ satisfy Case~1.1}}} \Ex[\diff_{ih}\diff_{il}\diff_{hs}\diff_{lt}] \right\vert\\
		&\hspace{10em}\leq
		\sum_{r\in\Nset{K}-\{g_i\}}\frac{1}{n_r}\sum_{k\in\Nset{K}}\frac{6\bra{C_1+1}n_k^2n_r^2}{5N(k,r)}\\
		&\hspace{10em}\leq
		\frac{6 \cdot 2 \cdot 4}{5 \cdot 3}\bra{C_1+1}\sum_{r\in\Nset{K}-\{g_i\}}\frac{1}{n_r}\sum_{k\in\Nset{K}}n_kn_r\\
		&\hspace{10em}\leq
		\frac{16}{5}(C_1+1)(K-1)n.
	\end{split}
\end{equation}
Above, we have used the fact that $n_{k} \ge 4$ implies $n_{k}/(n_{k}-1) \le 4/3$.\\

\noindent\textbf{Case~1.2:} $h\neq l$, $s\neq i$, $t\neq i$, $(s,t)\neq (l,h)$, $g_h=g_l$, and $r= g_i$.\\
We denote $g_h$ by $k$. In this case, $\diff_{ih}$, $\diff_{il},\diff_{hs}$, $\diff_{lt}$ are all entries of the block $\Diff_{\grp _{g_i}\grp _{k}}$. Thus,
{\small
\begin{equation}
	\begin{split}
		&\ABS{\Ex\sqbra{\diff_{ih}\diff_{il}\diff_{hs}\diff_{lt}}}\\
		&\hspace{1em}=\left\vert\sum_{\fset_{g_ik}\subset \crbra{\frac{i}{N+1} : i\in\Nset{N}}}\prob\sqbra{\set_{g_ik}=\fset_{g_ik}}\Ex_{\fset_{g_ik}}\sqbra{(a-\rankspec_{g_ik})(b-\rankspec_{g_ik})(c-\rankspec_{g_ik})(d-\rankspec_{g_ik})}\right\vert\\
		&\hspace{1em}\leq \frac{N(g_i,k)^3\Ex\sqbra{(\overline{\set}_{g_i,k}-\rankspec_{g_ik})^4}+6\sqbra{N(g_i,k)+1}^2}{\sqbra{N(g_i,k)-1}\sqbra{N(g_i,k)-2}\sqbra{N(g_i,k)-3}} \\
		&\hspace{1em}\leq \frac{N(g_i,k)^3\Var[\overline{\set}_{g_ik}]+6\sqbra{N(g_i,k)+1}^2}{\sqbra{N(g_i,k)-1}\sqbra{N(g_i,k)-2}\sqbra{N(g_i,k)-3}}\\
		&\hspace{1em}\leq \frac{C_1N(g_i,k)^2+6\sqbra{N(g_i,k)+1}^2}{\sqbra{N(g_i,k)-1}\sqbra{N(g_i,k)-2}\sqbra{N(g_i,k)-3}}\\
		&\hspace{1em}\leq \bra{\frac{18C_1}{5}+\frac{147}{5}}\frac{1}{N(g_i,k)}.
	\end{split}
\end{equation}
}
Above, the first inequality is established using \cref{equ::xy1} and \cref{inequ::xy2}, and the second inequality is established using \cref{inequ::xy3}. For any fixed $k$, the number of $h,l,s,t$ tuples satisfying Case~1.2 does not exceed $n_{g_i}^{2} n_{k}^{2}$. Thus,
\begin{equation}
	\begin{split}
		\left\vert \sum_{r\in\Nset{K}}\frac{1}{n_r}\sum_{\substack{h,l\in\Nset{n};s,t\in \grp _{r}\\ h,l,s,t \text{ satisfy Case~1.2}}} \Ex[\diff_{ih}\diff_{il}\diff_{hs}\diff_{lt}] \right\vert
		&\leq \frac{1}{n_{g_i}}\sum_{k\in\Nset{K}}\bra{\frac{18C_1}{5}+\frac{147}{5}}\frac{n_{g_i}^2n_{k}^2}{N(g_i,k)}\\
		&\leq \bra{\frac{48C_1}{5}+\frac{392}{5}}n.
	\end{split}
\end{equation}

\noindent\textbf{Case~1.3:} $h\neq l$, $s\neq i$, $t\neq i$, $(s,t)\neq (l,h)$, $g_h\neq g_l$, $r\neq g_i$, and $\{g_h,g_l\}\neq\{g_i,r\}$.\\ In this case, $\diff_{ih}$, $\diff_{il}$, $\diff_{hs}$, $\diff_{lt}$ are in four different blocks of $\Diff$, determined by the structure in the membership matrix $\Theta$. Denote $g_h$ and $g_l$ by $k$ and $k'$, respectively. Then,
{\small
\begin{equation}
	\begin{split}
		&\ABS{\Ex\sqbra{\diff_{ih}\diff_{il}\diff_{hs}\diff_{lt}}}\\
		&\hspace{1em}=\Bigg\lvert \sum_{\fset_{g_ik},\fset_{g_ik'},\fset_{kr},\fset_{k'r}\subset \crbra{\frac{i}{N+1} : i\in\Nset{N}}}\prob\sqbra{\set_{g_ik}=\fset_{g_ik},\set_{g_ik'}=\fset_{g_ik'},\set_{kr}=\fset_{kr},\set_{k'r}=\fset_{k'r}}\\
		&\hspace{4em}\times \Ex_{\fset_{g_ik}}\sqbra{(a-\rankspec_{g_ik})}\Ex_{\fset_{g_ik'}}\sqbra{(a-\rankspec_{g_ik'})}\Ex_{\fset_{kr}}\sqbra{(a-\rankspec_{kr})}\Ex_{\fset_{k'r}}\sqbra{(a-\rankspec_{k'r})} \Bigg\rvert\\
		&\hspace{1em}=\left\vert\Ex\sqbra{(\overline{\set}_{g_ik}-\rankspec_{g_ik})(\overline{\set}_{g_ik'}-\rankspec_{g_ik'})(\overline{\set}_{kr}-\rankspec_{kr})(\overline{\set}_{k'r}-\rankspec_{k'r})} \right\vert\\
		&\hspace{1em}\leq \sqrt{\Var\sqbra{\overline{\set}_{g_ik}}\Var\sqbra{\overline{\set}_{k'r}}}\\
		&\hspace{1em}\leq \frac{C_1}{\sqrt{N(g_i,k)N(k',r)}}.
	\end{split}
\end{equation}
}
Above, the first inequality is established using \cref{inequ::xy4}. For any fixed $r,k,k'\in\Nset{K}$ satisfying $r\neq g_i$, $k\neq k'$ and $\{k,k'\}\neq\{g_i,r\}$, the number of $h,l,s,t$ tuples satisfying Case~1.3 does not exceed $n_{k} n_{k'} n_{r}^{2}$. Thus,
\begin{equation}
	\begin{split}
		\left\vert \sum_{r\in\Nset{K}}\frac{1}{n_r}\sum_{\substack{h,l\in\Nset{n};s,t\in \grp _{r}\\ h,l,s,t \text{ satisfy Case~1.3}}} \Ex[\diff_{ih}\diff_{il}\diff_{hs}\diff_{lt}] \right\vert&\leq\sum_{r\in\Nset{K}-\{g_i\}} \frac{1}{n_{r}}\sum_{k,k'\in\Nset{K}}\frac{C_1n_kn_{k'}n_r^2}{\sqrt{N(g_i,k)N(k',r)}}\\
		&\leq C_1\sqrt{\frac{8}{3}}\frac{1}{\sqrt{n_{g_i}}} \sum_{k,k',r\in\Nset{K}}\sqrt{n_kn_{k'}n_r}\\
		&\leq C_1\sqrt{\frac{8}{3}}\frac{1}{\sqrt{n_{g_i}}} \bra{\sum_{k\in\Nset{K}}\sqrt{n_k}}^3\\
		&\leq C_1\sqrt{\frac{8}{3}}\sqrt{\frac{K^{3} n^{3} }{n_{g_i}}}.
	\end{split}
\end{equation}

\noindent\textbf{Case~1.4:} $h\neq l$, $s\neq i$, $t\neq i$, $(s,t)\neq (l,h)$, $g_h\neq g_l$, $r\neq g_i$ and $\{g_h,g_l\}=\{g_i,r\}$.\\
If $(g_h,g_l)=(g_i,r)$, then $\diff_{ih}$ and $\diff_{lt}$ are entries of blocks $\Diff_{\grp_{g_i}\grp_{g_i}}$ and $\Diff_{\grp_{r}\grp_{r}}$, respectively, while $\diff_{il}$ and $\diff_{hs}$ are entries of the block $\Diff_{\grp_{g_i}\grp_{r}}$. Thus,
\begin{equation}
	\begin{split}
		&\ABS{\Ex\sqbra{\diff_{ih}\diff_{il}\diff_{hs}\diff_{lt}}}\\
		&\hspace{1em}=\Bigg\lvert\sum_{\fset_{g_ig_i},\fset_{rr},\fset_{g_ir}\subset \crbra{\frac{i}{N+1} : i\in\Nset{N}}}\prob\sqbra{\set_{g_ig_i}=\fset_{g_ig_i},\set_{rr}=\fset_{rr},\set_{g_ir}=\fset_{g_ir}}\\
		&\hspace{4em}\times \Ex_{\fset_{g_ig_i}}\sqbra{(a-\rankspec_{g_ig_i})}\Ex_{\fset_{rr}}\sqbra{(a-\rankspec_{rr})}\Ex_{\fset_{g_ir}}\sqbra{(a-\rankspec_{g_ir})(b-\rankspec_{g_ir})} \Bigg\rvert\\
		&\hspace{1em}\leq \sum_{\fset_{g_ir}\subset \crbra{\frac{i}{N+1} : i\in\Nset{N}}}\prob\sqbra{\set_{g_ir}=\fset_{g_ir}}\cfrac{N(g_i,r)\Ex^2_{\fset_{g_ir}}[a-\rankspec_{g_ir}]+\Ex_{\fset_{g_ir}}[(a-\rankspec_{g_ir})^2]}{N(g_i,r)-1}\\
		&\hspace{1em}\leq \frac{6}{5}\bra{\Var[\overline{\set}_{g_ir}]+\frac{1}{N(g_i,r)}}\\
		&\hspace{1em}\leq \frac{6}{5}\bra{C_1+1}\frac{1}{N(g_i,r)}.
	\end{split}
\end{equation}
The inequalities above are established in the same way as \cref{inequ::xy5}. Note that the last inequality also holds for the situation where $(g_l,g_h)=(g_i,r)$. For any fixed $r\in\Nset{K}$, the number of $h,l,s,t$ tuples satisfying Case~1.4 does not exceed $2n_{g_i}n_r^3$. Therefore,
\begin{equation}
	\begin{split}
		\left\vert \sum_{r\in\Nset{K}}\frac{1}{n_r}\sum_{\substack{h,l\in\Nset{n};s,t\in \grp _{r}\\ h,l,s,t \text{ satisfy Case~1.4}}} \Ex[\diff_{ih}\diff_{il}\diff_{hs}\diff_{lt}] \right\vert
		&\leq \frac{6}{5}\bra{C_1+1}\sum_{r\in\Nset{K}-\{g_i\}} \frac{1}{n_{r}}\frac{2n_{g_i}n_r^3}{N(g_i,r)}\\
		&\leq \frac{12}{5}\bra{C_1+1} \sum_{r\in\Nset{K}-\{g_i\}} n_r\\
		&\leq \frac{12}{5}\bra{C_1+1}(n-n_{g_i}).
	\end{split}
\end{equation}

\noindent\textbf{Case~1.5:} $h\neq l$, $s\neq i$, $t\neq i$, $(s,t)\neq (l,h)$, $g_h\neq g_l$ and $r = g_i$.\\
Denote $g_h$ and $g_l$ by $k$ and $k'$, respectively. In this case, $\diff_{ih}$, $\diff_{hs}$ are entries of the block $\Diff_{\grp_{g_i}\grp_{k}}$ and $\diff_{il}$, $\diff_{lt}$ are entries of the block $\Diff_{\grp_{g_i}\grp_{k'}}$. Thus,
\begin{equation}
	\begin{split}
		&\ABS{\Ex\sqbra{\diff_{ih}\diff_{il}\diff_{hs}\diff_{lt}}}\\
		&\hspace{1em}=\Bigg\lvert\sum_{\fset_{g_ik},\fset_{g_ik'}\subset \crbra{\frac{i}{N+1} : i\in\Nset{N}}}\prob\sqbra{\set_{g_ik}=\fset_{g_ik},\set_{g_ik'}=\fset_{g_ik'}}\\
		&\hspace{4em}\times \Ex_{\fset_{g_ik}}\sqbra{(a-\rankspec_{g_ik})(b-\rankspec_{g_ik})}\Ex_{\fset_{g_ik'}}\sqbra{(a-\rankspec_{g_ik'})(b-\rankspec_{g_ik'})} \Bigg\rvert\\
		&\hspace{1em}\leq \frac{6}{5}\min\crbra{\bra{\Var[\overline{\set}_{g_ik}]+\frac{1}{N(g_i,k)}},\bra{\Var[\overline{\set}_{g_ik'}]+\frac{1}{N(g_i,k')}}}\\
		&\hspace{1em} \leq \frac{6}{5}(C_1+1)\frac{1}{\max\crbra{N(g_i,k),N(g_i,k')}}.
	\end{split}
\end{equation}
The inequalities above are established in the same way as \cref{inequ::xy5}. For fixed $k,k'\in\Nset{K}$, the number of $h,l,s,t$ tuples satisfying Case~1.5 does not exceed $n_{k} n_{k'} n_{g_{i}}^{2}$. Therefore,
\begin{equation}
	\begin{split}
		&\left\vert \sum_{r\in\Nset{K}}\frac{1}{n_r}\sum_{\substack{h,l\in\Nset{n};s,t\in \grp _{r}\\ h,l,s,t \text{ satisfy Case~1.5}}} \Ex[\diff_{ih}\diff_{il}\diff_{hs}\diff_{lt}] \right\vert\\
		&\hspace{5em}
		\leq
		\frac{1}{n_{g_i}}\sum_{k,k'\in\Nset{K},k\neq k'}\frac{6}{5}\bra{C_1+1}\frac{n_kn_{k'}n_{g_i}^2}{\max\crbra{N(g_i,k),N(g_i,k')}}\\
		&\hspace{5em}
		\leq
		\frac{6}{5}\bra{C_1+1} K n.
	\end{split}
\end{equation}

\noindent\textbf{Case~2.1:} $h= l$ and $s=t=i$. The total number of $h,l,s,t$ tuples satisfying Case~2.1 is $n-1$. Using the trivial bound $\vert\Ex[\diff_{ih}\diff_{il}\diff_{hs}\diff_{lt}]\vert\leq 1$ yields
\begin{equation}
	\left\vert \sum_{r\in\Nset{K}}\frac{1}{n_r}\sum_{\substack{h,l\in\Nset{n};s,t\in \grp _{r}\\ h,l,s,t \text{ satisfy Case~2.1}}} \Ex[\diff_{ih}\diff_{il}\diff_{hs}\diff_{lt}] \right\vert \leq \frac{n-1}{n_{g_i}}.
\end{equation}

\noindent\textbf{Case~2.2:} $h= l$ and $s=t\neq i$. With any fixed $r$, the number of $h,l,s,t$ tuples satisfying Case~2.2 does not exceed $n_{r} n$. Therefore,
\begin{equation}
	\left\vert \sum_{r\in\Nset{K}}\frac{1}{n_r}\sum_{\substack{h,l\in\Nset{n};s,t\in \grp _{r}\\ h,l,s,t \text{ satisfy Case~2.2}}} \Ex[\diff_{ih}\diff_{il}\diff_{hs}\diff_{lt}] \right\vert \leq \sum_{r\in\Nset{K}} \frac{n_rn}{n_{r}}\leq K n.
\end{equation}

\noindent\textbf{Case~2.3:} $h= l$, $s\neq t$ and $i\in\{s,t\}$. In this case, since $i\in\{s,t\}$, we have $r=g_i$. Consequently, the total number of $h,l,s,t$ tuples satisfying Case~2.3 does not exceed $2n_{g_i}n$, hence
\begin{equation}
	\left\vert \sum_{r\in\Nset{K}}\frac{1}{n_r}\sum_{\substack{h,l\in\Nset{n};s,t\in \grp _{r}\\ h,l,s,t \text{ satisfy Case~2.3}}} \Ex[\diff_{ih}\diff_{il}\diff_{hs}\diff_{lt}] \right\vert \leq \frac{2n_{g_i}n}{n_{g_i}}\leq 2n.
\end{equation}

\noindent\textbf{Case~2.4:} $h= l$, $s\neq t$, $i\notin\{s,t\}$ and $r=g_i$. Let $k$ denote $g_{h}$. In this case, $\diff_{ih}\diff_{il}\diff_{hs}\diff_{lt}=\diff_{ih}^2\diff_{hs}\diff_{ht}$ and $\diff_{ih},\diff_{hs},\diff_{ht}$ are all entries of the block $\Diff_{\grp_{g_i}\grp_{k}}$. Thus,
{\small
\begin{equation}\label{inequ::xy6}
	\begin{split}
		&\ABS{\Ex\sqbra{\diff_{ih}\diff_{il}\diff_{hs}\diff_{lt}}}\\
		&\hspace{1em}=\Bigg\lvert\sum_{\fset_{g_ik}\subset \crbra{\frac{i}{N+1} : i\in\Nset{N}}}\prob\sqbra{\set_{g_ik}=\fset_{g_ik}}\\
		&\hspace{4em}\times \Ex_{\fset_{g_ik}}\sqbra{(a-\rankspec_{g_ik})^2(b-\rankspec_{g_ik})(c-\rankspec_{g_ik})} \Bigg\rvert\\
		&\hspace{1em}= \Bigg\lvert\sum_{\fset_{g_ik}\subset \crbra{\frac{i}{N+1} : i\in\Nset{N}}}\prob\sqbra{\set_{g_ik}=\fset_{g_ik}}\frac{1}{N(g_i,k)\sqbra{N(g_i,k)-1}\sqbra{N(g_i,k)-2}}\\
		&\hspace{4em}\times \left\{N(g_i,k)^3\Ex^2_{\fset_{g_ik}}[a-\rankspec_{g_ik}]\Ex_{\fset_{g_ik}}[(a-\rankspec_{g_ik})^2]\right.\\
		&\hspace{5em}- N(g_i,k)^2\Ex^2_{\fset_{g_ik}}[(a-\rankspec_{g_ik})^2]+2N(g_i,k) \Ex_{\fset_{g_ik}}[(a-\rankspec_{g_ik})^4]\\
		&\hspace{5em}\left. -2N(g_i,k)^2\Ex_{\fset_{g_ik}}[a-\rankspec_{g_ik}]\Ex_{\fset_{g_ik}}[(a-\rankspec_{g_ik})^3]\right\}\Bigg\rvert\\
		&\hspace{1em}\leq \sum_{\fset_{g_ik}\subset \crbra{\frac{i}{N+1} : i\in\Nset{N}}}\prob\sqbra{\set_{g_ik}=\fset_{g_ik}}\frac{N(g_i,k)^3(\overline{\fset}_{g_ik}-\rankspec_{g_ik})^2+3N(g_i,k)^2+2N(g_i,k)}{N(g_i,k)\sqbra{N(g_i,k)-1}\sqbra{N(g_i,k)-2}}\\
		&\hspace{1em}\leq \frac{9}{5}\bra{\Var\sqbra{\overline{\set}_{g_ik}}+\frac{3}{N(g_i,k)}+\frac{2}{N(g_i,k)^2}}\\
		&\hspace{1em}\leq \frac{9}{5}\bra{C_1+\frac{10}{3}}\frac{1}{N(g_i,k)}.
	\end{split}
\end{equation}
}
Above, the second equation is established using \cref{equ::xy4}. With any fixed $k$, the number of $h,l,s,t$ tuples satisfying Case~2.4 does not exceed $n_{k} n_{g_i}^2$. Thus,
\begin{equation}
	\begin{split}
		\left\vert \sum_{r\in\Nset{K}}\frac{1}{n_r}\sum_{\substack{h,l\in\Nset{n};s,t\in \grp _{r}\\ h,l,s,t \text{ satisfy Case~2.4}}} \Ex[\diff_{ih}\diff_{il}\diff_{hs}\diff_{lt}] \right\vert
		&\leq\frac{1}{n_{g_i}}\sum_{k\in\Nset{K}}\frac{9}{5}\bra{C_1+\frac{10}{3}}\frac{n_kn_{g_i}^2}{N(g_i,k)}\\
		&\leq \bra{\frac{24}{5}C_1+16}K.
	\end{split}
\end{equation}

\noindent\textbf{Case~2.5:} $h= l$, $s\neq t$, $i\notin\{s,t\}$ and $r\neq g_i$. Let $k$ denote $g_{h}$. In this case, $\diff_{ih}\diff_{il}\diff_{hs}\diff_{lt}=\diff_{ih}^2\diff_{hs}\diff_{ht}$, where $\diff_{ih}$ is an entry of the block $\Diff_{\grp_{g_i}\grp_{k}}$. Further, $\diff_{hs}$ and $\diff_{ht}$ are both entries of the block $\Diff_{\grp_{k}\grp_{r}}$. Thus,
\begin{equation}\label{inequ::xy7}
	\begin{split}
		&\ABS{\Ex\sqbra{\diff_{ih}\diff_{il}\diff_{hs}\diff_{lt}}}\\
		&\hspace{1em}= \Bigg\lvert\sum_{\fset_{g_ik},\fset_{kr}\subset \crbra{\frac{i}{N+1} : i\in\Nset{N}}}\prob\sqbra{\set_{g_ik}=\fset_{g_ik},\set_{kr}=\fset_{kr}}\\
		&\hspace{4em}\times \Ex_{\fset_{g_ik}}\sqbra{(a-\rankspec_{g_ik})^2}\Ex_{\fset_{kr}}\sqbra{(a-\rankspec_{kr})(b-\rankspec_{kr})} \Bigg\rvert\\
		&\hspace{1em}\leq \sum_{\fset_{kr}\subset 	\crbra{\frac{i}{N+1} : i\in\Nset{N}}}\prob\sqbra{\set_{kr}=\fset_{kr}}\cfrac{N(k,r)\Ex^2_{\fset_{kr}}[a-\rankspec_{kr}]+\Ex_{\fset_{kr}}[(a-\rankspec_{kr})^2]}{N(k,r)-1}\\
		&\hspace{1em}\leq \frac{6}{5}\sum_{\fset_{kr}\subset 	\crbra{\frac{i}{N+1} : i\in\Nset{N}}}\prob\sqbra{\set_{kr}=\fset_{kr}}(\overline{\fset}_{kr}-\rankspec_{kr})^2+\frac{1}{N(k,r)-1}\\
		&\hspace{1em}\leq
		\frac{6}{5}\bra{\Var[\overline{\set}_{kr}]+\frac{1}{N(k,r)}}\\
		&\hspace{1em}\leq \frac{6}{5}(C_1+1)\frac{1}{N(k,r)}.
	\end{split}
\end{equation}
The first inequality above is established using \cref{equ::xy3} and the fact that $\Ex_{\fset_{g_ik}}[(a-\rankspec_{g_ik})^2]\leq 1$. For any fixed $r$ and $k$, the number of $h,l,s,t$ tuples satisfying Case~2.5 does not exceed $n_{k} n_{r}^{2}$. Therefore,
\begin{equation}
	\begin{split}
		\left\vert \sum_{r\in\Nset{K}}\frac{1}{n_r}\sum_{\substack{h,l\in\Nset{n};s,t\in \grp _{r}\\ h,l,s,t \text{ satisfy Case~2.5}}} \Ex[\diff_{ih}\diff_{il}\diff_{hs}\diff_{lt}] \right\vert&\leq\sum_{r\in\Nset{K}-\{g_i\}} \frac{1}{n_{r}}\sum_{k\in\Nset{K}}\frac{6}{5}(C_1+1)\frac{n_kn_r^2}{N(k,r)}\\
		&\leq \frac{16}{5}(C_1+1)(K-1)K.
	\end{split}
\end{equation}

\noindent\textbf{Case~3.1:} $h\neq l$ and $s=t=i$. The total number of $h,l,s,t$ tuples satisfying Case~3.1 does not exceed $n^2$. Therefore, applying the simple bound $\vert\Ex[\diff_{ih}\diff_{il}\diff_{hs}\diff_{lt}]\vert\leq 1$ yields
\begin{equation}
	\left\vert \sum_{r\in\Nset{K}}\frac{1}{n_r}\sum_{\substack{h,l\in\Nset{n};s,t\in \grp _{r}\\ h,l,s,t \text{ satisfy Case~3.1}}} \Ex[\diff_{ih}\diff_{il}\diff_{hs}\diff_{lt}] \right\vert \leq \frac{n^2}{n_{g_i}}.
\end{equation}

\noindent\textbf{Case~3.2:} $h\neq l$, $s\neq t$, $i\in\{s,t\}$, $(s,t)\neq (l,h)$ and $g_h=g_l$. Denote $g_h$ by $k$. When $s\neq t=i$, we have $\diff_{ih}\diff_{il}\diff_{hs}\diff_{lt}=\diff_{ih}\diff_{il}^2\diff_{hs}$. Further, $\diff_{ih}$,$\diff_{il}$,$\diff_{hs}$ are all entries of the block $\Diff_{\grp_{g_i}\grp_{k}}$. Thus,
\begin{equation}
	\begin{split}
		&\ABS{\Ex\sqbra{\diff_{ih}\diff_{il}\diff_{hs}\diff_{lt}}}\\
		&\hspace{1em}= \Bigg\lvert\sum_{\fset_{g_ik}\subset \crbra{\frac{i}{N+1} : i\in\Nset{N}}}\prob\sqbra{\set_{g_ik}=\fset_{g_ik}}\\
		&\hspace{4em}\times \Ex_{\fset_{g_ik}}\sqbra{(a-\rankspec_{g_ik})^2(b-\rankspec_{g_ik})(c-\rankspec_{g_ik})} \Bigg\rvert\\
		&\hspace{1em}\leq \frac{9}{5}\bra{C_1+\frac{10}{3}}\frac{1}{N(g_i,k)}.
	\end{split}
\end{equation}
The inequality above is established in the same way as \cref{inequ::xy6}. Moreover, this inequality is also true for the situation where $t\neq s=i$. For any fixed $k$, the number of $h,l,s,t$ tuples satisfying Case~3.2 does not exceed $2n_k^2n_{g_i}$. Thus,
\begin{equation}
	\begin{split}
		\left\vert \sum_{r\in\Nset{K}}\frac{1}{n_r}\sum_{\substack{h,l\in\Nset{n};s,t\in \grp _{r}\\ h,l,s,t \text{ satisfy Case~3.2}}} \Ex[\diff_{ih}\diff_{il}\diff_{hs}\diff_{lt}] \right\vert
		&\leq\frac{1}{n_{g_i}}\sum_{k\in\Nset{K}} \frac{18}{5}\bra{C_1+\frac{10}{3}}\frac{n_k^2n_{g_i}}{N(g_i,k)}\\
		&\leq \bra{\frac{48}{5}C_1+32}\frac{n}{n_{g_i}}.
	\end{split}
\end{equation}

\noindent\textbf{Case~3.3:} $h\neq l$, $s\neq t$, $i\in\{s,t\}$, $(s,t)\neq (l,h)$ and $g_h\neq g_l$. Denote $g_h$ and $g_j$ by $k$ and $k'$, respectively. In the situation when $s\neq t=i$, we have $\diff_{ih}\diff_{il}\diff_{hs}\diff_{lt}=\diff_{ih}\diff_{il}^2\diff_{hs}$, where $\diff_{ih},\diff_{hs}$ are both entries of the block $\Diff_{\grp_{g_i}\grp_{k}}$, and $\diff_{il}$ is an entry of the block $\Diff_{\grp_{g_i}\grp_{k'}}$. Thus,
\begin{equation}
	\begin{split}
		&\ABS{\Ex\sqbra{\diff_{ih}\diff_{il}\diff_{hs}\diff_{lt}}}\\
		&\hspace{1em}= \Bigg\lvert\sum_{\fset_{g_ik'},\fset_{g_ik}\subset \crbra{\frac{i}{N+1} : i\in\Nset{N}}}\prob\sqbra{\set_{g_ik'}=\fset_{g_ik'},\set_{g_ik}=\fset_{g_ik}}\\
		&\hspace{4em}\times \Ex_{\fset_{g_ik'}}\sqbra{(a-\rankspec_{g_ik'})^2}\Ex_{\fset_{g_ik}}\sqbra{(a-\rankspec_{g_ik})(b-\rankspec_{g_ik})} \Bigg\rvert\\
		&\hspace{1em}\leq \frac{6}{5}(C_1+1)\frac{1}{N(g_i,k)}.
	\end{split}
\end{equation}
The inequality above is established in the same way as \cref{inequ::xy7}. In the situation that $t\neq s=i$, the same method yields
\begin{equation}
	\ABS{\Ex\sqbra{\diff_{ih}\diff_{il}\diff_{hs}\diff_{lt}}}\leq \frac{6}{5}(C_1+1)\frac{1}{N(g_i,k')}.
\end{equation}
For any fixed $k$ and $k'$, the number of $h,l,s,t$ tuples satisfying Case~3.3 with $s\neq t=i$ ($t\neq s=i$) does not exceed $n_kn_{k'}n_{g_i}$. Therefore,
\begin{equation}
	\begin{split}
		&\left\vert \sum_{r\in\Nset{K}}\frac{1}{n_r}\sum_{\substack{h,l\in\Nset{n};s,t\in \grp _{r}\\ h,l,s,t \text{ satisfy Case~3.3}}} \Ex[\diff_{ih}\diff_{il}\diff_{hs}\diff_{lt}] \right\vert\\
		&\hspace{5em}
		\leq
		\frac{1}{n_{g_i}}\sum_{k,k'\in\Nset{K}} \frac{6}{5}(C_1+1)n_kn_{k'}n_{g_i}\sqbra{\frac{1}{N(g_i,k)}+\frac{1}{N(g_i,k')}}\\
		&\hspace{5em}
		\leq
		\frac{16}{5}\bra{C_1+1}\frac{K n}{n_{g_i}}.
	\end{split}
\end{equation}

\noindent\textbf{Case~4: $(s,t)=(l,h)$.} For any fixed $r$, the number of $h,l,s,t$ tuples satisfying Case~4 is $n_r^2$. Thus, since $\ABS{\Ex\sqbra{\diff_{ih}\diff_{il}\diff_{hs}\diff_{lt}}}\leq 1$, we have
\begin{equation}
	\left\vert \sum_{r\in\Nset{K}}\frac{1}{n_r}\sum_{\substack{h,l\in\Nset{n};s,t\in \grp _{r}\\ h,l,s,t \text{ satisfy Case~4}}} \Ex[\diff_{ih}\diff_{il}\diff_{hs}\diff_{lt}] \right\vert\leq \sum_{r\in\Nset{K}}\frac{n_r^2}{n_r}\leq n.
\end{equation}

We have finished analyzing all possible cases for the $h,l,s,t$ tuple, while restricting to the situation $i\neq h$, $i\neq l$, $h\neq s$, and $l\neq t$. Combining the above results yields
\begin{equation}
	\begin{split}
		&\left\vert \sum_{r\in\Nset{K}}\frac{1}{n_r}\sum_{h,l\in\Nset{n};s,t\in \grp _{r}} \Ex[\diff_{ih}\diff_{il}\diff_{hs}\diff_{lt}] \right\vert\\
		&\hspace{2em}\leq
		\left(\frac{22C_1+27}{5}\right)Kn + \left(\frac{44C_1+403}{5}\right)n + \bra{\frac{24}{5}C_1+16}K^2\\
		&\hspace{4em}+
		\left(C_1\sqrt{\frac{8}{3}}\right)\sqrt{\frac{K^{3}n^{3}}{n_{gi}}} + \frac{n^2}{n_{g_i}} + \left(\frac{16}{5}(C_1+1)\right)\frac{Kn}{n_{g_i}} + \bra{\frac{48}{5}C_1+33}\frac{n}{n_{g_i}}\\
		&\hspace{2em}\leq
		C\max\left\{\frac{n^2}{n_{g_i}},\sqrt{\frac{K^{3} n^{3}}{n_{g_i}}} \right\},
	\end{split}	
\end{equation}
where $C > 0$ is a constant. This concludes the proof of \cref{lem:ho_cctrtn}.
\end{proof}

\vskip 0.2in
\bibliography{jrc_RSCR_biblio}

\end{document}